\documentclass{article} 
\usepackage{amsmath} 
\usepackage{xspace}
\usepackage{caption}
\usepackage{amssymb}
\usepackage{mathtools}
\usepackage{subcaption}
\usepackage{enumitem}
\usepackage{algorithm}
\usepackage{algpseudocode}
\usepackage{xcolor}
\usepackage{graphicx}
\usepackage{nicefrac}
\usepackage{array}
\usepackage{multirow}
\usepackage{hhline}
\newcommand{\naturals}{\mathbb{N}}
\newcommand{\wholes}{\mathbb{N}_0}
\newcommand{\x}{\mathbf{x}}
\newcommand{\y}{\mathbf{y}}
\newcommand{\z}{\mathbf{z}}

\newcommand{\w}{\mathbf{w}}

\newcommand{\uvec}{\mathbf{u}}

\newcommand{\vvec}{\mathbf{v}}

\newcommand{\A}{\mathbf{A}}
\newcommand{\X}{\mathbf{X}}

\newcommand{\I}{\mathbf{I}}

\newcommand{\fancyF}{\mathcal{F}}
\newcommand{\fancyQ}{\mathcal{Q}}

\newcommand{\indi}{{(i)}}

\newcommand{\expText}[1]{\exp\left(#1\right)}

\newcommand{\abs}[1]{\left\lvert#1\right\rvert}
\newcommand{\norm}[2]{\left\lVert#2\right\rVert_{#1}}
\newcommand{\esqnorm}[1]{\left\lVert#1\right\rVert_2^2}
\newcommand{\enorm}[1]{\left\lVert#1\right\rVert_2}

\DeclareMathOperator*{\argmax}{arg\,max}
\newcommand{\Real}{\mathbb{R}}

\newcommand{\inner}[1]{\left\langle#1\right\rangle}

\newcommand{\roundbrack}[1]{\left( #1 \right)}
\newcommand{\curlybrack}[1]{\left\lbrace #1 \right\rbrace}
\newcommand{\squarebrack}[1]{\left\lbrack #1 \right\rbrack}

\newcommand{\textequal}[1]{\stackrel{\text{#1}}{=}}
\newcommand{\textleq}[1]{\stackrel{\text{#1}}{\leq}}

\newcommand{\diffx}{\x_0~-~\Pi_{\X^*}(\x_0)}
\newcommand{\pathlength}{\zeta}
\newcommand{\vel}[1]{\dot{#1}}
\newcommand{\lyapunov}{\varepsilon}

\newcommand{\quasiconvexity}{\text{QC}\xspace}

\newcommand{\LG}{\text{LG}\xspace}
\newcommand{\PKL}{\text{PKL}\xspace} %% WARNING: Some instances in text may refer to PL without using the macro \PL
\newcommand{\KL}{\text{KL}\xspace}
\newcommand{\GF}{\text{GF}\xspace} %% WARNING: Some instances in text refer to GF without using the macro \GF
\newcommand{\HB}{\text{HB}\xspace}
\newcommand{\GD}{\text{GD}\xspace} %% WARNING: Some instances in text refer to GD without using the macro \GD
\newcommand{\SGD}{\text{SGD}\xspace} %% WARNING: Some instances in text may refer to SGD without using the macro \SGD
\newcommand{\PGD}{\text{PGD}\xspace}
\newcommand{\dist}[2]{\mathrm{dist}\roundbrack{#1, #2}}
\newcommand{\convexSet}{\Omega}
\newcommand{\convexHull}{K}
\newcommand{\futurePath}{\Gamma}
\newcommand{\width}{W}
\newcommand{\shell}{\mathbb{S}}
\newcommand{\convexClosure}{\Omega}
\newcommand{\increment}{\epsilon}
\newcommand{\placeholder}{28^{2d^2}}
\newcommand{\improvedPlaceholder}{2^{(4d\log d) - 1}}

\newcommand{\ceil}[1]{\left \lceil{#1}\right \rceil}

\allowdisplaybreaks

\usepackage{url}
\usepackage{amsthm}
\newcommand{\BlackBox}{\rule{1.5ex}{1.5ex}}  % end of proof
\ifdefined\proof
    \renewenvironment{proof}{\par\noindent{\bf Proof\ }}{\hfill\BlackBox\\[2mm]}
\else
    \newenvironment{proof}{\par\noindent{\bf Proof\ }}{\hfill\BlackBox\\[2mm]}
\fi
 
\newtheorem{theorem}{Theorem}
\newtheorem{lemma}[theorem]{Lemma} 
 
\newtheorem{remark}[theorem]{Remark}
\newtheorem{corollary}[theorem]{Corollary}
\newtheorem{definition}[theorem]{Definition}

\long\def\acks#1{\vskip 0.3in\noindent{\large\bf Acknowledgments}\vskip 0.2in
\noindent #1}

\newenvironment{proofsketch}{\par\noindent{\bf Proof sketch\ }}{\hfill\BlackBox\\[2mm]}

\usepackage[toc,page]{appendix}
\usepackage[margin=1.2in]{geometry}
\usepackage{natbib}
\usepackage{hyperref}
\usepackage{minitoc}

\newcommand{\rev}[1]{#1}%{{\color{blue}{#1}}}
\newcommand{\revminor}[1]{#1}%{{\color{blue}{#1}}}
\newcommand{\cg}[1]{{\color{red}{[#1]}}}
\begin{document}
\doparttoc % Tell to minitoc to generate a toc for the parts
\faketableofcontents % Run a fake tableofcontents command for the partocs

\title{Path Length Bounds for Gradient Descent and Flow}

% \author{\name Chirag Gupta \email chiragg@cmu.edu\\ 
% \addr Machine Learning Department\\
%       Carnegie Mellon University\\
%       Pittsburgh, PA 15217, USA
% \\ \\ \name Sivaraman Balakrishnan \email siva@stat.cmu.edu\\  \addr Department of Statistics and Data Science\\
%       Carnegie Mellon University\\
%       Pittsburgh, PA 15217, USA 
%       \\ \\
% \name Aaditya Ramdas \email aramdas@cmu.edu
% \\ \addr Department of Statistics and Data Science, Machine Learning Department\\
%       Carnegie Mellon University\\
%       Pittsburgh, PA 15217, USA
% }
% \ShortHeadings{Path Length Bounds for Gradient Descent and Flow}{Gupta, Balakrishnan and Ramdas}
% \editor{Ambuj Tewari}
\author{
  Chirag Gupta$^1$, Sivaraman Balakrishnan$^2$ and Aaditya Ramdas$^{12}$  \\
  Machine Learning Department$^1$\\
  Department of Statistics and Data Science$^2$\\
  Carnegie Mellon University\\
  \texttt{chiragg@cmu.edu, \{siva,aramdas\}@stat.cmu.edu} 
}

\maketitle 
\begin{abstract}%
We derive bounds on the path length $\zeta$ of gradient descent (GD) and gradient flow (GF) curves for various classes of smooth convex and nonconvex functions. Among other results, we prove that: (a) if the iterates are linearly convergent with factor $(1-c)$, then $\zeta$ is at most $\mathcal{O}(1/c)$; (b) under the Polyak-Kurdyka-\L ojasiewicz (PKL) condition, $\zeta$ is at most $\mathcal{O}(\sqrt{\kappa})$, where $\kappa$ is the condition number, and at least $\widetilde\Omega(\sqrt{d} \wedge \kappa^{1/4})$; (c) for quadratics, $\zeta$ is $\Theta(\min\{\sqrt{d},\sqrt{\log \kappa}\})$ and in some cases can be independent of $\kappa$; (d) assuming just convexity, $\zeta$ can be at most $2^{4d\log d}$; (e) for separable quasiconvex functions, $\zeta$ is ${\Theta}(\sqrt{d})$. Thus, we advance current understanding of the properties of \GD and \GF curves beyond rates of convergence. We expect our techniques to facilitate future studies for other algorithms. 
\end{abstract}
% \begin{keywords}
% optimization, trajectory analysis, condition number, self-contracted curves, Polyak-Kurdyka-\L ojasiewicz functions
% \end{keywords}
\part{}
\parttoc

\section{Introduction}
%\rev{\subsection{Problem Motivation}
%The empirical success of gradient descent on a variety of machine learning problems has inspired researchers to study properties of the specific trajectory or curve followed by the optimization iterates
Recent work in machine learning has sought to understand the empirical success of gradient descent (\GD) through trajectory analysis~\citep{ge2015escaping, lee2017first, li2018visualizing, oymak2019overparameterized}. Trajectory analysis techniques aim to discover desirable geometric properties satisfied by the \GD curve that lead to good convergence behavior. %through analyses that discover specific desirable properties of the optimization trajectory%that seeks to go beyond global convergence rate analysis%through trajectory analysis methods that seek to understand %gradient descent and stochastic gradient descent beyond analysis of convergence rates% to properties of the specific trajectory or curve followed by the optimization iterates
One such property is an upper bound on the path length of the \GD curve. Path length bounds have been used in recent convergence analyses for deep neural networks~\citep{du2018gradient, du2019gradient,allen2019convergence}. These papers argue that fast convergence can be achieved even if desirable curvature properties 
%required for fast convergence 
are not satisfied globally, provided they hold within a local region around initialization. A path length bound is used to ensure that the iterates stay within this fast convergence region. 
%a path length bound shows that the iterates stay within a local region around the initialization, where desirable curvature properties required for fast convergence are satisfied. 
In this work, we study the path length of \GD and gradient flow (\GF) curves as an independent object of interest. This problem has been considered from slightly different perspectives in machine learning \citep{oymak2019overparameterized} and convex analysis \citep{bolte2010characterizations, manselli1991maximum, daniilidis2015rectifiability}. We draw from both these areas and further the study by establishing novel upper and lower path length bounds in a variety of settings. %in order to establish novel upper and lower bounds on the length of \GF and \GD curves. %\citet{oymak2019overparameterized} studied the path length of \GD and \SGD curves for least squares objectives where the prediction functions is nonlinear. More general functional forms have been considered by works in convex analysis~\citep{bolte2010characterizations, manselli1991maximum, daniilidis2015rectifiability}, but primarily for gradient flow (\GF) curves. 

%While the focus of this paper is on path length bounds for \GD and \GF, we also note other problems for which path length bounds have been considered. Bounds on the $\ell_1$ path length of lasso and forward stagewise regression have been studied by~\citet{hastie2007forward}. Adaptive regret bounds for bandits have been shown that depend on the total path length of the observed losses at each step~\citep{wei2018more, pmlr-v99-bubeck19b}. \citet{argue2019nearly} showed a result for nested convex body chasing using a bound on the path length of self-contracted curves~\citep{manselli1991maximum}, which was originally developed for proving the convergence of \GF for any quasiconvex function.

Formally, consider the problem of minimizing an objective function 
$f: \Real^d \to \Real$ using an iterative update rule. %In this work, we analyze the path length exhibited by the sequence of iterates formed by an optimization algorithm. 
%Suppose the set of global minimizers of $f$ is $\X^* := \argmin f(\x)$. A minimization path is a sequence of iterates that ends in a global minimizer. Formally, we denote a minimization path by
An optimization curve refers to the sequence of iterates of the update rule and is denoted by a mapping 
$g: S \to \Real^d$, where $S$ is either $\Real^+_0$ or $\naturals_0$. % is such that ${\lim_{s \to \infty}g(s) \in \X^*}
$S = \Real^+_0$ corresponds to the continuous time setting, where the iterative update rule is specified using an ordinary differential equation. Here, we typically denote an element of $S$ as $t$, to be thought of as `time'. On the other hand, $S = \naturals_0$ corresponds to discrete update rules, where we denote an element of $S$ as $k$, to be thought of as an iterate count. In both cases, we use $\x_s$ to denote $g(s)$. Iterative optimization techniques construct this mapping $g$ using local update rules based on the gradient of $f$ at $\x_s$, starting at some initial point $\x_0$. If $f$ is differentiable, one such update rule takes the following form: %of an ordinary differential equation:
\begin{equation*} 
  \text{Gradient flow (\GF):}\qquad \vel{\x}_t = -\nabla f(\x_t).
\end{equation*}
A forward Euler discretization of the above ordinary differential equation with a fixed step-size $\eta$ yields
\begin{equation*}
    \text{Gradient descent (\GD):} \qquad
    \x_{k+1} = \x_k - \eta \nabla f(\x_k).
\end{equation*}
In this paper, we bound the path lengths of the aforementioned update rules:
\begin{subequations}
\begin{equation*}
\text{(continuous) }\qquad \pathlength(f, \x_0) := \int_0^\infty \enorm{\vel{\x}_t} dt
 = \int_0^\infty \enorm{\nabla f(\x_t)} dt. \quad\ \ \ \ \ \ \ \ 
\end{equation*}
\begin{equation*}
\text{(discrete) }\qquad \pathlength_\eta(f, \x_0) := \sum_{k = 0}^\infty \enorm{\x_k - \x_{k + 1}} = \sum_{k=0}^\infty \enorm{\eta \nabla f(\x_k)}. \qquad\ \ 
\end{equation*}
\end{subequations}
Under the assumptions specified in the next section, the above integrals and sums are well defined---our results implicitly show that they converge to a finite value.% since we show upper bounds on the path length, the above integrals and sums are finite. 

We use the terms \emph{path} and \emph{curve} to refer to the same mathematical object, but we make the following stylistic choices for their usage. The mathematical object itself is typically referred to as the `\GD curve' or `\GF curve', while when discussing its length we always use the term `path length' instead of `curve length'. %When not qualified with `length', %will be used interchangeably in this paper. %Our bounds depend on the smoothness and curvature properties of $f$, which we introduce next.

\subsection{Notation and Assumptions}
\label{subsec:notation}
Throughout this paper, we assume that the objective function $f$ is continuously differentiable on $\Real^d$ , and that the minimum of the function $f$ is achieved by at least one point $\x^*$ with $\|\x^*\|_2 < \infty$. Denote the minimum value of the objective as 
\begin{align*}
f^* := \min_{\x \in {\Real^d}} f(\x) ~=~ f(\x^*), 
\end{align*}
and the set of all minimizers as $\X^* := \{\x \in \Real^d : f(\x) = f^*\}$.  By first order optimality, $\nabla f(\x^*) = 0$ is a necessary (but not sufficient) condition for $\x^*$ to be an element of $\X^*$. Under the standard curvature assumptions that we consider later in this paper 
$\X^*$ will be a convex set. Note that $\X^*$ is closed since $f$ is continuous. Hence the projection of the initial point $\x_0$ on the optimal set $\X^*$ is uniquely defined. We denote this projection as $\Pi_{\X^*}(\x_0)$. We then denote the minimum distance between the initialization point $\x_0$ and the optimal set $\X^*$ as
\begin{align*}
\text{dist}(\x_0, \X^*) := \min_{\x^* \in \X^*}\enorm{\x_0 - \x^*} = \enorm{\diffx}, 
\end{align*}
which is finite since $\text{dist}(\x_0, \X^*) \leq \text{dist}(\x_0, \x^*) < \infty$.
Many of our path length guarantees are a product of $\text{dist}(\x_0, \X^*)$ and a factor ($>1$) that depends on curvature or smoothness properties of the function (defined shortly) or the dimension $d$. When we write path length bounds using $\mathcal{O}(\cdot)$ or $\Omega(\cdot)$ notation, we absorb the factor $\text{dist}(\x_0, \X^*)$ as a constant. 

\newcolumntype{H}{>{\setbox0=\hbox\bgroup}c<{\egroup}@{}}

\begin{table}
\centering 
\rev{
 \begin{tabular}{|c|c | c | c|} 
 \hline
 {\bf Theorem} & {\bf Assumption} &  {\bf Upper bound} & {\bf Lower  bound}  \\ 
\hhline{|=|=|=|=|} % \hline \hline\dist{\x}{\X^*}
 Theorem~\ref{thm:LCGD}, \ref{thm:LCGext}, \ref{thm:PLlowerBound} & \begin{tabular}{@{}c@{}}\LG, $(1,c)$-linear \\  convergence\end{tabular} & $\mathcal{O}(1/c)$ & $\widetilde\Omega(\sqrt{1/c})$ \\
\hline 
 {Theorem~\ref{thm:PKLLGGF}$^*$, \ref{thm:PLlowerBound}} & \PKL, \LG & $\mathcal{O}(\sqrt{\kappa})$  & $\widetilde{\Omega}(\min \{\kappa^{1/4}, \sqrt{d}\})\ $\\
\hline 
 {Theorem~\ref{thm:QUADGF}, \ref{thm:PLlowerBoundQuadratic}}& Quadratic objective & $\mathcal{O}(\min\{\sqrt{\log \kappa}, \sqrt{d}\})$ & $\Omega(\min \{\sqrt{\log \kappa}, \sqrt{d} \})\ $\\
\hline 
{Theorem~\ref{thm:QC}$^*$}, \ref{thm:PLlowerBoundQuadratic}& Convexity, \LG & $e^{\mathcal{O}(d\log d)} $ & $\Omega(\sqrt{d})$\\
\hline 
 {Theorem~\ref{thm:decomposable}, \ref{thm:PLlowerBoundQuadratic}}& \begin{tabular}{@{}c@{}} Quasiconvexity, \LG, \\ separability\end{tabular} %\quasiconvexity, \LG, separability 
 & $\mathcal{O}(\sqrt{d})\ $& $\Omega(\sqrt{d}) $\\
\hline 
\end{tabular}
\caption{\rev{Summary of path length bounds for \GF and \GD. \LG refers to Lipschitz Gradients (Definition~\ref{def:lg}) and \PKL is the Polyak-Kurdyka-\L ojasiewicz condition (Definition~\ref{def:pkl}). Other terms are also defined in Section~\ref{subsec:notation}. %\cg{For all bounds except the upper bound for Convexity+\LG we subsume the factor $\dist{\x_0}{\X^*}$ which represents the minimum distance to the optimal set.}  
$^*$The \GF version of Theorem~\ref{thm:PKLLGGF} is due to \citet{bolte2010characterizations}. The \GF version of Theorem~\ref{thm:QC} is due to \citet{manselli1991maximum}. Further relationships to prior work are discussed in Section~\ref{subsec:prior}.} }
\label{table:summary}}
\end{table}

Most of our results are under the following standard smoothness assumption on $f$ (for instance, see \citet[Section 1.2.2]{nesterov2013introductory}). 
\begin{definition}[Lipschitz Gradients (\LG)] For all $\x, \y \in \Real^d$ and some $L > 0$, 
\begin{align*}
\enorm{\nabla f(\x) - \nabla f(\y)} \leq L\enorm{\x - \y}.
\end{align*} 
The above condition implies $f(\y) - f(\x) \leq \inner{\nabla f(\x), \y - \x} + \frac{L}{2}\esqnorm{\y - \x}$.
\label{def:lg}
\end{definition}

%In order to show path length bounds, we must also make certain curvature assumptions that guarantee convergence. 
We next introduce the curvature conditions under which path length bounds are studied in this paper.

\begin{definition}[Convexity and quasiconvexity (\quasiconvexity)] $f$ is convex if $f(\y) \geq f(\x) + \inner{\nabla f(\x), \y - \x}$ for all $\x, \y \in \Real^d$.
$f$ is quasiconvex if $f(\y) \leq f(\x) \implies \inner{\nabla f(\x), \y - \x} \leq 0$ for all $\x, \y \in \Real^d$. \label{def:qc}
\end{definition}
Since $f$ is differentiable, this condition for quasiconvexity is equivalent to saying that all sub-level sets of $f$ are convex \citep[Chapter 3]{avriel2010generalized}. It is also equivalent to the condition $f(t\x + (1-t)\y) \leq \max\{f(\x),f(\y)\}$ for all  $t \in [0,1]$ \citep{fenchel53}. It is clear from the definition that convex functions are quasiconvex. %In Section~\ref{sec:quasiconvex} we show path length bounds under convexity and quasiconvexity.

\begin{definition}[Polyak-Kurdyka-\L ojasiewicz condition (\PKL)] For all $\x \in \Real^d$, and some $\mu > 0$, 
\begin{align*}
\esqnorm{\nabla f(\x)} \geq 2\mu \roundbrack{f(\x) - f^*}.
\end{align*}
The PKL condition implies quadratic growth: $f(\x) - f^* \geq \frac{\mu}{2}\text{dist}^2(\x,\X^*)$ \citep{karimi2016linear}, and hence also linear growth for the gradient: $\|\nabla f(\x)\| \geq \mu \cdot \text{dist}(\x,\X^*)$.
\label{def:pkl}
\end{definition}
\rev{This inequality was introduced by \citet{polyak1963gradient} to show linear convergence for gradient descent. Independently, \citet{lojasiewicz1963propriete} showed that a generalized version of this inequality is true locally around a critical point for any real-analytic function. This result was further extended by \citet{kurdyka1998gradients}. The generalized inequality is now referred to as the Kurdyka-\L ojasiewicz (KL) inequality---we refer the reader to \citet{bolte2007lojasiewicz, bolte2010characterizations} for more background.}% For an example of a smooth convex function where the KL inequality does not hold, see \citet[Section 5.1]{bolte2020curiosities}.} 

The \PKL inequality is weaker than strong convexity:
\begin{align*}
\text{for all $\x, \y \in \Real^d,$ }f(\y) \geq f(\x) + \inner{\nabla f(\x), \y - \x} + \frac{\mu}{2}\esqnorm{\y - \x},
\end{align*}
in the sense that $\mu$-strongly convex functions are also $\mu$-\PKL. Linear convergence can be shown for \GD if $f$ satisfies $\mu$-\PKL (instead of the stronger $\mu$-strong convexity) and has Lipschitz gradients \citep{polyak1963gradient}. Yet the \PKL condition incorporates a significantly larger class of functions; in particular \PKL functions can even be nonconvex. Nonconvex \PKL function arise naturally in some machine learning optimization problems (for instance, see Lemma 26 \citep{soltanolkotabi2019theoretical}, Section 3~\citep{fazel2018global}, Equation (3.4) \citep{balakrishnan2017statistical} for nonconvex examples and Section 2.3 \citep{karimi2016linear} for convex but non-strongly convex examples). 

Convergence of \GD is often studied with respect to its dependence on the ratio of $L$ and $\mu$, called as the condition number (for instance, see \citet[Section 2.1.3]{nesterov2013introductory}).
\begin{definition}[Condition number]
For a $\mu$-\PKL function $f$ with $L$-Lipschitz gradients, we define the condition number $\kappa$ as $L/\mu$. 
\end{definition}
Convex quadratic objective functions satisfy \PKL and \LG, and the $\kappa$ here turns out to be the ratio of the largest and smallest non-zero singular values of the Hessian.
We also consider path length bounds under a general linear convergence assumption, formalized next. The following definition is for any discrete-time iterative updates $\{\x_k\}_{k \in \naturals_0}$ or any continuous-time dynamics $\{\x_t\}_{t \geq 0}$ (not restricted to \GD or \GF).
\begin{definition}[$(A, c)$-Linear convergence]
\label{def:linConvOptimal}
We say that a procedure is linearly convergent with constants $c \in (0, 1)$ and $A \geq 1$ if for every initial point $\x_0 \in \Real^d$ and any $s \geq 0$,
\begin{equation}
    \dist{\x_s}{\X^*}  \leq A(1-c)^s\dist{\x_0}{\X^*}. \label{eqn:linConv}
\end{equation}
\end{definition}
In the discrete case, we have $s \in \wholes$ and in the continuous case we have $s \in \Real_0$. In the discrete case $c$, should really be thought of as $c_\eta$ since the linear convergence rate depends on the step-size $\eta$. However, we use $c$ in both cases to simplify exposition. %(and as we shall see later, $c$ depends on the step-size $\eta$ used so that it should be thought of as $c_\eta$) 
%In Section~\ref{sec:quasiconvex} we show path length bounds under convexity and quasiconvexity \citep{fenchel53}. 
%\PKL functions are quasiconvex; we show this by contradiction. Suppose there exists a witness $(\x, \y, t \in (0, 1))$ that violates quasiconvexity: $f(\z = t\x + (1-t)\y) > \max(f(\x), f(\y))$. Define $g: [0,1] \to \Real$ as $g(t) = f(t\x + (1-t)\y)$. $g$ is continuously differentiable and by assumption its minimum must be achieved at some $t \in (0, 1)$. By first order optimality, we must have $g'(t) =  \inner{\nabla f(t\x + (1-t)\y), \x - \y} 0$.   Since $f$ is continuous, by the extreme value theorem, there exists a value $u \in (0, 1)$ such that 

Table~\ref{table:summary} summarizes our results in these various settings. Figure~\ref{fig:vennClasses} shows the containment relationships between various conditions introduced here (these are standard known facts). Finally, we note that the smoothness and curvature assumptions we make are global. However, since path length bounds imply that the iterates always stay within a ball, these assumptions can always be restricted to that ball (as is done formally by \citet{oymak2019overparameterized}). %Before discussing the results in each of these settings in detail, we discuss motivation and background for the problem of computing path length bounds. 

\begin{figure}[t]
    \centering
    \includegraphics[width=\textwidth,trim=0cm 2cm 0cm 1cm, clip]{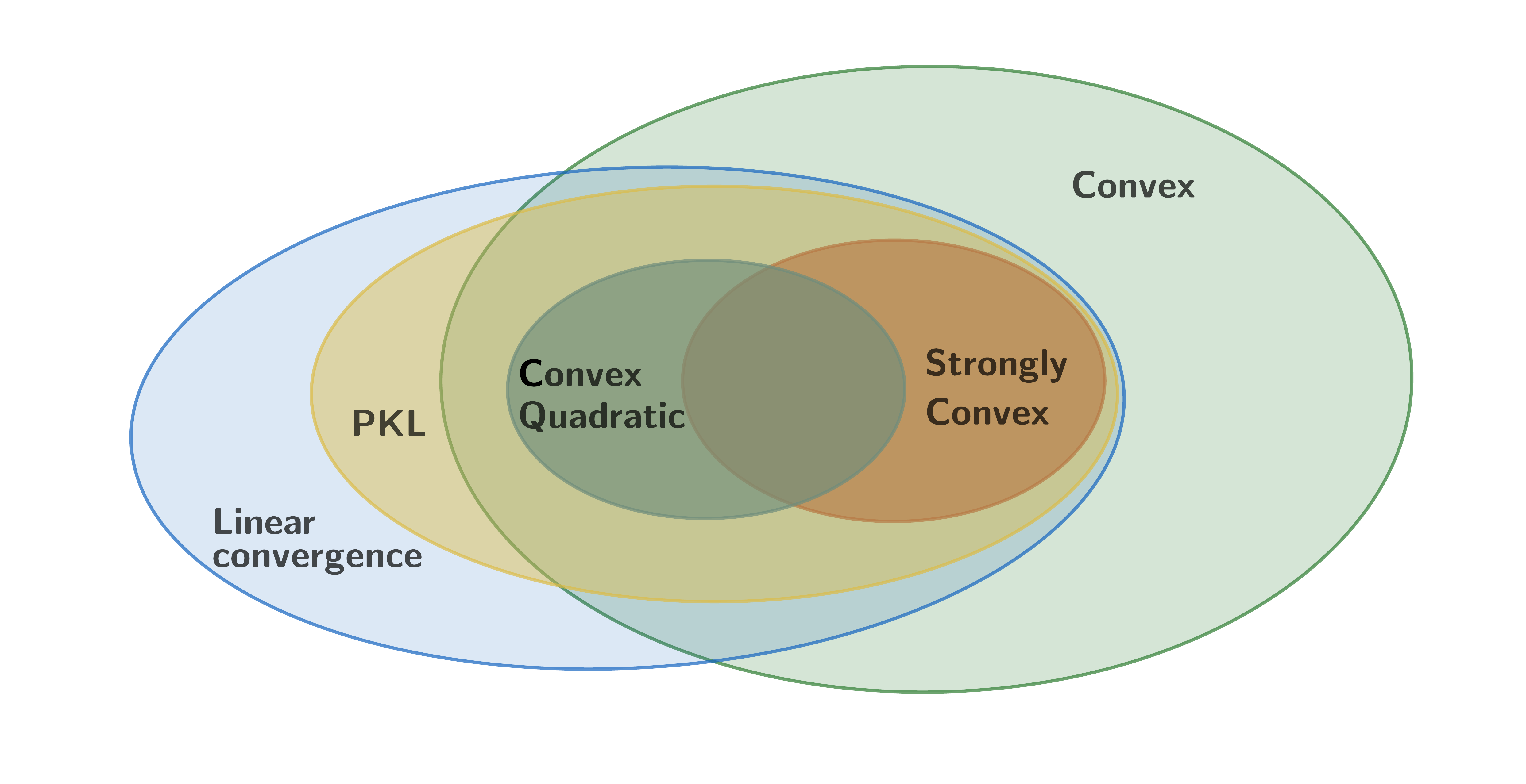}
    \caption{Venn diagram showing different curvature conditions (defined in Section~\ref{subsec:notation}) and their relationships. ``Linear convergence" refers to functions for which \GD or \GF is linearly convergent. The quadratic losses we consider may have Hessians with zero singular values, in which case the loss is not strongly convex. }
    \label{fig:vennClasses}
\end{figure}

\ifx false 
\rev{\subsection{Problem Motivation}
Recent work in machine learning has sought to expand the study of \GD and \SGD from convergence rates to properties of the specific trajectory or curve followed by the optimization iterates~\citep{oymak2019overparameterized, li2018visualizing, ge2015escaping, lee2017first}. The length of the trajectory is one such property that we study in this work. Path length bounds have been used implicitly in recent papers that have provided convergence and generalization guarantees for deep neural networks~\citep{du2018gradient, du2019gradient, oymak2019towards, soltanolkotabi2019theoretical,allen2019convergence}, in order to prove that the parameters stay within a local optimization region in which desirable smoothness properties are satisfied. \cg{\sout{The proof technique used to provide these guarantees is often as follows: (1) at initialization, the function is locally strongly convex due to overparameterization; (2) during optimization the parameters don't change too much because of a path length bound; 
(3) this fact ensures that the we don't leave the local optimization region in which strong convexity is satisfied; (4) thus we can use strong convexity to guarantee linear convergence rates.}}

While the focus of this paper is on path length bounds for \GD and \GF, we also note other problems for which path length bounds have been considered. Bounds on the $\ell_1$ path length of lasso and forward stagewise regression have been studied by~\citet{hastie2007forward}. Adaptive regret bounds for bandits have been shown that depend on the total path length of the observed losses at each step~\citep{wei2018more, pmlr-v99-bubeck19b}. \citet{argue2019nearly} showed a result for nested convex body chasing using a bound on the path length of self-contracted curves~\citep{manselli1991maximum}, which was originally developed for proving the convergence of \GF for any quasiconvex function.}
\fi

\subsection{Prior Work}
\label{subsec:prior}
The problem of bounding the length of \GF curves of quasiconvex functions was first considered by~\citet{manselli1991maximum}. Their analysis is via a reduction to Lipschitz continuous curves (meaning that the curve map $g: \Real \to \Real^d$ is Lipschitz continuous) that exhibit the self-contracted property (see Definition~\ref{def:selfcontracted}). Under an a priori assumption that the \GF curve has finite length, they prove that the length can be at most $2^{\mathcal{O}(d \log d)}$.~\citet[Corollary~2.4]{daniilidis2015rectifiability} show that the finiteness assumption can be dispensed. Further while Manselli and Pucci require the \GF curve to be smooth, Daniilidis et al. provide a new analysis that includes non-smooth self-contracted curves (with a weaker $2^{10d^2}$ bound). This fact enables us to extend the analysis from \GF curves to \GD curves in the convex+\LG case (Theorem~\ref{thm:QC}), leading to the first bound in this setting. Path length bounds for self-contracted curves have been studied more generally in non-Euclidean settings~\citep{stepanov2017self, daniilidis2018self}. It has also been shown that self-contracted curves that satisfy certain smoothness properties must be \GF curves of convex functions~\citep{durand2019self}. The papers in this line of work use geometric properties of self-contracted curves without referencing the specifics of the \GF dynamics.   

While these papers are concerned about fundamental questions about \GF curves under the mild assumption of quasiconvexity, we know that quasiconvexity, or even convexity, is not enough to obtain the fast convergence rates that optimization algorithms exhibit empirically in many machine learning settings. Indeed, we may hope for path length bounds that are potentially independent of $d$ when assumptions such as strong convexity or \PKL hold. ~\citet{bolte2010characterizations} showed a dimension independent path length bound for \GF curves whenever $f$ satisfies a Kurdyka-{\L}ojasiewicz (\KL) condition, a larger class that subsumes the \PKL condition. %This assumption is more general than the \PKL inequality we consider. In the setting of our interest, we note that the \KL condition is satisfied for \PKL and convex quadratic functions. 
\rev{Their result provides the best known upper bound of $\mathcal{O}(\sqrt{\kappa})$ for \GF curves of \PKL functions (reproduced in Theorem~\ref{thm:PKLLGGF}). However, for quadratic objective functions, we  show an exponentially improved $\mathcal{O}(\sqrt{\log \kappa})$ bound (Theorem~\ref{thm:QUADGF}). The best known bound for strongly convex functions is the $\mathcal{O}(\sqrt{\kappa})$ bound for \PKL functions, but our findings for quadratics suggests that this bound may be loose. We also show two novel lower bounds in these settings: $\widetilde{\Omega}(\kappa^{1/4})$ for \PKL functions (Theorem~\ref{thm:PLlowerBound}) and $\Omega(\sqrt{\log \kappa})$ for quadratics (Theorem~\ref{thm:PLlowerBoundQuadratic}). Thus, although the \KL condition holds very generally, an all purpose bound depending on the KL constant may not be tight. \citet{oymak2019overparameterized} analyzed \GD and \SGD path lengths for nonlinear least square objectives, under spectral assumptions on the Jacobian of the non-linear mapping. %Our general study is not limited to least squares objectives. 
Apart from least squares objectives, %\citet{oymak2019overparameterized} considered this general problem in Section~5 of their paper---
they also showed an $\mathcal{O}(\kappa)$ bound for PKL functions. We improve their result by showing an $\mathcal{O}(\sqrt{\kappa})$ bound for \GD (Theorem~\ref{thm:PKLLGGF}).}
%\cg{Compare to our result regarding $\kappa$ looser than Bolte et al. $\sqrt{\kappa}$.} \rev{\sout{They also show a bound under the \PKL assumption for \GD which is identical to the \GF bound of}~\citet{bolte2010characterizations}.}  

\subsection{Organization} 
The rest of the paper is organized as follows: 
\begin{enumerate} 
    \item In Section~\ref{sec:linconv}, we prove a general purpose path length bound for \GD and \GF curves that is applicable under any set of curvature conditions for which linear convergence can be established. \rev{Using the technique of Section~\ref{sec:linconv}, in Appendix~\ref{appsec:additionalLC} we show path length bounds for Polyak's heavy ball method and projected gradient descent.}
    \item In Section~\ref{sec:PKL} we show path length bounds for \GD and \GF curves of \PKL functions. In Section~\ref{sec:PLlowerBound}, we present a worst case lower bound in the same setting. 
    \item In Section~\ref{sec:quadratics}, we provide the tightest known path length bound for \GD and \GF curves of convex quadratic objective functions (potentially overparameterized). In Section~\ref{sec:QuadraticLowerBound}, we provide a a matching lower bound construction. 
    \item In Section~\ref{sec:quasiconvex} we derive explicit dimension dependent path length bounds for quasiconvex functions (\GF) and convex functions (\GD and \GF). This leads to the first known bounds for GD curves in this setting.
\end{enumerate}

Table~\ref{table:summary} summarizes the results in this paper. %. All bounds are true for both \GD and \GF curves. 
We remark that proofs in the \GF case are often more straightforward, but lead to conclusions that continue to hold in the \GD case with appropriate step-size restrictions.

\section{Dimension Independent Path Length Bounds}
\label{sec:dimBounds}
In this section, we provide dimension independent bounds on the length \GD and \GF curves. All the bounds in this section hold for functions for which \GD and \GF exhibit linear convergence towards the optimal set. In particular we discuss PKL objectives, strongly convex objectives, and convex quadratic objectives. More generally if the function does not belong to one of the aforementioned function classes, but linear convergence is known to hold, we can prove path length bounds that depend on the constant of convergence. We first prove this meta-theorem and then discuss specific function classes (where the meta-theorem bound can be improved). 

\subsection{General Bounds Under Linear Convergence}
\label{sec:linconv}
Linear convergence (Definition~\ref{def:linConvOptimal}) is known for \GD (with appropriate step-size) if the function has Lipschitz gradients and is strongly convex with $A = 1$ and $c = 1/\kappa$. For \PKL functions we have such a linear convergence result in function values instead of distance to the optimal set with $A = 1$ and $c = 1/\kappa$~\citep[Theorem~1]{karimi2016linear}. Using the LG condition and a quadratic growth result due to~\citet[Theorem~2]{karimi2016linear}, this convergence in function value can be converted to parametric linear convergence in the sense of Definition~\ref{def:linConvOptimal} with $A = \kappa$ and $c = 1/\kappa$. All of these convergence results are also known for \GF with $c = 1 - e^{-\mu}$ in each  case. \GD is also known to be linearly convergent when performing maximum likelihood estimation in logistic regression with unseparable data~\citep[Theorem~3.3]{freund2018condition}. In the theorem to follow, we guarantee a path length bound in each of these settings. 

\begin{theorem}
\label{thm:LCGD}
	Suppose $f$ has $L$-Lipschitz gradients. If the \GF dynamics for $f$ exhibits linear convergence with constants $(A, c)$, then its path length is bounded as: 
    \begin{equation}
        \label{eqn:LCGF}
        \pathlength \leq (AL/\log\roundbrack{1/(1-c)})\ \dist{\x_0}{\X^*},
    \end{equation}
and if the \GD iterates with step-size $\eta$ exhibit linear convergence with constants $(A, c)$, then their path length is bounded as: 
    \begin{equation}
        \label{eqn:LCGD1}
        \pathlength_\eta \leq (\eta AL/c)\ \dist{\x_0}{\X^*}.    
    \end{equation}
\end{theorem}

\begin{proofsketch}
The detailed proof of a more general result can be found in Theorem~\ref{thm:LCGDextended}, Appendix~\ref{appsec:additionalLC}. As a consequence of linear convergence, as the distance to the optimal set decreases geometrically, so does the contribution of consecutive iterates to the path. As illustrated in Figure~\ref{fig:linConv}, we then have the bound
\begin{equation}
 \pathlength_\eta \leq \dist{\x_0}{\X^*}\roundbrack{\eta AL + (1-c)\eta AL + (1-c)^2 \eta AL + \ldots} = (\eta AL / c)\ \dist{\x_0}{\X^*}, \label{eqn:linConvRecurrence}
\end{equation}
which proves claim~\eqref{eqn:LCGD1}. 
\end{proofsketch}
\begin{figure}[t!]
    \centering
    \includegraphics[width=0.8\linewidth,trim=3cm 2cm 3.5cm 6.5cm]{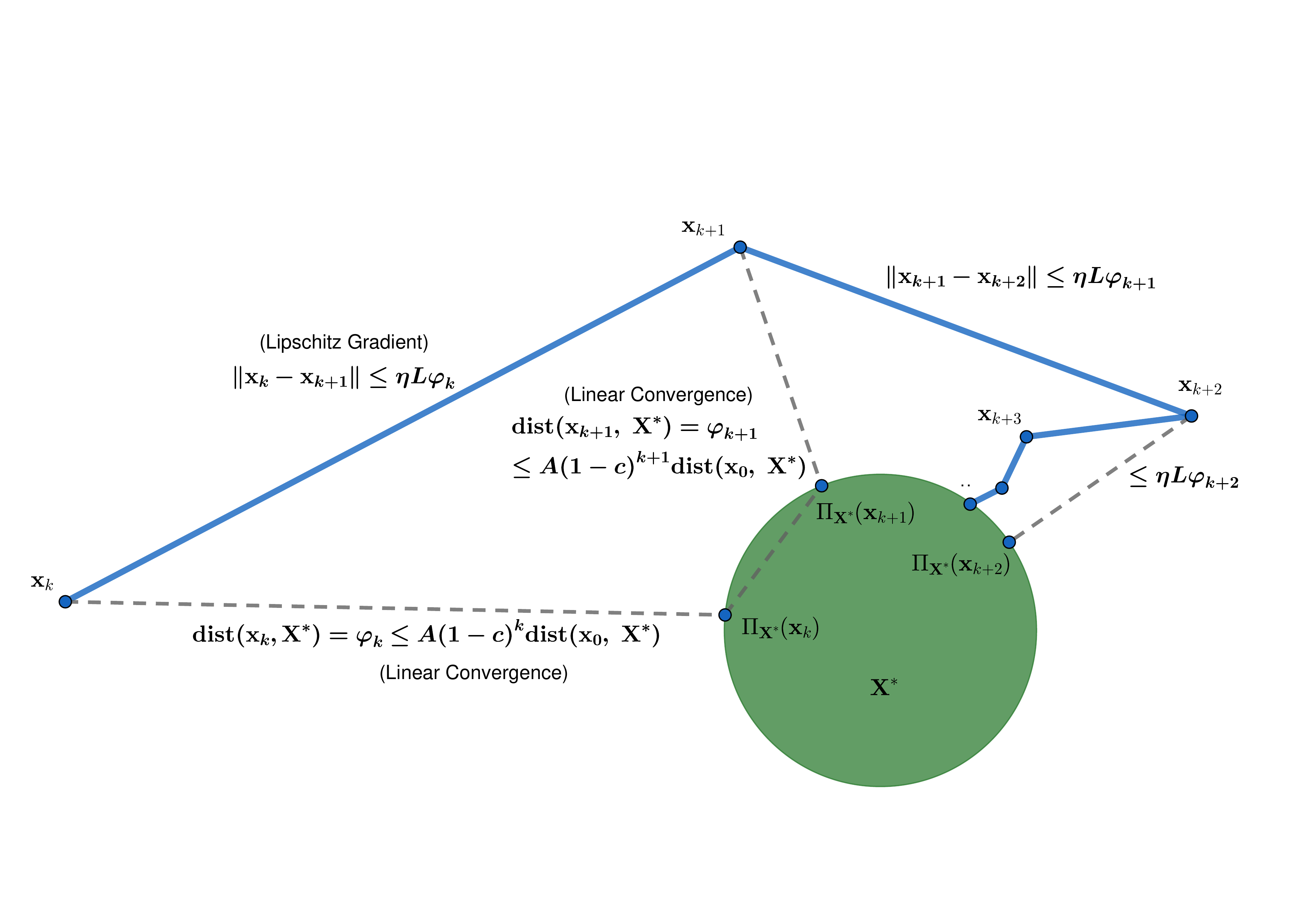}
    \caption{Path length bound under linear convergence. The \GD curve is shown with solid blue lines. The inequalities in bold upper bound individual terms in the path length sum, as noted in Equation~\eqref{eqn:linConvRecurrence}.}
    \label{fig:linConv}
\end{figure}
\begin{remark}\label{remark:linConv}
Variants of this theorem can be derived in more general settings:% (see Theorem~\ref{thm:LCGDextended} in Appendix~\ref{appsec:additionalLC} for details):
\begin{enumerate}[label=(\alph*)]
    \item %Theorem~\ref{thm:LCGDextended} is a strict generalization that is stated and proved in Appendix~\ref{appsec:additionalLC}. 
    Theorem~\ref{thm:LCGDextended} in Appendix~\ref{appsec:additionalLC} extends the result from linear convergence towards the globally optimal set $\X^*$ to linear convergence towards any convex set of points that satisfy first order optimality conditions.
    \item  \rev{Theorem~\ref{thm:LCPGD} in Appendix~\ref{appsec:additionalLC} extends the result to projected gradient descent, with a larger multiplicative constant.}
    \item Theorem~\ref{thm:LCGext} extends the result to any iterative update rule not limited to \GD, if descent can be established with respect to all minimizers in $\X^*$.
%    \item \rev{The result can be extended to projected gradient descent as long as the optima is in the interior of the constraint set so that $\nabla f(\x^*) = 0$. Thus our bound can account for the additional projection step if the gradient descent updates land outside the constraint set.}
\end{enumerate}

\end{remark} 
\rev{For \GD with a $\mu$-strongly convex function that has $L$-Lipschitz gradients, it is a standard result that linear convergence holds for $\eta \leq 1/L$, with $A = 1$ and $c = \eta \mu$. This leads to the bound $\pathlength_\eta \leq \kappa\enorm{\x_0 - \x^*}$. For \GF in the same setting, linear convergence holds with $A = 1$ and $c = 1 - e^{-\mu}$. Since $\log(1/(1-c)) = \log(e^{\mu}) = \mu$, the bound also resolves to $\pathlength \leq \kappa\enorm{\x_0 - \x^*}$.} However, the $\mathcal{O}(\kappa)$ dependence can be improved. In Section~\ref{sec:PKL}, following the work of~\citet{bolte2010characterizations}, we show that in fact the $\kappa$ bound for $\mu$-strongly convex functions can be improved to $\mathcal{O}(\sqrt{\kappa})$ with a weaker $\mu$-\PKL assumption. In Section~\ref{sec:quadratics} we study a specific strongly convex problem, namely convex quadratic objective functions. In this case we show that the above bound can be  improved to $\mathcal{O}(\sqrt{\log\kappa})$. 

Theorem~\ref{thm:LCGD} can be generalized to include any iterative update rule (not limited to \GD), and any function class for which linear convergence holds, as long as descent holds with respect to all minimizers in $\X^*$. In the following, we denote the update for a point $\x$ as $\x^+$. %Equation~\eqref{eqn:linConv} holds with $\X^*$ replaced by a fixed minimum point $\x^* \in \X^*$ (that is we have linear convergence to a single point $\x^*$). 

\rev{
\begin{theorem}
    If the iterates for a discrete update rule satisfy the following conditions: 
    \begin{enumerate}[label=(\alph*)]
    \item linear convergence with constants $(A, c)$, and 
    \item descent towards all minimizers: for all $\x^* \in \X^*$, $\x \in \Real^d$, $\enorm{\x^+ - \x^*} \leq \enorm{\x - \x^*}$, 
\end{enumerate}     
then their path length is bounded as: 
    \label{thm:LCGext}
\begin{equation}
        \label{eqn:LCGD2}
        \pathlength_\eta \leq \roundbrack{\frac{2A}{c}} \dist{\x_0}{\X^*}.
    \end{equation}
\end{theorem}}
\ifx false 
\begin{theorem}
\label{thm:LCGext}
    Given any $f$, if the iterates for a discrete update rule ($\x \to \x^+$) exhibit linear convergence with constants $(A, c)$ to a fixed minimum $\x^*$, then their path length is bounded as: 
    \begin{equation}
        \label{eqn:LCGD2}
        \pathlength_\eta \leq ((A + A^2(1 - c))/c) \enorm{\x_0 - \x^*}.    
    \end{equation}
\end{theorem}
 
\begin{proofsketch}
The proof of a more general result can be found in Appendix~\ref{appsec:additionalLC} (see Theorem~\ref{thm:LCGD2extended}). 
Observe that 
\begin{align*}
\enorm{\x_k - \x_{k+1}} \leq \enorm{\x_k - \x^*} + \enorm{\x_{k+1} - \x^*} &\leq (1 + A(1 - c))\enorm{\x_k - \x^*} \\ &\leq (1 + A(1 - c)) A(1-c)^k \enorm{\x_0 - \x^*}.
\end{align*}
Consequently, we may bound the path length as: 
\begin{equation*}
 \pathlength_\eta \leq (A + A^2(1 - c))\enorm{\x_0 - \x^*}\roundbrack{1 + (1-c) + (1-c)^2 + \ldots} = (A + A^2(1 - c))\enorm{\x_0 - \x^*}/c,
\end{equation*}
 as required.
\end{proofsketch}
\fi
The proof of a more general result can be found in Appendix~\ref{appsec:additionalLC} (see Theorem~\ref{thm:LCGD2extended}). \rev{%A special case of Theorem~\ref{thm:LCGext} is: 
Note that if $\X^*$ is singleton and $A = 1$, condition (b) is implied by condition (a). %This theorem applies to any algorithm that exhibits linear convergence. 
Using this observation, in Corollary~\ref{thm:HBSCLG} (Appendix~\ref{appsec:additionalLC}) we show a path length bound of $\sqrt{\kappa}\enorm{\x_0 - \x^*}$ for Polyak's heavy ball method~\citep{polyak1964some} with a twice continuously differentiable and strongly convex function with Lipschitz gradients. %Condition (b) of Theorem~\ref{thm:LCGext} can be established for projected gradient descent (PGD) if $f$ is convex and has $L$-Lipschitz gradients. Thus, if additionally condition (a) can be shown, then we get a path length bound for PGD in this setting.

For \GD, condition (b) of Theorem~\ref{thm:LCGext} is known if $f$ is convex, has $L$-Lipschitz gradients and $\eta \leq 1/L$---in fact Lemma~\ref{lemma:QCGD} argues that the stronger property of self-contractedness (Definition~\ref{def:selfcontracted}) holds. However, there exist scenarios where condition (a) is known but not condition (b), such as nonconvex \PKL functions. In such cases, Theorem~\ref{thm:LCGD} can still be applied, as long as $f$ has $L$-Lipschitz gradients. %can still be applied in such cases. %where linear convergence can be shown to a set and we don't know if descent holds to all optimal points (for instance, descent to all optimal points is not known for nonconvex \PKL functions). 
}

\subsection{Path Length Under the \PKL Condition is $\boldsymbol{\mathcal{O}(\sqrt{\kappa})}$}
The bound in the \PKL case can be improved by analyzing a particular potential/Lyapunov function:
\label{sec:PKL}
\[
    \lyapunov_t = \sqrt{f(\x_t) - f^*}.
\]
It can be shown  by differentiating $\lyapunov_t$ and applying the \PKL condition that, 
\[
    -\vel{\lyapunov}_t \geq \sqrt{\frac{\mu}{2}}\enorm{\nabla f(\x_t)}.
\]
Integrating the above with respect to $t$ and using \LG leads to the following theorem path length bound for \GF curves (the \GD bound also follows a similar approach).
\begin{theorem}
\label{thm:PKLLGGF}
For any $\mu$-\PKL function $f$ with $L$-Lipschitz gradients, the \GF dynamics have a path length bounded as:
\[
\pathlength \leq \sqrt{\kappa} \ \dist{\x_0}{\X^*},
\]
while the \GD iterates with $\eta \leq 1/L$ have a path length bounded as:
$\pathlength_\eta \leq 2\sqrt{\kappa} \ \dist{\x_0}{\X^*}.$
\end{theorem}

The proof can be found in Appendix~\ref{appsec:proofsPKL}. \rev{The \GF version of this bound is a special case of a general statement shown by \citet[Theorem 27]{bolte2010characterizations} for a larger class of functions, namely those that satisfy the Kurdyka-{\L}ojasiewicz inequality. %Theorem~\ref{thm:PKLLGGF} is a special case of this more general result. 
The \GD version is new. \citet[Corollary~5.3]{oymak2019overparameterized} obtained an $\mathcal{O}(\kappa)$ bound while we show an $\mathcal{O}(\sqrt{\kappa})$ bound. Our result relies on Theorem 5.2 of their paper, where they show a path length bound in terms of $(f(\x_0) - f^*)$. However their final result in terms of $\dist{\x_0}{\X^*}$ is weaker than ours.} %(Corollary~5.3) in terms of $\dist{\x_0}{\X^*}$. This leads to an $\mathcal{O}(\kappa)$ dependence instead of the $\mathcal{O}(\sqrt{\kappa})$ bound as we have shown.

\subsection{Path Length for Convex Quadratic Objectives is $\boldsymbol{\mathcal{O}(\sqrt{\log{\kappa}})}$}
\label{sec:quadratics}

For convex quadratic objective functions we can explicitly write down the \GD or \GF iterates. This allows us to significantly improve the general linear convergence result. To make use of some standard notation, we write a general convex quadratic objective function as a linear regression problem specified by a matrix $\A$ of dimensions $n \times d$, and an output vector $\y \in \Real^n$. The objective is
\begin{equation}
    f(\x) = \frac{1}{2n}\esqnorm{\y - \A\x}. \label{eqn:quadratics}
\end{equation}

The columns of $\A$ may be linearly dependent (which is necessarily true for the overparameterized setting, when $d > n$). In this case the solution set $\X^*$ has more than one element. However, it is possible to show that both \GF and \GD converge to $\Pi_{\X^*}(\x_0)$, 
 given by 
 \begin{align*}
 \Pi_{\X^*}(\x_0) := (\I_{n\times n} - \A^T(\A^T)^\dagger)\x_0 + (\A^T\A)^\dagger \A^T\y, 
 \end{align*}
 where $B^\dagger$ denotes the Moore-Penrose inverse of a matrix $B$. If the columns of $\A$ are linearly independent, this reduces to the standard least squares solution $\Pi_{\X^*}(\x_0) = (\A^T\A)^{-1}\A^T\y$.

Define $\Sigma := (\A^T\A)/n$. The \GF curve can be computed in closed form as: 
\begin{align}
\x_t = \Pi_{\X^*}(\x_0)  - \Sigma^\dag\exp(-t \Sigma)\Sigma(\Pi_{\X^*}(\x_0) - \x_0), \label{eqn:quadraticFlow}
\end{align} 
so that $\x_\infty = \Pi_{\X^*}(\x_0)$. By definition, $\Sigma$ is positive semidefinite. If $\A$ has linearly dependent columns, $\Sigma$ would have zero singular values. In general, suppose the number of non-zeros singular values of $\Sigma$ is $d^+ \leq d$. We denote them as $\sigma_1 \geq \sigma_2 \cdots \geq \sigma_{d^+} > 0$. Then for any step-length $\eta \leq 1/\sigma_1$, the \GD iterates converge to $\Pi_{\X^*}(\x_0)$ via the following updates for $k \geq 1$:
\begin{equation}
    \x_{k} = \x_{k-1} - \frac{\eta \A^T(\A\x_{k-1} - \y)}{n}. \label{eqn:quadraticsGD}
\end{equation}
For $i \in [d^+-1]$, define $\kappa_i := \sigma_i/\sigma_{i+1}$. The overall condition number is $\kappa := \sigma_1/\sigma_{d^+}$. The following theorem shows a path length bound for quadratic objectives in terms of each of the quantities: $d^+$, the $\kappa_i$'s and $\kappa$. %As we show in the following theorem, the path length bound for quadratic objective functions depends on the $\kappa_i$'s. 

\begin{theorem}
    \label{thm:QUADGF} For convex quadratic objective functions \eqref{eqn:quadratics}, the \GF dynamics \eqref{eqn:quadraticFlow} have a path length bounded as:
    \begin{equation}
        \pathlength \leq \min{\curlybrack{\sqrt{d^+}, 1 + \sum_{j=1}^{d^+-1} \kappa_j^{-1/(\kappa_j-1)} (1-1/\kappa_j),  1+2.5\sqrt{\log\kappa}}\ \dist{\x_0}{\X^*}}, \label{eqn:QUADGF}
    \end{equation}
    while the \GD iterates \eqref{eqn:quadraticsGD} with $\eta \leq 1/\sigma_1$ have a path length bounded as: 
       \begin{equation}
        \pathlength_\eta \leq \ \dist{\x_0}{\X^*} + \pathlength. \label{eqn:QUADGD}
    \end{equation}
\end{theorem}
To clarify, when $\kappa_j = 1$, $\kappa_j^{-1/(\kappa_j-1)} (1-1/\kappa_j)$ is defined as $\lim_{\kappa_j \to 1^+} \kappa_j^{-1/(\kappa_j-1)} (1-1/\kappa_j) = 0$. The proof of Theorem~\ref{thm:QUADGF} can be found in Appendix~\ref{appsec:proofsQuadratics}.

\begin{remark} We show in the proof that $\sum_{j=1}^{d^+-1} \kappa_j^{-1/(\kappa_j-1)} (1-1/\kappa_j) \leq \frac{\log \kappa}{e}$. This cannot be improved to $o(\log \kappa)$ as shown next. Suppose $\kappa_j = 6$ for every $j \in [d^+-1]$. It can be verified that $6^{-1/(6-1)} (1-1/6) \geq 0.5 $, and thus $\sum_{j=1}^{d^+-1}  \kappa_j^{-1/(\kappa_j-1)} (1-1/\kappa_j) \geq 0.5( d^+-1) \geq \frac{\log \kappa}{2\log 6}$, since $\kappa = 6^{d^+-1}$. This is true for any $d^+$ and consequently for all large $\kappa$. Thus the bound cannot be improved to $o(\log\kappa)$ in general, making the $\mathcal{O}(\sqrt{\log\kappa)}$ bound the best asymptotic result with respect to $\kappa$. However for special cases, such as if there is only one large singular value, this bound may even be independent of $\kappa$ and $d^+$. For instance, suppose $\kappa_j = 1$ for $j \in [d^+-2]$, and $\kappa = \kappa_{d^+-1} \in [1, \infty)$. Then, no matter the value of $d^+, \kappa$, we obtain the bound $\pathlength \leq 2\ \dist{\x_0}{\X^*}$.
\end{remark}

%\begin{remark} We show in the proof that $\sum_{j=1}^{d^+-1} \kappa_j^{-1/(\kappa_j-1)} (1-1/\kappa_j) \leq \frac{\log \kappa}{e}$. This cannot be improved to $o(\log \kappa)$ as shown next. Set $\kappa_j = t$ for every $j \in [d^+-1]$ and let $t \to \infty$. Then $\sum_{j=1}^{d^+-1}  \kappa_j^{-1/(\kappa_j-1)} (1-1/\kappa_j) = (d^+-1) \roundbrack{t^{-1/(t-1)} (1-1/t)} \to d^+-1$. Thus for large enough $t$, we have $\sum_{j=1}^{d^+-1}  \kappa_j^{-1/(\kappa_j-1)} (1-1/\kappa_j) \geq (d^+ - 1)/1.001$. However $\kappa = t^{d^+-1}$ %and thus $\log\kappa = (d^+-1)\log t$, and so
% and so $\sum_{j=1}^{d^+-1} \kappa_j^{-1/(\kappa_j-1)} (1-1/\kappa_j) \geq \frac{\log \kappa}{1.001 \cdot \log t}$ for any value of $d^+$, and consequently any large value of $\kappa$. For special cases, such as if there is only one large singular value, the bound in Theorem~\ref{thm:QUADGF} may even be independent of $\kappa$ and $d^+$. For instance, suppose $\kappa_j = 1$ for $j \in [d^+-2]$, and $\kappa = \kappa_{d^+-1} \in [1, \infty)$. Then, no matter the value of $d^+, \kappa$, we obtain the bound $\pathlength \leq 2\ \dist{\x_0}{\X^*}$.
%\end{remark}

There exists a family of quadratic functions that also satisfy a matching $\Omega(\sqrt{\log \kappa})$ lower bound on their path length, as shown in Theorem~\ref{thm:PLlowerBoundQuadratic}. \rev{However in Appendix~\ref{appsec:quadratic-simulation} we present experimental evidence to suggest that the constant $2.5$ may not be sharp.} This bound %The $\mathcal{O}(\sqrt{\log \kappa})$ bound for quadratic objectives 
fundamentally improves the $\sqrt{\kappa}$ dependence we expect via the best known bound for strongly convex functions (Section~\ref{sec:PKL}). Although we show a $\text{poly}(\kappa)$ lower bound for \PKL functions (Section~\ref{sec:PLlowerBound}), we are not aware of $\text{poly}(\kappa)$ lower bounds for strongly convex functions. In the absence of such lower bounds, the upper bound of this section suggests a potential improvement on the path length bound for strongly convex functions as well. To this end, we make preliminary progress on a special subclass of separable strongly convex functions. First, define a class of univariate functions
\[
\mathcal{G}_{\mu,L} := \{g: \text{$g$ is $\mu$-strongly convex, has $L$-Lipschitz gradients, and $g'''(x) \geq 0$ for all $x$}\}.
\]
Equivalently, the second derivative $g''$ satisfies $\mu \leq g''(x) \leq L$ for all $x$ and $g''$ is non-decreasing. Using the above as building blocks, define
\[
\mathcal{F}_{\mathrm{sep}, \mathcal{G}_{\mu,L}} := \{f : f(\x) = \sum_{i=1}^d g_\indi(\x_\indi) \text{ where } g_\indi \in \mathcal G_{\mu,L} \}.
\]
% where each $g_\indi$ is $\mu$-strongly convex, has $L$-Lipschitz gradients, and has non-negative third derivative: $g_\indi'''(x) \geq 0$ for all $x$. 

\rev{To motivate the above definition, consider the quadratic function $f_1$ given by $f_1(\x)= \sum_{i=1}^d i\x_{(i)}^2$ and the nearly quadratic function $f_2$ given by $f_2(\x)= \sum_{i=1}^d (i\x_{(i)}^2 + 0.1\x_{(i)}^4)$. $f_1$ has $\mu =2$ and  $L = 2d$. Note that $f_2$ also has $\mu  =2$, and if we restrict it to a bounded domain $[-B, B]^d$, it has $L = (2 + 1.2B^2)d$. We expect that the path length of $f_2$ behaves like $f_1$; however the only applicable path length bound we know is the $\mathcal{O}(\sqrt{\kappa})$ bound of Section~\ref{sec:PKL}. Note however that $f_2 \in \mathcal F_{\mathrm{sep}}$ and so Theorem~\ref{thm:upperSpecialSC} below shows an $\mathcal{O}(\log\kappa)$ bound on its path length.} %The proof technique used for this class of functions is directly inspired from what we learned in the analysis of quadratic objective functions. 

\begin{theorem}
\label{thm:upperSpecialSC}
For any $f \in \mathcal F_{\mathrm{sep}, \mathcal{G}_{\mu,L}}$, the \GF path length is bounded as: 
\[\pathlength \leq (2 + \log \kappa)\ \dist{\x_0}{\X^*}.\]
\end{theorem}
\noindent The proof of this theorem can be found in Appendix~\ref{appsec:proofsQuadratics}. 

\section{Dimension Dependent Path Length Bounds Under Convexity}
\label{sec:quasiconvex}
If our function class does not exhibit linear convergence, path length bounds can still be provided that depend on the dimension $d$. In this section, we analyze path lengths of \GD and \GF under convexity. In fact, the results of this section hold for \GF under the weaker assumption of quasiconvexity (Definition~\ref{def:qc}). Under quasiconvexity or convexity, even finiteness of path length is a surprising result since there exist planar convex functions whose \GF curves spiral around infinitely many times while going arbitrarily close to the minimum \citep{daniilidis2010asymptotic}. \rev{In other words, \GF exhibits convergence but not the stronger notion of tangential convergence \citep[Section 5.5]{bolte2020curiosities}.} Since there is no natural notion of a condition number here, we look for bounds that depend on the dimension $d$. The analysis of path lengths of \GF and \GD in the convex case goes via a reduction to the notion of self-contracted curves. 

\begin{definition}[Self-contracted curve \citep{daniilidis2010asymptotic}]\label{def:selfcontracted}
A curve $g: S \to \Real^d$ is self-contracted if for all $s_1, s_2, s_3 \in S $ such that $s_1 \leq s_2 \leq s_3$, 
\begin{equation}
    \enorm{g(s_3) - g(s_2)} \leq \enorm{g(s_3) - g(s_1)}. \label{eq:descent-condition}
\end{equation}
\end{definition}
\begin{figure}[t]
    \centering
    \includegraphics[width=0.45\linewidth,trim=4.5cm 8cm 2.5cm 8cm,clip=true]{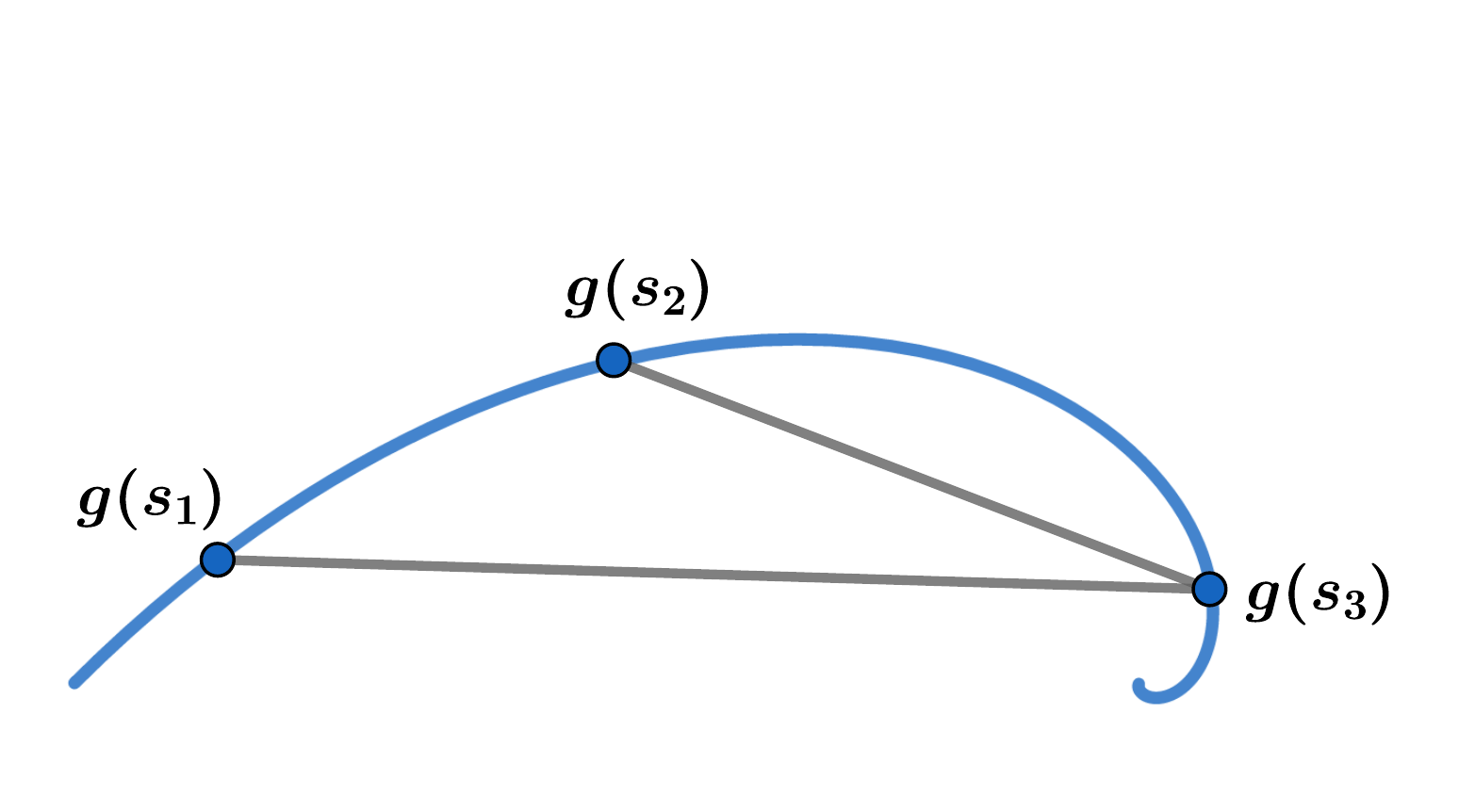}
    \includegraphics[width=0.45\linewidth,trim=2.5cm 9cm 4cm 8cm,clip=true]{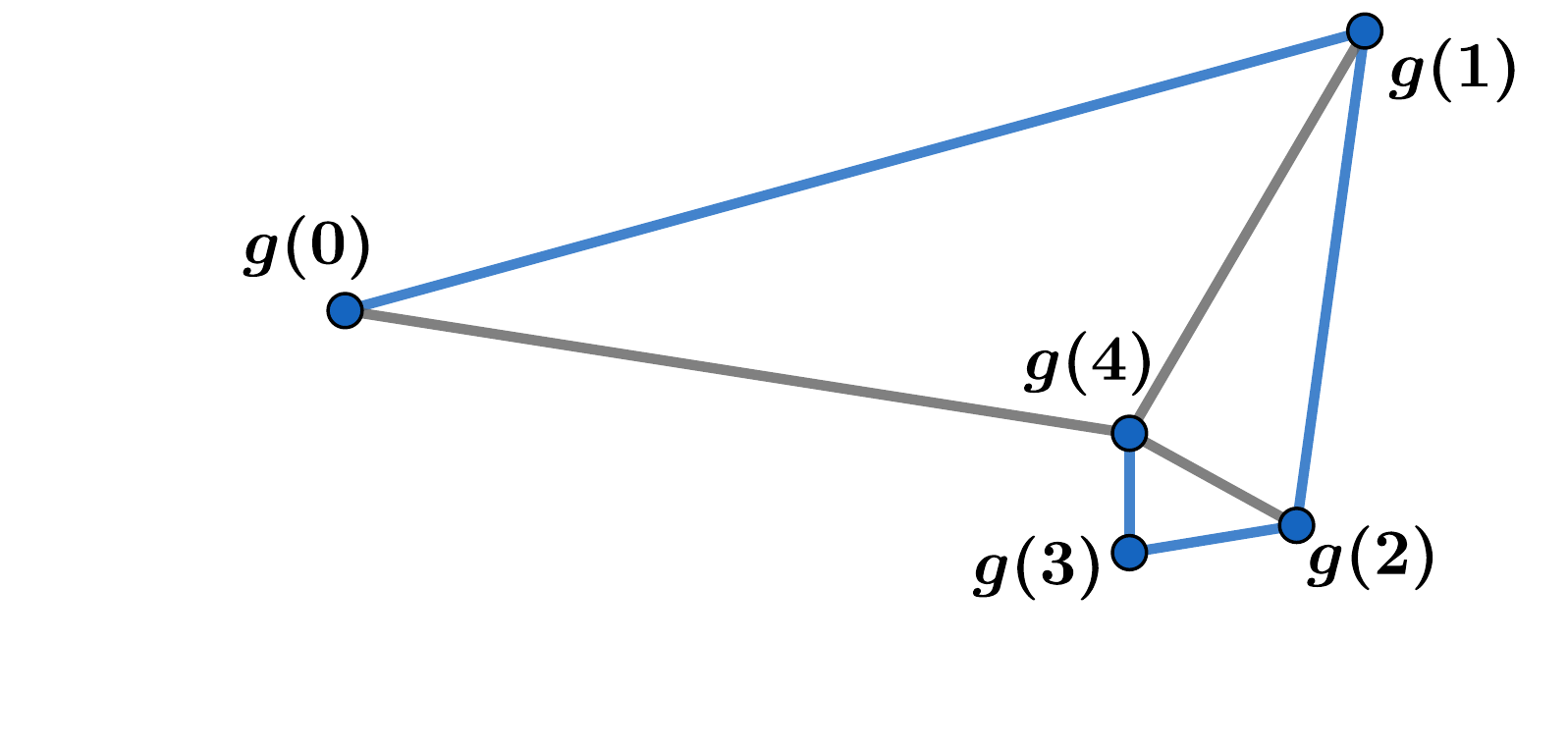}
    \caption{Self-contracted curves. In the continuous case (left), $S = \Real_0^+$; in the discrete case (right), $S = \naturals_0$.}
    \label{fig:selfcontracted}
\end{figure}
\rev{Setting $s_3$ such that $g(s_3) = \x^*$, we can see that self-contracted curves are descent curves, \revminor{in the sense that consecutive iterates cannot go farther away from $\x^*$}. However self-contracted curves require the descent condition \revminor{\eqref{eq:descent-condition}} to hold more generally for any $s_1 \leq s_2 \leq s_3$.} Figure~\ref{fig:selfcontracted} illustrates a self-contracted curve in two dimensions. 

It is well known that the \GF curve is a self-contracted curve \citep{manselli1991maximum, daniilidis2015rectifiability, daniilidis2010asymptotic} for quasiconvex functions. To see this, first note that for any~$t$, $\frac{d{f(\x_t)}}{dt} = -\esqnorm{\nabla f(\x_t)} \leq 0$, and thus for any $s \geq t$, $f(\x_s) \leq f(\x_t)$.  Now, fix $s$ and define the potential function $\lyapunov(t) = \esqnorm{\x_s - \x_t}$ for $t \leq s$. Then, 
\begin{align*}
    \vel{\lyapunov}(t) &= 2\inner{\x_t - \x_s, \vel{\x}_t}
    = 2\inner{\x_t - \x_s, -\nabla f({\x}_t)}
    = 2\inner{\nabla f({\x}_t), \x_s - \x_t} \leq~ 0,
\end{align*}
where the inequality follows by quasiconvexity since $f(\x_s) \leq f(\x_t)$. Thus, $\esqnorm{\x_s - \x_t}$ is a decreasing function of $t$ which is the same as self-contractedness for $s_3 = s$.

For \GD, we prove self-contractedness under the additional assumptions of convexity and Lipschitz gradients: 
\begin{lemma}\label{lemma:QCGD}For any convex function $f$ with $L$-Lipschitz gradients, the \GD curve with $\eta \leq 1/L$ is self-contracted. \end{lemma}
The proof of Lemma~\ref{lemma:QCGD} is in Appendix~\ref{appsec:gd-self-contracted-proof}. We do not know if one can relax convexity to quasiconvexity. \rev{Also note the step-size restriction: for convex functions with $L$-Lipschitz gradients, \GD with $\eta \in (0, 2/L]$ is a descent method; however for \GD to be self-contracted, further restriction on step-size (such as $\eta \in (0, 1/L]$) is needed. This is discussed in Appendix~\ref{appsec:gd-self-contracted-proof} after the proof of the lemma. Thus while self-contracted curves are descent curves \revminor{(in terms of the iterates)}, there exist descent curves that are not self-contracted. }

The \GD curve in Lemma~\ref{lemma:QCGD} refers to the iterates $g(0), g(1), \ldots$, and not the affine extension of the iterates (obtained by connecting consecutive iterates by a line). It is unclear whether the affine extension itself is self-contracted. This precludes a direct application of path length bounds known for \GF curves~\citep{manselli1991maximum, daniilidis2015rectifiability} since these bounds require all points to be part of a self-contracted curve. Despite this limitation, we show that the self-contractedness guarantee provided by Lemma~\ref{lemma:QCGD} is enough to show a path length bound for \GD.

\begin{theorem}\label{thm:QC}
For any quasiconvex function $f$, the \GF  path length is bounded as: 
\[\pathlength \leq 2^{2d\log d} \enorm{\x_0 - \x_\infty}.\]
If $f$ is convex with $L$-Lipschitz gradients, then the \GD iterates with a step-size $\eta \leq 1/L$ admit a path length bound: 
\[\pathlength_\eta \leq 2^{10d^2} \enorm{\x_0 - \x_\infty},\]
while the \GD iterates with a step-size $\eta \leq 1/2L\sqrt{d}$ admit a path length bound: 
\[\pathlength_\eta \leq 2^{4d\log d} \enorm{\x_0 - \x_\infty}.\]
\end{theorem}
The proof of Theorem~\ref{thm:QC} can be found in Appendix~\ref{appsec:proofsQuasiconvex}. \rev{The \GF bound is due to~\citet{manselli1991maximum}, while our contribution is the analysis for \GD curves (nevertheless, we include the \GF proof for completeness). } 
To the best of our knowledge, ours is the only path length bound known for \GD curves of convex \LG functions (without further assumptions). However the step-size restriction of $\eta \leq 1/2L\sqrt{d}$ for the $2^{\mathcal{O}(d\log d)}$ result is much smaller than the usual step-sizes required for convergence to hold. It would be interesting to study if the $2^{\mathcal{O}(d\log d)}$ can be obtained with $\eta = \mathcal{O}(1/L)$. 

Observe that this bound is with respect to $\enorm{\x_0 - \x_\infty}$ instead of $\dist{\x_0}{\X^*}$. For convex functions $f$ with $L$-Lipschitz gradients, it is known that \GD or \GF both converge to a point in the optimal set, that is, $\x_\infty \in \X^*$. Thus if $\X^*$ is singleton, $ \x_\infty = \x^*$. However in general,
$\x_\infty$ may be distinct from $\Pi_{\X^*}(\x_0)$ and $\enorm{\x_0 - \x_\infty}$ may be larger than $\dist{\x_0}{\X^*}$.

The exponential bound of Theorem~\ref{thm:QC} can be significantly improved if the quasiconvex function is separable, that is it exhibits the decomposition
\begin{equation}
f(\x) = \sum_{i=1}^d g_{(i)}(\x_{(i)}), 
\label{eqn:decomposable}    
\end{equation}
for some functions $g_{(i)}:\Real \to \Real$. Note that if $f$ is quasiconvex, each $g_{(i)}$ is quasiconvex. 
\begin{theorem}
\label{thm:decomposable}
Suppose $f$ is quasiconvex and exhibits the decomposition~\eqref{eqn:decomposable}. Then the path length of \GF is bounded as $\pathlength \leq \sqrt{d}\ \dist{\x_0}{\X^*}$. If $f$ has $L$-Lipschitz gradients then the path length of \GD with $\eta \leq 1/L$ also satisfies $\pathlength_\eta \leq \sqrt{d}\ \dist{\x_0}{\X^*}$.
\end{theorem}

The proof of this theorem can be found in Appendix~\ref{appsec:proofsQuasiconvex}. The decomposition~\eqref{eqn:decomposable} ensures that \GD/\GF always follows a descend direction in each of the components. The theorem generalizes easily to a larger class of functions that exhibit component-wise descent for any orthogonal basis (and not necessarily the canonical basis). It would be interesting to study if this can be shown for some standard class of functions larger than separable quasiconvex functions. Theorem~\ref{thm:PLlowerBoundQuadratic} shows a matching $\Omega(\sqrt{d})$ lower bound for separable quasiconvex functions. % In Section~\ref{sec:QuadraticLowerBound}, we will prove a lower bound for path lengths in the quadratic case with a $\Omega(\sqrt{d})$ dependence on the dimension. The function constructed there is separable in the components, and quasiconvex since it is quadratic. Hence we have matching upper and lower bounds for separable quasiconvex functions. 

\section{Lower Bounds}
\label{sec:lower-bounds}
In this section we provide lower bounds on the path length for quadratic functions, \PKL functions, and separable quasiconvex functions. In each case, given problem parameters ($d$, $\kappa$), we construct a worst-case lower bound---that is we exhibit a function $f$ that satisfies the problem parameters and specify an initial point $\x_0$ for which the path length is lower bounded by some function of $d$ and $\kappa$, times the length of the shortest path $\dist{\x_0}{\X^*}$.

\subsection{An $\boldsymbol{\widetilde\Omega(\sqrt{d} \wedge \kappa^{1/4})}$ Lower Bound for \PKL Functions}
\label{sec:PLlowerBound}
In Section~\ref{sec:quadratics} we obtained an $\mathcal{O}(\sqrt{\log\kappa})$ dependence for the path length of quadratics objectives. Thus, a natural question is whether the $\mathcal{O}(\sqrt{\kappa})$ bound for the path length of PKL objectives can be improved to $\mathcal{O}(\text{polylog}(\kappa))$. In this section, we show that such a dependence is precluded for \PKL functions without further assumptions. Previously,~\citet[Theorem~5.4]{oymak2019overparameterized} have presented a lower bound in terms of $f(\x_0) - f^*$. However, this bound when translated in terms of $\dist{\x_0}{\X^*}$ leads to a trivial result. Theorem~\ref{thm:PLlowerBound} is the first non-trivial lower bound for functions that satisfy \PKL. The constructed function also satisfies linear convergence, leading to a lower bound for linearly convergent functions as well. Let $\fancyF_{\kappa}$ be the class of real-valued functions on $\Real^d$ such that every $f \in \fancyF_\kappa$ satisfies: 
\begin{itemize}
    \item $f$ is continuously differentiable.
    \item There exist constants $\mu, L > 0$ such that $\kappa \geq L/\mu$ and a) $f$ has $L$-Lipschitz gradients, b) $f$ satisfies the \PKL inequality with constant $\mu$. 
\end{itemize}
\begin{theorem}
\label{thm:PLlowerBound}
For every $d \geq 6$ and $\kappa \geq 216$, there exists a function $f \in \fancyF_{\kappa}$ and an initial point $\x_0$ such that the \GF dynamics on $f$ with the initial point $\x_0$ satisfies
\[
\pathlength \geq \min\curlybrack{\frac{\sqrt{d}}{6\log d}, \frac{{\kappa}^{1/4}}{6 \log \kappa}}\ \dist{\x_0}{\X^*}.
\]
Similarly, there exists a function $f \in \fancyF_{\kappa}$, an initial point $\x_0$, and some step-size $\eta \in [\nicefrac{1}{2L}, \nicefrac{1}{L}]$ such that the \GD iterates on $f$ with the initial point $\x_0$ satisfy 
\[
\pathlength_\eta \geq \min\curlybrack{\frac{\sqrt{d}}{16\log d}, \frac{{\kappa}^{1/4}}{16 \log \kappa}}\ \dist{\x_0}{\X^*}.
\]
\rev{The same construction guarantees the following: for any $c \in (0, 5.8\cdot 10^{-3})$, there exists a function $f \in \fancyF_\kappa$ and a point $\x_0$, such that $f$ satisfies $(1, c)$-linear convergence (for \GF and \GD) and the path length with the initial point $\x_0$ can be bounded as 
\[
%\zeta \geq \roundbrack{\frac{c^{-0.5}}{12\log^{1.5}(c^{-1})}}\dist{\x_0}{\X^*}; \ \ \zeta_\eta \geq \roundbrack{\frac{c^{-0.5}}{64\log^{1.5}(c^{-1})}}\dist{\x_0}{\X^*},
\zeta \geq \roundbrack{\frac{\sqrt{1/c}}{12\log^{1.5}(1/c)}}\dist{\x_0}{\X^*}; \ \ \zeta_\eta \geq \roundbrack{\frac{\sqrt{1/c}}{64\log^{1.5}(1/c)}}\dist{\x_0}{\X^*},
\]
for $\GF$ and $\GD$ respectively. }
\end{theorem}
\begin{proofsketch}The function $f$ that we construct decomposes as $f(\x) = \sum_{i=1}^d g(\x_{(i)})$, where the function $g$ (Figure~\ref{fig:PLlowerBoundMain}) is $L$-Lipschitz and $\mu$-\PKL (and thus so is $f$) with $\kappa \geq L/\mu$. $g$ is designed so that it is equal to $x^2$ in the interval $[0, 0.5]$, so that it is strongly convex in that region. In $[0.5, 1]$, $g$ is not strongly convex (or convex) and in some sense tapers off. However, $g$ continues to maintain the PKL curvature condition with some constant $\mu$ globally. Next we stagger the components of the initial point $\x_0$ so that at every consecutive time interval, a single component starts has value $0.5$ at the beginning of the time interval and decreases to almost $0.0$ at the end of the time interval (Figure~\ref{fig:PKLiterates}). In this way, at every time interval, a single additional component is captured. Loosely speaking, \GD while optimizing approximately follows the edges of a cube instead of the diagonal. This ensures that the path length is a factor $\approx 0.5\sqrt{d}$ larger than the shortest path. Then, we compute $\kappa$ and relate it to $d$ to obtain the final bound. Similarly, we show that the function satisfies $(1,c)$-linear convergence and we relate the linear convergence constant $c$ to the dimension $d$ to obtain the linear convergence result. See Appendix~\ref{appsec:proofsLowerBound} for details.
\end{proofsketch}

These lower bounds do not match the $O(\sqrt{\kappa})$ upper bound for \PKL functions and the $O(1/c)$ upper bound for linearly convergent functions. We do not know which of these bounds are tight. \rev{In Appendix~\ref{appsec:PL-simulation} we simulate the lower bound constructed in the proof of Theorem~\ref{thm:PLlowerBoundQuadratic} with \GD. This simulation allows us to verify that the dependence of the path length of the lower bound construction is indeed $\Omega(\kappa^{1/4}/\log\kappa)$ (and not something larger like $\Omega(\sqrt{\kappa})$). We observed that the dependence is $\approx 3\kappa^{1/4}/\log \kappa$. Thus the dependence on $\kappa$ is correct, but the constants are loose. } Observe that the function constructed above is not strongly convex. Thus proving a $\text{poly}(\kappa)$ lower bound for strongly convex functions remains an open problem. 

\begin{figure}[t]
    \centering
    \includegraphics[width=\linewidth,trim=0cm 5cm 0cm 6cm, clip]{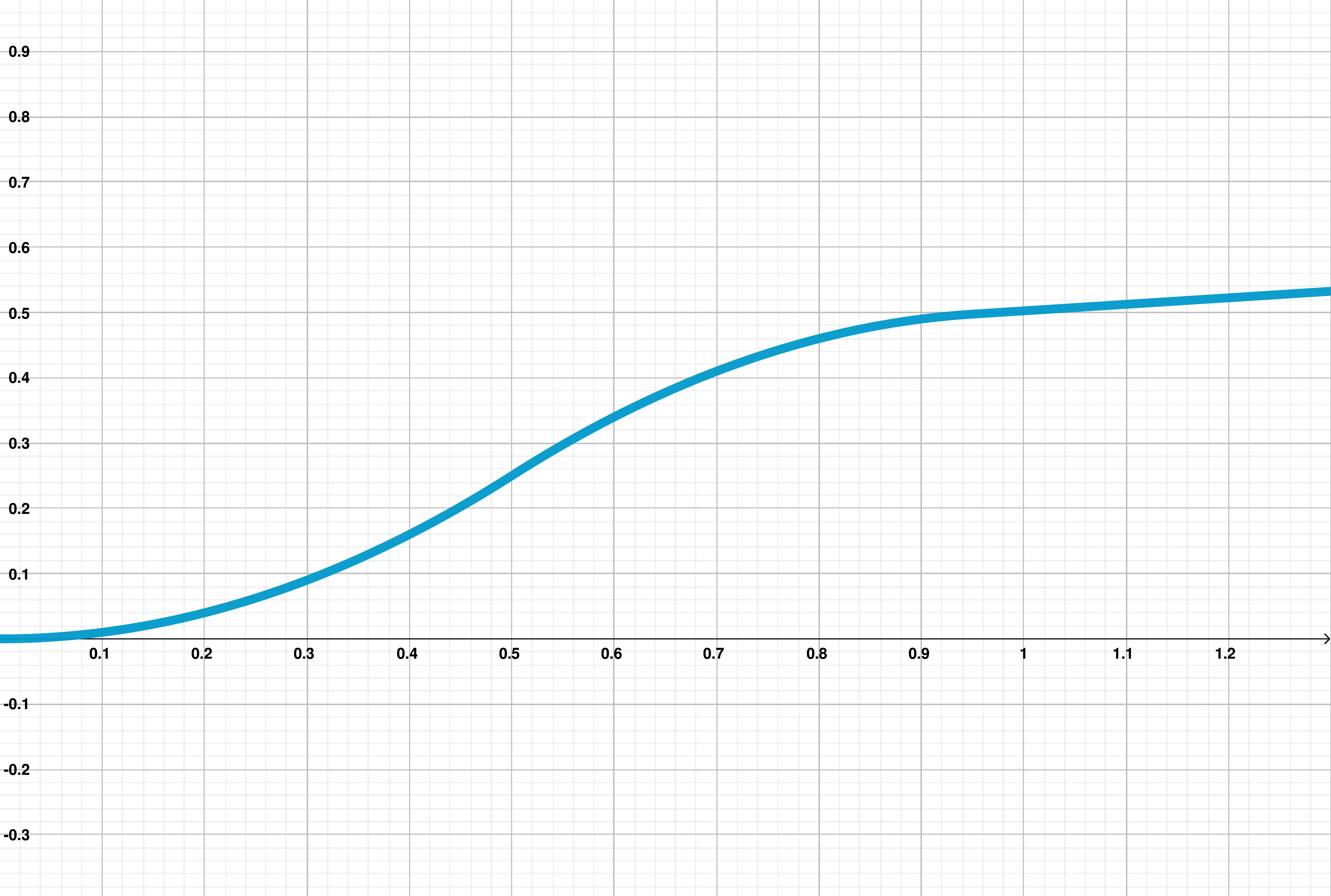}
    \caption{Component function $g$ for \PKL lower bound.}
    \label{fig:PLlowerBoundMain}
\end{figure}
\begin{figure}[t]
    \centering
    \begin{subfigure}[b]{0.32\linewidth}
    \includegraphics[width=\linewidth,trim=0cm 4.5cm 4cm 5.5cm, clip]{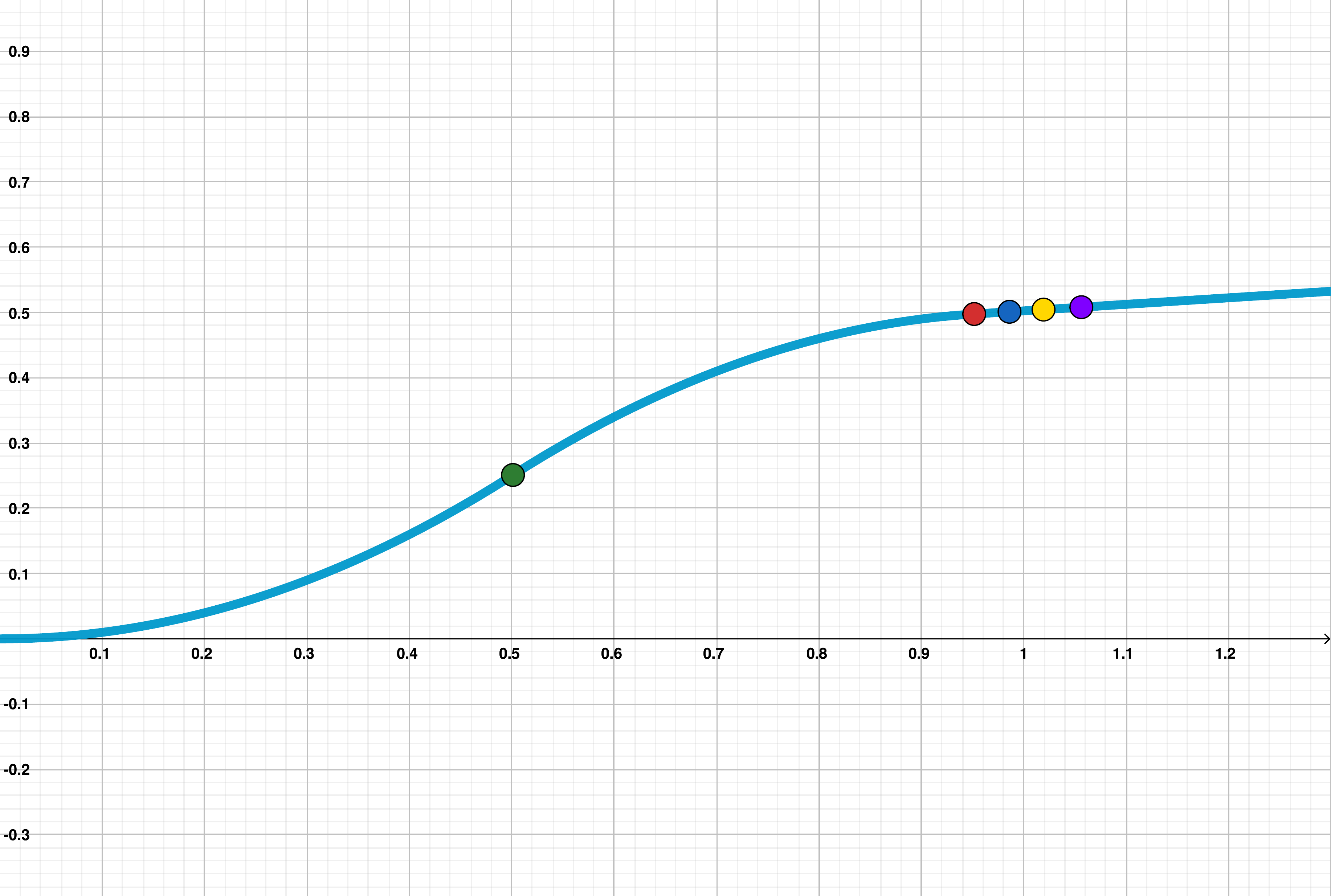}
    \hfill
    \caption{$k = 0$}
    \end{subfigure}
    \begin{subfigure}[b]{0.32\linewidth}
    \includegraphics[width=\linewidth,trim=0cm 4.5cm 4cm 5.5cm, clip]{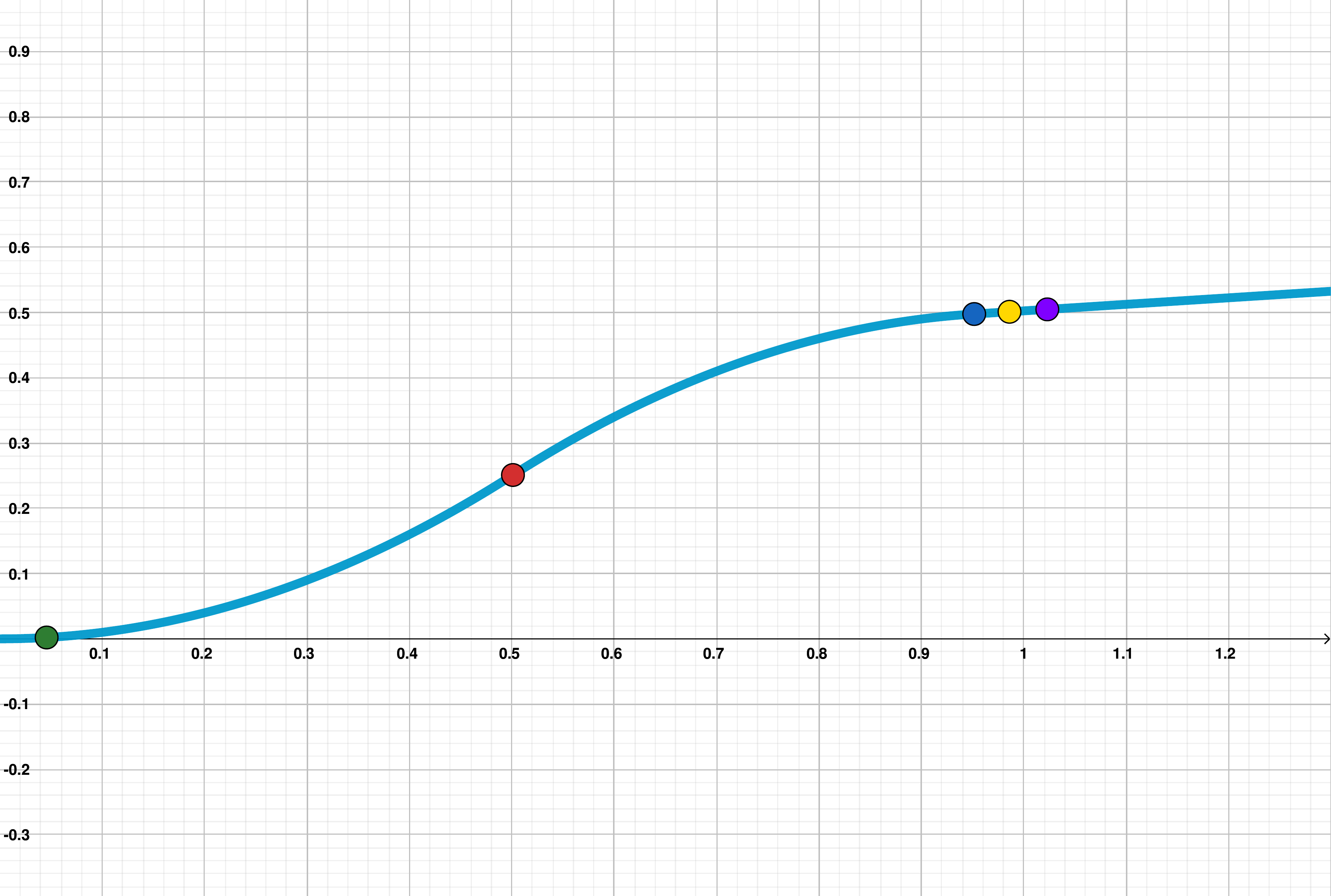}
    \hfill
    \caption{$k = 1$}
    \end{subfigure}
    \begin{subfigure}[b]{0.32\linewidth}
    \includegraphics[width=\linewidth,trim=0cm 3.5cm 3.7cm 5.5cm, clip]{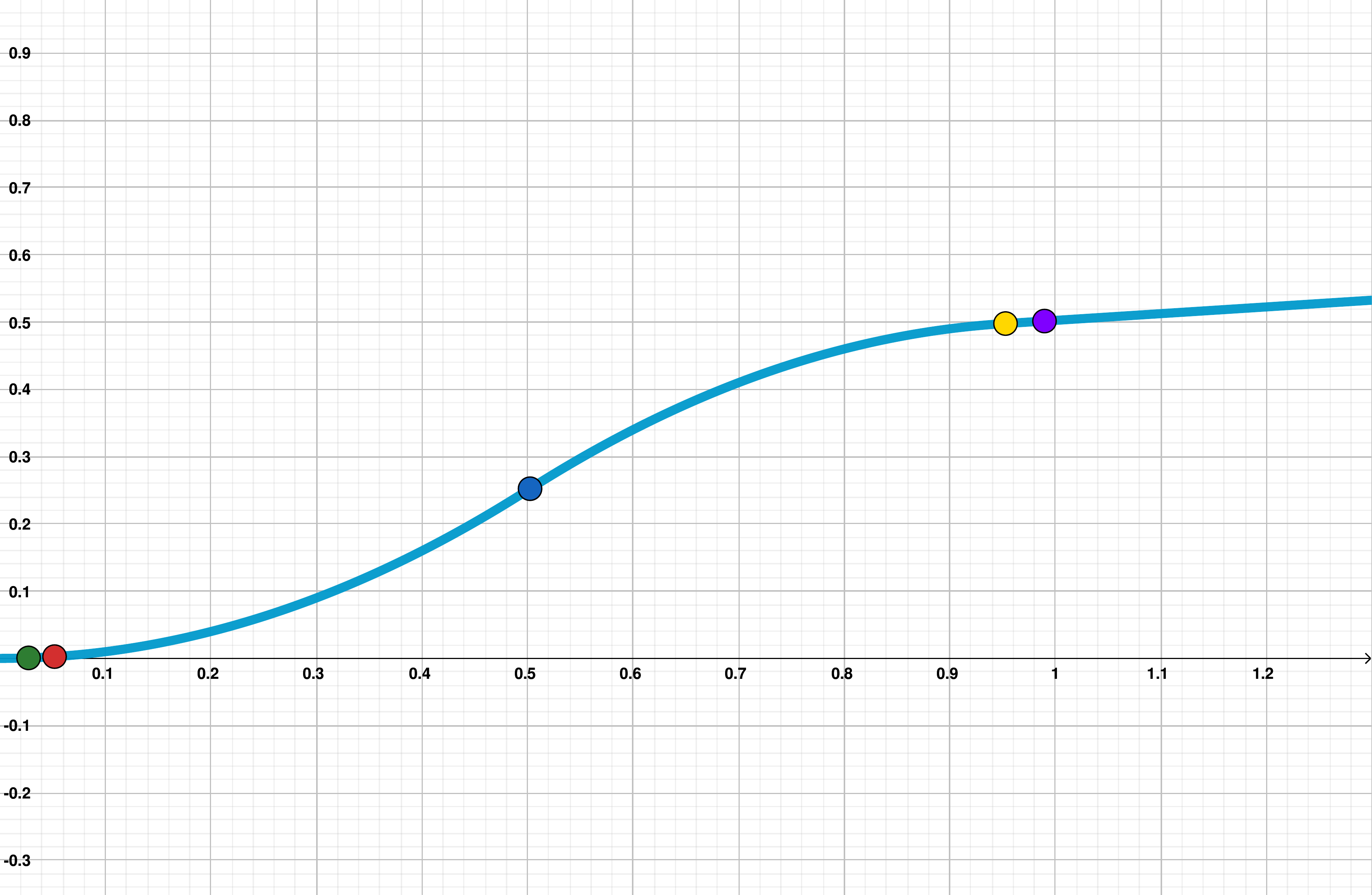}
    \vspace*{0mm}
    \caption{$k = 2$}
    \end{subfigure}
    \caption{Illustration of lower bound path length construction of Theorem~\ref{thm:PLlowerBound} for \GD. Every colored circle denotes the value of $\x_k$ in a different component. In iterate $k = 1$, the green component goes from $0.5$ to almost $0.0$. Then in iterate $k = 2$ the red components decrease in the same manner and so on.}
    \label{fig:PKLiterates}
\end{figure}

\subsection{An $\boldsymbol{\Omega(\sqrt{d} \wedge \sqrt{\log \kappa} )}$ Lower Bound for Quadratics}
\label{sec:QuadraticLowerBound}
In this section, we show that the upper bound for quadratic objectives proved in Section~\ref{sec:quadratics} is tight by constructing an instance of \GD and \GF where the path length is $\Omega(\sqrt{\log \kappa})$, if the dimension can be set arbitrarily \rev{(this is the `large-scale' optimization assumption made in many lower bounds; for instance, see \citet{guzman2015lower} and \citet[Theorem 2.1.7]{nesterov2013introductory}.} Let $\fancyQ_\kappa$ be the class of quadratic functions on $\Real^d$ such that the Hessian has non-negative singular values and the ratio of the largest and smallest non-zero singular values is at most $\kappa$. 
\begin{theorem}
\label{thm:PLlowerBoundQuadratic}
For every $\kappa \geq 5$, there exists a quadratic function $f \in \fancyQ_{\kappa}$ and an initial point $\x_0$ such that the \GF dynamics on $f$ with the initial point $\x_0$ satisfies
\begin{equation}
\pathlength \geq \min\curlybrack{0.7\sqrt{d}, 0.45\sqrt{\log \kappa}}\ \dist{\x_0}{\X^*}.
\label{eqn:quadraticLowerGF}    
\end{equation}
Similarly, there exists a function $f \in \fancyQ_{\kappa}$ and an initial point $\x_0$, such that for step-size $\eta = \nicefrac{1}{2L}$ the \GD iterates on $f$ with the initial point $\x_0$ satisfy 
\begin{equation}
    \pathlength_\eta \geq \min\curlybrack{0.5\sqrt{d}, 0.3\sqrt{\log \kappa}}\ \dist{\x_0}{\X^*}.
    \label{eqn:quadraticLowerGD}
\end{equation}
The quadratic functions constructed are separable (see \eqref{eqn:decomposable}) and convex. Thus the $\Omega(\sqrt{d})$ bound holds for separable quasiconvex functions and convex functions as well. 
%\item The $\Omega(\sqrt{\log\kappa})$ bounds leads to a $\Omega(\sqrt{\log(1/c)})$ bound in terms of the linear convergence factor $c$.
%The quadratic functions constructed are separable (see \eqref{eqn:decomposable}) and satisfy linear convergence with $(A, c) = (1, 1/\kappa)$ (see \eqref{eqn:linConv}). Thus we have the following consequences of this result: 
%\begin{enumerate}[label=(\alph*)]
%\item The $\Omega(\sqrt{d})$ bound holds for separable quasiconvex functions and convex functions. 
%\item The $\Omega(\sqrt{\log\kappa})$ bounds leads to a $\Omega(\sqrt{\log(1/c)})$ bound in terms of the linear convergence factor $c$.
%\end{enumerate}

\end{theorem}
\begin{proofsketch}
The quadratic function we consider has geometrically increasing spectra---the eigenvalues are $1, \omega, \omega^2, \ldots \omega^{d-1}$. Similar to the proof of Theorem~\ref{thm:PLlowerBound} for the \PKL case, in each time interval (or iterate), a single component is captured. This leads to a lower bound of $\Omega(\sqrt{d})$. We then relate $\sqrt{d} = \Theta(\sqrt{\log \kappa})$ to write the final bound. Details can be found in Appendix~\ref{appsec:proofLowedQuadratic}. 
\end{proofsketch}
In Appendix~\ref{appsec:quadratic-simulation} we simulate the quadratic lower bound constructed in the proof of Theorem~\ref{thm:PLlowerBoundQuadratic} and compare the lower bound and upper bound (Theorem~\ref{thm:QUADGF}) with the empirically observed path length. 

\section{Discussion}\label{sec:discussion}
Bounds on the path length of \GD (and related algorithms like \SGD) have implicitly been studied in some recent papers seeking to understand optimization for deep neural networks. 
In this paper, we provide unified results for \GD and \GF under common smoothness and curvature assumptions, obtaining the tightest known bounds for quadratics and convex objectives. We also presented a meta-theorem that gives us a path length bound for any linearly convergent iterative algorithm. To complement these results, for \PKL objectives, we also show a lower bound, which to our knowledge is the first such lower bound for path lengths. For quadratics, we give a matching lower bound thus completely characterizing path lengths in this setting. For separable quasiconvex objectives, we give matching (up to constants) upper and lower bounds on the path length.  

Our meta-theorem suggests that path lengths are intricately tied with convergence properties. While we focused on GD and GF, it would be interesting to consider the path length bounds exhibited by other optimization algorithms that have good convergence properties, such as SGD, projected gradient descent, accelerated methods, heavy ball methods, second order methods, proximal methods, and so on. In Appendix~\ref{appsec:additionalLC}, we show some preliminary results to this end. \revminor{\citet[Fact 4]{attouch2016rate} showed a path length bound style result for Nesterov's acceleration method.} %We extend the linear convergence ideas of Section~\ref{sec:linconv} to prove path length bounds for Polyak's heavy ball method and for projected gradient descent. 

A broader open direction concerns understanding better the statistical implications of our path length bounds. As an example, it is known that stable optimization algorithms can exhibit better generalization guarantees~\citep{bousquet2002stability,chen2018stability,hardt2016train} and it is natural to expect similar qualitative behaviour from optimization algorithms that have short path lengths. Some recent works~\citep{balakrishnan2017statistical,mei2018landscape}, have provided statistical guarantees for optimization-based estimators via uniform (statistical) convergence over the possible algorithm iterates, and we believe these might also be strengthened by a deeper understanding of the path followed by these algorithms.

While the focus of this paper is on path length bounds for \GD and \GF, we also note other problems for which path length bounds have been considered. Bounds on the $\ell_1$ path length of lasso and forward stagewise regression have been studied by~\citet{hastie2007forward}. Adaptive regret bounds for bandits have been shown that depend on the total path length of the losses at each step~\citep{wei2018more, pmlr-v99-bubeck19b}. \citet{argue2019nearly} showed a result for nested convex body chasing using a bound on the path length of self-contracted curves~\citep{manselli1991maximum}, which was originally developed for proving the convergence of \GF for any quasiconvex function.

\acks{SB is grateful to Marco Molinaro for inspiring discussions. We thank S\'{e}bastien Bubeck for pointing us to some references on self-contracted curves, and Anupam Gupta for informing us about the nested convex body chasing problem that sparked our interest in path length bounds. We thank the anonymous JMLR reviewers and action editor Ambuj Tewari for their comments that helped improve an earlier version of this paper. SB and CG were supported in part by the NSF grant DMS-1713003. All figures in the paper were created using GeoGebra~\citep{gg2}.
}

\bibliographystyle{plainnat}
\bibliography{19-979}

\newpage
\appendix 
\part{Appendix}
\parttoc
\newpage
\section{Proofs and Additional Results from Section~\ref{sec:linconv}}
\label{appsec:additionalLC}
In this section, we state and prove more general versions of Theorem~\ref{thm:LCGD} and Theorem~\ref{thm:LCGext}. These provide a path length bound if linear convergence can be established with respect to a local minimum rather than the global minimum. We also show path length bounds for Polyak's heavy ball method and projected gradient descent. 

\subsection{Generalized Version of Theorem~\ref{thm:LCGD}}
%As claimed in Remark~\ref{remark:linConv}, 
Theorem~\ref{thm:LCGD} can be generalized for linear convergence to any set $\widehat{\X}$ as long as it is stationary, defined next. %we give a bound on the path length as long as the local minimum is stationary.

\begin{definition}[Stationary convex set]
A stationary convex set is a convex set $\widehat\X$ such that for every $\widehat\x \in \widehat\X$, $\nabla f(\widehat\x) = 0$. 
\end{definition}

Given any convex set $\widehat{\X}$, we can generalize Definition~\ref{def:linConvOptimal} to allow linear convergence to this set instead of the globally optimal set $\X^*$. Consider the following modification of condition~\eqref{eqn:linConv} replacing $\X^*$ with $\widehat\X$:
\[
    \dist{\x_s}{\widehat\X}  \leq A(1 - c)^s\ \dist{\x}{\widehat\X}.
\]
Here $s$ may belong to $\wholes$ (discrete-time) or $\Real_0$ (continuous-time). %We now state and prove general versions of Theorem~\ref{thm:LCGD} and Theorem~\ref{thm:LCGext}. %We are now ready to state our most general result.

\begin{theorem}
    \label{thm:LCGDextended}
    Suppose $f$ has $L$-Lipschitz gradients. If the \GF dynamics for $f$ exhibits linear convergence towards a stationary convex set $\widehat\X$ with constants $(A, c)$, then its path length is bounded as:
    \begin{equation}
        \label{eqn:lcgfapp}
        \pathlength \leq (AL/\log\roundbrack{1/(1-c)})\ \dist{\x_0}{\widehat\X},
    \end{equation}
    and if the \GD iterates with step-size $\eta$ exhibit linear convergence towards $\widehat\X$ with constants $(A, c)$, then their path length is bounded as: 
    \begin{equation}
        \label{eqn:lcgdapp}
        \pathlength_\eta \leq (\eta AL/c)\ \dist{\x_0}{\widehat\X}.    
    \end{equation}
\end{theorem}

\begin{proof} We first prove the \GF bound~\eqref{eqn:lcgfapp}. Using \LG and the fact that for every $\widehat{\x} \in \widehat{\X}$, $\nabla f(\widehat\x) = 0$, we bound the path length increment at every instance $t$: \[\enorm{\vel{\x}_t}dt = \enorm{\nabla f(\x_t)}dt =\enorm{\nabla f(\x_t) - \nabla f(\Pi_{\widehat{\X}}(\x_t))}dt  \leq L\ \dist{\x_t}{\widehat\X} dt.\] Thus,
\begin{align*}
    \int_0^\infty \enorm{\vel{x}_t} \ dt &\leq \int_0^\infty AL(1 - c)^t\ \dist{\x_0}{\widehat\X} \ dt
    \\ &= \int_0^\infty ALe^{t\log(1-c)}\ \dist{\x_0}{\widehat\X} \ dt
    \\ &= \roundbrack{\frac{AL}{\log\roundbrack{1/(1-c)}}} \ \dist{\x_0}{\widehat\X}.
\end{align*}
This concludes the proof of claim~\eqref{eqn:lcgfapp}.

For claim~\eqref{eqn:lcgdapp}, using \LG we have the following regularity condition on the distance travelled at every step:
\begin{align*}
    \enorm{\x - \x^+} %&= \enorm{\x - \Pi_\convexSet(\x - \eta \nabla f(\x))} 
    %\\ &\textleq{(i)} \enorm{\x - (\x -\eta  \nabla f(\x))}  
    &= \enorm{\x - (\x -\eta  \nabla f(\x))}  
    \\ &= \enorm{\eta \nabla f(\x)} 
    \\ &= \eta\enorm{ \nabla f(\x) - \nabla f(\Pi_{\widehat\X}(\x))} 
    \\ &\textleq{\LG} \eta L\ \dist{\x}{\widehat\X}.
\end{align*}
%Above, $\xi$ follows since for a convex set $\convexSet$, $\enorm{\Pi_\convexSet(\x) - \Pi_\convexSet(\y)} \leq \enorm{\x - \y}$, and $\x = \Pi_\convexSet(\x)$. 
The equality on the second last line follows by the stationarity assumption on $\widehat\X$. Thus we have the following bound on the overall path length:
\begin{align*}
    \pathlength_\eta &= \sum_{k=0}^\infty \enorm{\x_k - \x_{k+1}}
    \\ &\leq \sum_{k=0}^\infty \eta L\ \dist{\x}{\widehat\X}
    \\ &\leq \sum_{k=0}^\infty \eta AL\ (1 - c)^k\dist{\x_0}{\widehat\X}
    \\ &= \roundbrack{\frac{\eta AL}{c}} \ \dist{\x_0}{\widehat\X},
\end{align*}
completing the proof.

\end{proof}

%The \PGD bound of Theorem~\ref{thm:LCGDextended} assumes convergence to a stationary convex set, that is for $\widehat{\x} \in \widehat{\X}$, $\nabla f(\widehat{\x}) = 0$. For a general constraint set $\Omega$, even if linear convergence holds to this set, the optimal points in the set need not satisfy $\nabla f(\widehat{\x}) = 0$, in which case Theorem~\ref{thm:LCGDextended} would not apply.
Next, we show a similar generalization of Theorem~\ref{thm:LCGext} for linear convergence to a local minimum $\widehat\X$, but without requiring a stationarity assumption. 

\subsection{Generalized Version of Theorem~\ref{thm:LCGext}}
%In the following theorem we prove a similar result for any iterative algorithm (not just \GD). However here we need to assume that the set of convergence is singleton: $\widehat{\X} = \{\widehat{\x}\}$.
The following result applies for any iterative algorithm (not just \GD) and any function class (not necessarily convex) for which linear convergence to a set $\widehat{\X}$ holds (i.e., Definition~\ref{def:linConvOptimal} with $\X^*$ replaced with $\widehat{\X}$). We additionally assume that descent holds with respect to all points in $\widehat{\X}$. We do not require $\widehat\X$ to consist of stationary points or be a convex set, as long as the conditions specified above hold. %conditions (a) and (b) hold. %can be established for all minimizers. %to any singleton set: $\widehat{\X} = \{\widehat{\x}\}$ (a single minimization set $\X^* = \{\x^*\}$ is a special case). 
\ifx false 
\rev{\begin{theorem}
    Given any $f$, if the iterates for an update rule exhibit linear convergence with constants $(A, c)$ to a singleton set $\widehat\X = \{\widehat\x\}$, then their path length is bounded as: 
    \label{thm:LCGD2extended}
\begin{equation}
        \label{eqn:LCGD2extended}
        \pathlength_\eta \leq \roundbrack{\frac{A + A^2(1 - c)}{c}} \enorm{\x_0 - \widehat\x}.    
    \end{equation}
\end{theorem}
}

\begin{proof}
Using triangle inequality for any two consecutive iterates $\x \to \x^+$, 
\begin{align*}
\enorm{\x - \x^+} &\leq \enorm{\x - \widehat\x} + \enorm{\x^+ - \widehat\x} 
\\ &\leq \enorm{\x - \widehat\x} + A(1 - c)\enorm{\x - \widehat\x} 
\\ &= (1 + A(1 - c)) \enorm{\x - \widehat\x}.
\end{align*}
Thus for any $T \in \naturals$ iterations,
\begin{align*}
    \pathlength_\eta &= \sum_{k = 0}^{T-1}\enorm{\x_k - \x_{k+1}}
    \\ &\leq (1 + A(1 - c)) \sum_{k = 0}^{T-1}\enorm{\x_k - \widehat\x}
    \\ &\leq (1 + A(1 - c)) \sum_{k = 0}^{T-1}A(1 - c)^k\enorm{\x_0 - \widehat\x}
    \\ &\leq \roundbrack{\frac{A + A^2(1 - c)}{c}} \enorm{\x_0 - \widehat\x}.
\end{align*}
Since the above is true for any $T \in \naturals$, indeed
\[
\pathlength_\eta  \leq  \roundbrack{\frac{A + A^2(1 - c)}{c}} \enorm{\x_0 - \widehat\x},
\]
as was to be shown. 
\end{proof}
\fi 
\rev{\begin{theorem}
    If the iterates for a discrete update rule satisfy the following conditions: 
    \begin{enumerate}[label=(\alph*)]
    \item linear convergence with constants $(A, c)$ to a set $\widehat\X$, and 
    \item descent towards all points in $\widehat{\X}$: for all $\widehat\x \in \widehat{\X}$, $\x \in \Real^d$, $\enorm{\x^+ - \widehat\x} \leq \enorm{\x - \widehat\x}$, 
\end{enumerate}     
then their path length is bounded as: 
    \label{thm:LCGD2extended}
\begin{equation}
        \label{eqn:LCGD2extended}
        \pathlength_\eta \leq \roundbrack{\frac{2A}{c}} \dist{\x_0}{\widehat\X}.    
    \end{equation}
\end{theorem}
}

\begin{proof}
Using triangle inequality for any two consecutive iterates $\x \to \x^+$, 
\begin{align*}
\enorm{\x - \x^+} &\leq \enorm{\x - \Pi_{\widehat\X}(\x)} + \enorm{\x^+ - \Pi_{\widehat\X}(\x)} 
\\ &\textleq{(i)} \enorm{\x - \Pi_{\widehat\X}(\x)} + \enorm{\x - \Pi_{\widehat\X}(\x)} 
%\\ &\textleq{(ii)} \enorm{\x - \Pi_{\widehat\X}(\x)} + A(1-c)\enorm{\x - \Pi_{\widehat\X}(\x)} %A(1-c)\enorm{\x^+ - \Pi_{\widehat\X}(\widehat\x^+)}  + \enorm{\x^+ - \Pi_{\widehat\X}(\widehat\x^+)} .
%\\ &= 
%\\ &\leq \enorm{\x - \widehat\x} + A(1 - c)\enorm{\x - \Pi_{\widehat{\X}}(\x)} 
%\\ &= (1 + A(1 - c))\enorm{\x - \Pi_{\widehat\X}(\widehat\x)} %\enorm{\x^+ - \Pi_{\widehat\X}(\widehat\x^+)} .
\\ &= 2\enorm{\x - \Pi_{\widehat\X}(\x)}.
\end{align*}
Above, inequality (i) holds by condition (b) since $\Pi_{\widehat\X}(\x) \in \widehat{\X}$. % the definition of a projection, and $\xi_2$ holds by the linear convergence hypothesis. 
We use this inequality for each iterate $\x_k$ in order to bound each term in the path length summation: %Consider the path length $T \in \naturals$ iterations,
\begin{align*}
%    \pathlength_\eta &= \sum_{k = 0}^{T-1}\enorm{\x_k - \x_{k+1}}
%    \\ &\leq (1 + A(1 - c)) \sum_{k = 0}^{T-1}\enorm{\x_k - \widehat\x}
%    \\ &\leq (1 + A(1 - c)) \sum_{k = 0}^{T-1}A(1 - c)^k\enorm{\x_0 - \widehat\x}
%    \\ &\leq \roundbrack{\frac{A + A^2(1 - c)}{c}} \enorm{\x_0 - \widehat\x}.
    \pathlength_\eta &= \sum_{k = 0}^\infty \enorm{\x_k - \x_{k+1}}
    \\ &= \sum_{k = 0}^\infty \enorm{\x_k - (\x_{k})^+}    
    \\ &\leq 2 \sum_{k = 0}^\infty \enorm{\x_k - \Pi_{\widehat{\X}}(\x_k)}
    \\ &\leq 2 \sum_{k = 0}^\infty A(1 - c)^k\enorm{\x_0 - \Pi_{\widehat\X}(\x_0)} &\text{(by condition (a))}
    \\ &= \roundbrack{\frac{2A}{c}} \dist{\x_0}{\widehat\X},
\end{align*}
as was to be shown. 
%Since the above is true for any $T \in \naturals$, indeed
%\[
%\pathlength_\eta  \leq  \roundbrack{\frac{A + A^2(1 - c)}{c}} \enorm{\x_0 - \widehat\x},
%\]
%as was to be shown. 
\end{proof}

\rev{If the convergence set $\widehat{\X}$ or $\X^*$ is singleton and $A = 1$, condition (b) of Theorem~\ref{thm:LCGD2extended} is implied by condition (a). %Condition (b) in the statement of Theorem~\ref{thm:LCGD2extended} is required to forbid the condition shown in Figure~\ref{fig:counterExample}. 
However, condition (a) is not sufficient if %This theorem does not admit a direct generalization if 
the convergence set is not singleton, without additional assumptions on $f$ or the specific algorithm used. Figure~\ref{fig:counterExample} illustrates that there may exist algorithms that satisfy condition (a), but not condition (b), and do not exhibit short path lengths.}
\begin{figure}[t!]
    \centering
    \includegraphics[width=0.6\linewidth,trim=4.5cm 3cm 7.5cm 5.5cm]{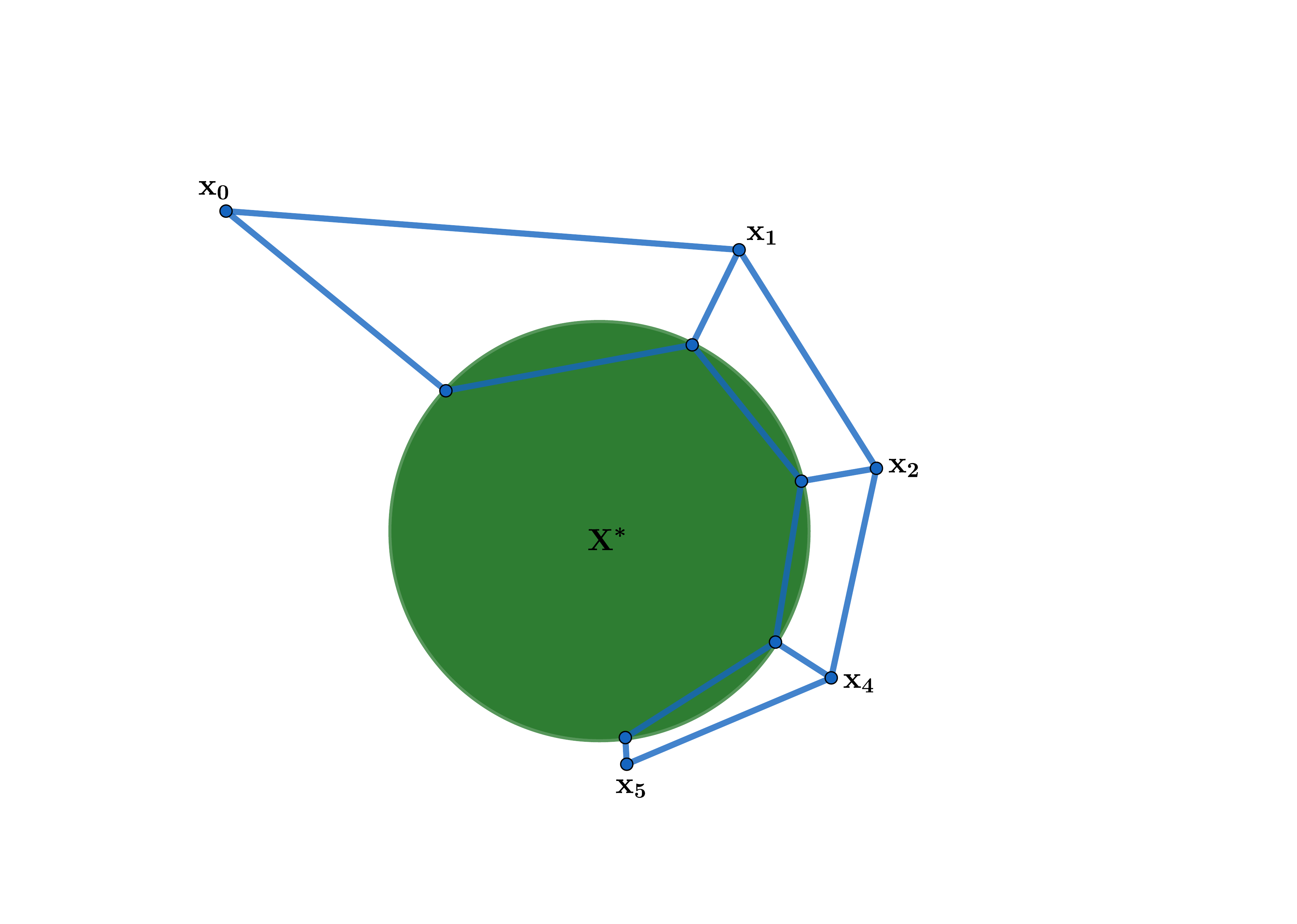}
    \caption{\rev{The iterates $\x_k$ converge linearly to the optimal set $\X^*$ but the path length is high. This situation is forbidden by condition (b) in Theorem~\ref{thm:LCGD2extended}.}}%We can only apply triangle inequality if the projection of each $\x_k$ is the same $\x^*$.}
    \label{fig:counterExample}
\end{figure}

As a corollary to Theorem~\ref{thm:LCGD2extended}, we prove a path length bounds for Polyak's heavy ball method~\citep{polyak1964some}.% and projected gradient descent. 

\subsection{Path Length Bound for Polyak's Heavy Ball Method}
Given suitable $\alpha, \beta$ the Polyak's Heavy Ball (HB) method takes the following update. Below, let $\x$ be the current iterate (initialized at some $\x_0$), $\x^-$ be the previous iterate (initialized as $\x_0$), and $\x^+$ be the update or the next iterate. 
\begin{equation}
    \text{Polyak's heavy ball (\HB):} \qquad
    \x^+ = \x - \alpha \nabla f(\x) + \beta(\x - \x^-).
\end{equation}
\rev{\revminor{\citet[Theorem 9]{polyak1964some}} showed that the HB method has linear convergence for a twice continuously differentiable and $\mu$-strongly convex function $f$ with $L$-Lipschitz gradients}, for the following choice of $\alpha, \beta$: 
    \begin{equation}
    \alpha = \frac{4}{\roundbrack{\sqrt{L} + \sqrt{\mu}}^2},\qquad \beta = \roundbrack{\frac{\sqrt{L} - \sqrt{\mu}}{\sqrt{L} + \sqrt{\mu}}}^2. \label{eqn:hbstepsize}    
    \end{equation}
    The linear convergence parameters are given by $A = 1$ and 
    \begin{align*}
        c &= 1 - \frac{\sqrt{\kappa} - 1}{\sqrt{\kappa} + 1}
        = \frac{2}{\sqrt{\kappa} + 1}.
    \end{align*}
    Note that a strongly convex function has a unique minimum and thus $\X^*$ is singleton. In this case, condition (a) of Theorem~\ref{thm:LCGD2extended} with $A = 1$ implies condition (b). Thus, we have the following corollary to Theorem~\ref{thm:LCGD2extended}.  
\begin{corollary}
\label{thm:HBSCLG}
    \rev{For any twice continuously differentiable and $\mu$-strongly convex function $f$ with $L$-Lipschitz gradients, the heavy ball method with $\alpha, \beta$ chosen according to~\eqref{eqn:hbstepsize}, has a path length bounded as:
    \[\pathlength_\eta \leq \sqrt{\kappa}\enorm{\x_0 - \x^*}.
    \]}
\end{corollary}
\begin{proof}
Given the linear convergence parameters of HB: $A = 1$ and $c = 2/(\sqrt{\kappa} + 1)$, we compute
\[
    \frac{A + A^2(1 - c)}{c} = \frac{2-c}{c} = \sqrt{\kappa}. 
\]
Applying Theorem~\ref{thm:LCGD2extended} we obtain the claimed result. 
\end{proof}
\citet{ghadimi2015global} showed that the HB method also admits a linear convergence bound assuming just continuous differentiability (not twice continuous differentiability), but for conservative values of $\alpha, \beta$ and at a slower rate: $c = \mathcal{O}(1/\kappa)$. This would lead to a path length bound of $\mathcal{O}(\kappa \enorm{\x_0 - \x^*})$. 
%\rev{Polyak's Heavy Ball method may satisfy path length bounds in settings more general than we consider, which is an avenue for future research.}

\rev{\subsection{Path Length Bound for Projected Gradient Descent}}
\label{subsec:PGD}
Consider a convex constraint set $\convexSet \subseteq \Real^d$ and let $\Pi_\convexSet(\x)$ denote the unique projection of a point $\x \in \Real^d$ on $\convexSet$. Projected gradient descent (\PGD) is an iterative optimization technique that ensures that the iterates stay within the constraint set. \PGD corresponds to taking a \GD step and projecting it onto $\convexSet$ as follows: 
\begin{equation}
    \label{eqn:pgd} 
     \text{Projected gradient descent (\PGD):}
    \qquad
    \x_{k+1} = \Pi_\convexSet(\x_k - \eta \nabla f(\x_k)).
\end{equation}
Clearly, \GD is a special case of \PGD with $\convexSet = \Real^d$. In this section, we provide a path length bound for \PGD if $f$ has $L$-Lipschitz gradients, and the iterates satisfy linear convergence. This result is closely related to the \GD bound of Theorem~\ref{thm:LCGD}. Technically, the significant difference is that in this case the stationarity assumption $\nabla f(\x^*) = 0$ may not be satisfied by the constrained optimal set $\X^*$. More work is needed to prove the result without the stationarity condition, and it leads to a larger constant (factor $(1 + \sqrt{2})$ below while Theorem~\ref{thm:LCGD} essentially obtains the factor $1$ after translation to the same setting). If stationarity does indeed hold (which is true if $\X^*$ is in the interior of $\convexSet$), then it can be shown that Theorem~\ref{thm:LCGD} applies directly and we get constant $1$.
%Contrast this to Theorem~\ref{thm:LCGDextended}, wherein we did not require convexity of $f$, but we did make an additional stationarity assumption on the constrained optimal set. Stationarity is not required in the analysis in this subsection. 
\ifx false 
Let $\widetilde{x} := x - \eta \nabla f(x)$. Consider the function $g(z) = \esqnorm{\widetilde{x} - z} + I_\Omega(z)$ minimized at $z = x^+$ so that by the subgradient condition, 
\begin{align*}
g(x^*) &\geq g(x^+) + \inner{\partial_z(\esqnorm{\widetilde{x} - z} + I_\Omega(z)) \mid_{x^+}, x^* - x^+}\\
%\implies \esqnorm{\widetilde{x} - x^*} &\geq \esqnorm{\widetilde{x} - x^+} + 2\inner{x^+ - \widetilde{x}, x^* - x^+}.% + t(\widetilde{x} - x^+),
\implies g(x^*) &\geq g(x^+) + 2\inner{x^+ - \widetilde{x}, x^* - x^+}.% + t(\widetilde{x} - x^+).
\end{align*}
%for all $t \geq 0$. 
Also, 
\begin{align*}
g(x^+) &\geq g(x^*) + \inner{\partial_z(\esqnorm{\widetilde{x} - z} + I_\Omega(z)) \mid_{x^*}, x^+ - x^*}\\
%\implies g(x^+) &\geq g(x^*) + 2\inner{x^* - \widetilde{x} + t \nabla f(x^*), x^+ - x^*}, \text{ for all } t \geq 0.
\implies g(x^+) &\geq g(x^*) + 2\inner{x^* - \widetilde{x}, x^+ - x^*}.
\end{align*}
Adding these two, 
\begin{align*}
0 &\geq 2\inner{x^+ - \widetilde{x}, x^* - x^+} + 2\inner{x^* - \widetilde{x} , x^+ - x^*}\\
\implies 0 &\geq \inner{x^+ - \widetilde{x} - (x^* - \widetilde{x}), x^* - x^+}.
\end{align*}

Similarly, define $x' = x^* - \nabla f(x^*)$ and consider $h(z) = \esqnorm{x'- z} + I_\Omega(z)$ minimized at $z = x^*$ (by the optimality of $x^*$). Then,
\begin{align*}
h(x^+) &\geq h(x^*) + \inner{\partial_z(\esqnorm{x' - z} + I_\Omega(z)) \mid_{x^*},x^+ - x^* }\\
\implies  h(x^+) &\geq h(x^*) + \inner{\eta \nabla f(x^*) , x^+ - x^*}.%,\text{ for all } t \geq 0.
\end{align*}
\fi 
\\\rev{\begin{theorem}
    \label{thm:LCPGD}
    Suppose $f$ has $L$-Lipschitz gradients. If the \PGD iterates with step-size $\eta$ exhibit linear convergence towards the optimal set $\X^*$ with constants $(A, c)$, then their path length is bounded as: 
    \begin{equation*}
%        \pathlength_\eta \leq (1 + \sqrt{2})(A/c) \ \dist{\x_0}{\X^*}.
                \pathlength_\eta \leq \roundbrack{\frac{\eta L + 1}{2} + \sqrt{\eta L + \roundbrack{\frac{\eta L + 1}{2}}^2}}(A/c) \ \dist{\x_0}{\X^*}.
    \end{equation*}
    If $\eta \leq 1/L$, the above simplifies to $\pathlength_\eta \leq (1 + \sqrt{2})(A/c) \ \dist{\x_0}{\X^*}$.
\end{theorem}}
\begin{proof}
For any $\x \in \convexSet$, denote $\widetilde{\x} = \x - \eta \nabla f(\x)$, and $\x^+ = \Pi_\convexSet(\widetilde{\x})$ which corresponds to the PGD update for $\x$. Also let $\x^* = \Pi_{\X^*}(\x)$ be the closest optimum point to $\x$. The next two inequalities use the following standard fact about projections to a convex set (for instance, see \citet[Lemma 3.1]{bubeck2015convex}): 
\[\text{for all }\x \in \Real^d \text{ and } \z \in \convexSet, \text{ the projection }\y = \Pi_\convexSet(\x) \text{ satisfies  } \inner{\y - \x, \z - \y} \geq 0.
\]
Since $\x^+$ is the projection of $\widetilde{\x}$, and $\x^* \in \convexSet$:
\begin{align}
\inner{\x^+ - \widetilde{\x}, \x^* - \x^+} \geq 0. \label{eqn:proj-eq-1}
\end{align}
Further, by the optimality of $\x^*$, $\Pi_\convexSet(\x^* - \eta \nabla f(\x^*)) = \x^*$, thus since $\x^+ \in \convexSet$, 
\begin{align}
\inner{\x^* - (\x^* -\eta \nabla f(\x^*)), \x^+ - \x^*} \geq 0. \label{eqn:proj-eq-2}
\end{align}
Adding \eqref{eqn:proj-eq-1} and \eqref{eqn:proj-eq-2},
\begin{align*}
\inner{\x^+ - \widetilde{\x} - \eta \nabla f(\x^*), \x^* - \x^+}  & \geq 0.
\end{align*}
Substituting $\widetilde{\x} = \x - \eta\nabla f(\x)$ and rearranging we get, 
\begin{align*}
 \eta \inner{\nabla f(\x) - \nabla f(\x^*), \x^* - \x^+} + \inner{\x^+ - \x, \x^* - \x}\geq \esqnorm{\x^+ - \x}.
\end{align*}
Using Cauchy-Schwarz and triangle inequality, this implies 
\begin{align*}
 \eta \enorm{\nabla f(\x) - \nabla f(\x^*)}(\enorm{\x^* - \x} + \enorm{\x - \x^+}) + \enorm{\x^+ - \x}\enorm{\x^* - \x}\geq \esqnorm{\x^+ - \x}.
\end{align*}
%Denote $ a= \enorm{\x - \x^+}$ and $b = \enorm{\x - \x^*}$. Using \LG, $\enorm{\nabla f(\x) - \nabla f(\x^*)} \leq L\enorm{\x - \x^*} = Lb$; thus we obtain,
%\[
%\eta L b^2 + \eta L ab + ab \geq a^2.  
%\]
%Simplifying using $\eta L \leq 1$ gives $b^2 + 2ab \geq a^2$. Completing squares, this leads to $(a - b)^2 \leq 2b^2$; thus we can conclude:
Denote $ a= \enorm{\x - \x^+}$ and $b = \enorm{\x - \x^*}$. Using \LG, $\enorm{\nabla f(\x) - \nabla f(\x^*)} \leq L\enorm{\x - \x^*} = Lb$; thus the above can be re-written as
\[
\eta L b (a + b) + ab \geq a^2.
\]
Completing squares, 
\[
\roundbrack{a - \roundbrack{\frac{\eta L + 1}{2}} b}^2 \leq \eta L b^2 +  \roundbrack{\frac{\eta L + 1}{2}}^2 b^2, %\leq 2\eta L b^2,%= \roundbrack{\frac{\eta L - 1}{2}}^2 b^2,
\]
which gives 
\begin{equation}
a \leq \roundbrack{\frac{\eta L + 1}{2} + \sqrt{\eta L + \roundbrack{\frac{\eta L + 1}{2}}^2}} b. \label{eq:plgd-lemma}
\end{equation}
Denote the factor multiplied to $b$ as $u$ so that $a \leq ub$. 
\ifx false since $\eta \leq 1/L$. Thus, 
\[
a \leq \roundbrack{\frac{\eta L + 1}{2}} b + \sqrt{2}\eta L b \leq (1 + \sqrt{2}) \eta L b,
\]
again using $\eta \leq 1/L$.
 \begin{equation} a \leq (1+\sqrt{2}) b.\label{eq:plgd-lemma} \end{equation} 
\fi 
 We now compute the path length bound by instantiating the above result for $\x$ given by the iterates of PGD so that for $\x = \x_k$, $a = \enorm{\x_k - \x_{k+1}}$ and $b = \dist{\x_k}{\X^*}$:
\begin{align*}
    \pathlength_\eta &= \sum_{k=0}^\infty \enorm{\x_k - \x_{k+1}}
%    \\ &\leq \sum_{k=0}^\infty (1 + \sqrt{2}) \ \dist{\x_k}{\X^*} &\text{(using \eqref{eq:plgd-lemma})}
%    \\ &\leq \sum_{k=0}^\infty (1 + \sqrt{2}) A\ (1 - c)^k\dist{\x_0}{\X^*} &\text{(using linear convergence)}
%    \\ &= (1 + \sqrt{2}) (A/c)\ \dist{\x_0}{\X^*}.
 	\\ &\leq \sum_{k=0}^\infty u \ \dist{\x_k}{\X^*} &\text{(using \eqref{eq:plgd-lemma})}
    \\ &\leq \sum_{k=0}^\infty u A\ (1 - c)^k\dist{\x_0}{\X^*} &\text{(using linear convergence)}
        \\ &= (uA/c)\ \dist{\x_0}{\X^*}.
\end{align*}

%where the inequality in the second last line follows by linear convergence. 
This completes the proof. 
\end{proof}

\rev{Notice that the only property of $\X^*$ used in the proof is stationarity with respect to the \PGD updates, required in equation~\eqref{eqn:proj-eq-2}. Thus, like Theorem~\ref{thm:LCGDextended}, Theorem~\ref{thm:LCPGD} can be extended to include convergence to any locally optimal set $\widehat{\X}$ that is stationary with respect to the \PGD updates.}

\section{Proof of Theorem~\ref{thm:PKLLGGF}}
\label{appsec:proofsPKL}
We first prove the statement for \GF. Consider some $t$ such that $\x_t \notin \X^*$, so that $f(\x_t) \neq f(\x^*)$. Note that once $\x_t \in \X^*$, $\nabla f(\x_\tau) = 0$ for all $\tau \geq t$, and hence the path length is $0$ henceforth. Consider the following Lyapunov function: \[
    \lyapunov_t = \sqrt{f(\x_t) - f^*}.
\]
Suppose $\x_t \notin \X^*$. Taking the derivative of $\lyapunov_t$ with respect to time, and using chain rule,
\begin{align*}
    \vel{\lyapunov}_t &= \frac{\frac{d f(\x_t)}{dt}}{2\sqrt{f(\x_t) - f^*}}
    \\ &= \frac{\inner{\frac{d f(\x_t)}{d\x_t}, \frac{d\x_t}{dt}}}{2\sqrt{f(\x_t) - f^*}}
    \\ &= -\frac{\esqnorm{\nabla f(\x_t)}}{2\sqrt{f(\x_t) - f^*}}
    \\ &\textleq{\PKL} -\sqrt{\frac{\mu}{2}}\enorm{\nabla f(\x_t)}.
\end{align*}
Although the above proof assumes $\x_t \notin \X^*$, the conclusion is true even for $\x_t \in \X^*$, since both sides are equal to $0$. Using the fundamental theorem of calculus, 
\begin{align*}
    \int_0^\infty \enorm{\nabla f(\x_t)}\ dt \leq -\sqrt{\frac{2}{\mu}}\squarebrack{\lyapunov_t}_0^\infty   = -\sqrt{\frac{2}{\mu}} \squarebrack{\sqrt{f(\x_t) - f^*}}_0^\infty = \sqrt{\frac{2(f(\x_0) - f^*)}{\mu}}.
\end{align*}
We can then use \LG to bound the right side in terms of the distance of the point $\x_0$ from the optimal set:
\[
f(\x_0) - f^* \leq \frac{L}{2}\esqnorm{\x_0 - \Pi_{\X^*}(\x_0)} =  \frac{L}{2}\dist{\x_0}{\X^*}^2.
\]
Substituting this back in, we obtain our result: 
\[
\pathlength = \int_0^\infty \enorm{\nabla f(\x_t)}\ dt \leq \sqrt{\frac{L}{\mu}}\ \dist{\x_0}{\X^*}, 
\]
as was to be shown. 

The \GD proof technique is directly inspired by the proof of~\citet{oymak2019overparameterized}: Equation~(5.3) of Theorem~5.2 in the paper. \rev{(However, their path length bound in terms of $\ \dist{\x_0}{\X^*}$ is weaker than ours. Equation~(5.5) in their paper corresponds to the bound $\frac{2L}{\mu} \dist{\x_0}{\X^*}$ instead of $2\sqrt{\frac{L}{\mu}} \dist{\x_0}{\X^*}$ as we show.)} Like the \GF case, consider some $k$ such that $\x_k \notin \X^*$. Then by \LG,
\begin{align*}
    \sqrt{f(\x_{k+1}) - f^*} &= \sqrt{f(\x_k - \eta \nabla f(\x_k)) - f^*}
    \\ &\leq \sqrt{f(\x_k) - \inner{\nabla f(\x_k), \eta \nabla f(\x_k)} + \frac{L}{2}\esqnorm{\eta \nabla f(\x_k)}- f^*}
    \\ &= \sqrt{f(\x_k)- f^* - (\eta - \eta^2L/2)\esqnorm{\nabla f(\x_k)}}
    \\ &\leq \sqrt{f(\x_k)- f^*} - \frac{(\eta - \eta^2L/2)\esqnorm{\nabla f(\x_k)}}{2\sqrt{f(\x_k)- f^*}}.
\end{align*}
\rev{The final inequality holds since $\sqrt{a - b } \leq \sqrt{a} - \frac{b}{2\sqrt{a}}$. To simplify the second term, note that $\eta \leq 1/L$, implies $\eta - \eta^2L/2 \geq \eta/2$. Using this and the \PKL inequality, we get%From \PKL, $\enorm{\nabla f(\x_k)} \geq \sqrt{2\mu (f(\x_k) - f^*)}$, we get
\[
\frac{(\eta - \eta^2L/2)\esqnorm{\nabla f(\x_k)}}{2\sqrt{f(\x_k)- f^*}} \geq (\eta/2) \cdot \frac{\enorm{\nabla f(\x_k)}}{2\sqrt{f(\x_k)- f^*}} \cdot \enorm{\nabla f(\x_k)} \geq \eta \sqrt{\mu/8}\ \enorm{\nabla f(\x_k)}. 
\]
Rearranging, we get
\begin{align*}
    \sqrt{f(\x_{k+1}) - f^*} &\leq \sqrt{f(\x_k)- f^*} - \eta \sqrt{\mu/8}\ \enorm{\nabla f(\x_k)}.%\sqrt{\frac{\mu}{8}} \roundbrack{\eta \enorm{\nabla f(\x_k)}}.
\end{align*}}
The above is trivially also true if $\x_k \in \X^*$, since both sides are $0$.
Note that $\x_{k+1} - \x_k = -\eta \nabla f(\x_k)$; thus for all $k \geq 0$,
\[
\enorm{\x_k - \x_{k+1}} \leq \sqrt{8/\mu} \ \roundbrack{\sqrt{f(\x_k)- f^*} - \sqrt{f(\x_{k+1})- f^*}}.
\]
Telescoping this from $k = 0, \ldots \infty$, 
\[
\pathlength_\eta = \sum_{k=0}^\infty \enorm{\x_k - \x_{k+1}} \leq \sqrt{8/\mu}\roundbrack{\sqrt{f(\x_0) - f^*} - \sqrt{f(\x_\infty) - f^*}} \leq \sqrt{8/\mu}\roundbrack{\sqrt{f(\x_0) - f^*}}.
\]
As an immediate consequence of \LG, we have
$
\sqrt{f(\x_0) - f^*} \leq \sqrt{\frac{L}{2}}\ \dist{\x_0}{\X^*},
$
and plugging this into the above bound yields
\[
\pathlength_\eta \leq 2\sqrt{\frac{L}{\mu}}\ \dist{\x_0}{\X^*},
\]
as claimed.
\hfill\BlackBox

\section{Proofs of Results in Section~\ref{sec:quadratics}}
\label{appsec:proofsQuadratics}
The proofs of Theorem~\ref{thm:QUADGF} and Theorem~\ref{thm:upperSpecialSC} are organized as subsections.
\subsection{Proof of Theorem~\ref{thm:QUADGF}}
We first write the proof in the \GF case. Let $\alpha_i$ be the component of $(\diffx)$ in the direction of the eigenvector of $\Sigma$ that corresponds to the eigenvalue $\sigma_i$. Observe that 
\begin{align*}
    \pathlength &= \int_{0}^{\infty} \| \Dot{\x}_t\|_2\ dt \nonumber
\\&= \int_{0}^{\infty} \|\exp(-t \Sigma) \Sigma(\diffx)\|_2\ dt \nonumber
\\& = \int_0^\infty \sqrt{\sum_{i=1}^{d^+}\expText{-2t\sigma_i}\sigma_i^2\alpha_i^2}\ dt \label{eqn:quadgflength}.
\end{align*}
The $\sqrt{d^+}$ bound is straightforward. Since $\sqrt{a + b} \leq \sqrt{a} + \sqrt{b}$ for nonnegative $a$ and $b$, we have 
\begin{align*}
    \int_0^\infty \sqrt{\sum_{i=1}^{d^+}\expText{-2t\sigma_i}\sigma_i^2\alpha_i^2}\ dt &\leq \int_0^\infty \sum_{i=1}^{d^+}\sqrt{\expText{-2t\sigma_i}\sigma_i^2\alpha_i^2}\ dt
    \\ &= \int_0^\infty \sum_{i=1}^{d^+}{\expText{-t\sigma_i}\sigma_i\abs{\alpha_i}}\ dt
    \\ &= \sum_{i=1}^{d^+}\int_0^\infty {\expText{-t\sigma_i}\sigma_i\abs{\alpha_i}}\ dt
    \\ &= \sum_{i=1}^{d^+} \abs{\alpha_i} 
    \\ &\leq \sqrt{d^+} \|\alpha\|_2 ~=~ \sqrt{d^+}\ \dist{\x_0}{\X^*}.
\end{align*}
We now prove the bound in~\eqref{eqn:QUADGF} that depends on the $\kappa_i$'s. For every $t \in \Real^+$, consider a function $g_t:\Real^+ \to \Real$, $g_t(x)=\exp(-2tx)x^2$. For every value of $t$, a term in the path length integral is a linear combination of evaluations of $g_t$ at the points $\sigma_1, \sigma_2, \ldots, \sigma_{d^+}$. We will bound each $g_t(\sigma_i)$ in the linear combination with $\max_j g_t(\sigma_j)$. Notice that for different values of $t$, $\argmax_j g_t(\sigma_j)$ is different. The maximum of $g_t$ occurs at $x_m=1/t$, and $g_t$ is an increasing function in $x$ before $x_m$ and decreasing in $x$ after $x_m$. To bound the path length we use this observation to identify $\argmax_j g_t(\sigma_j)$ for each $t$ (see Figure~\ref{fig:gPlots} for reference). 
\begin{figure}[t!]
    \centering
    \includegraphics[width=0.9\linewidth,trim = 0cm 2cm 0cm 3cm, clip]{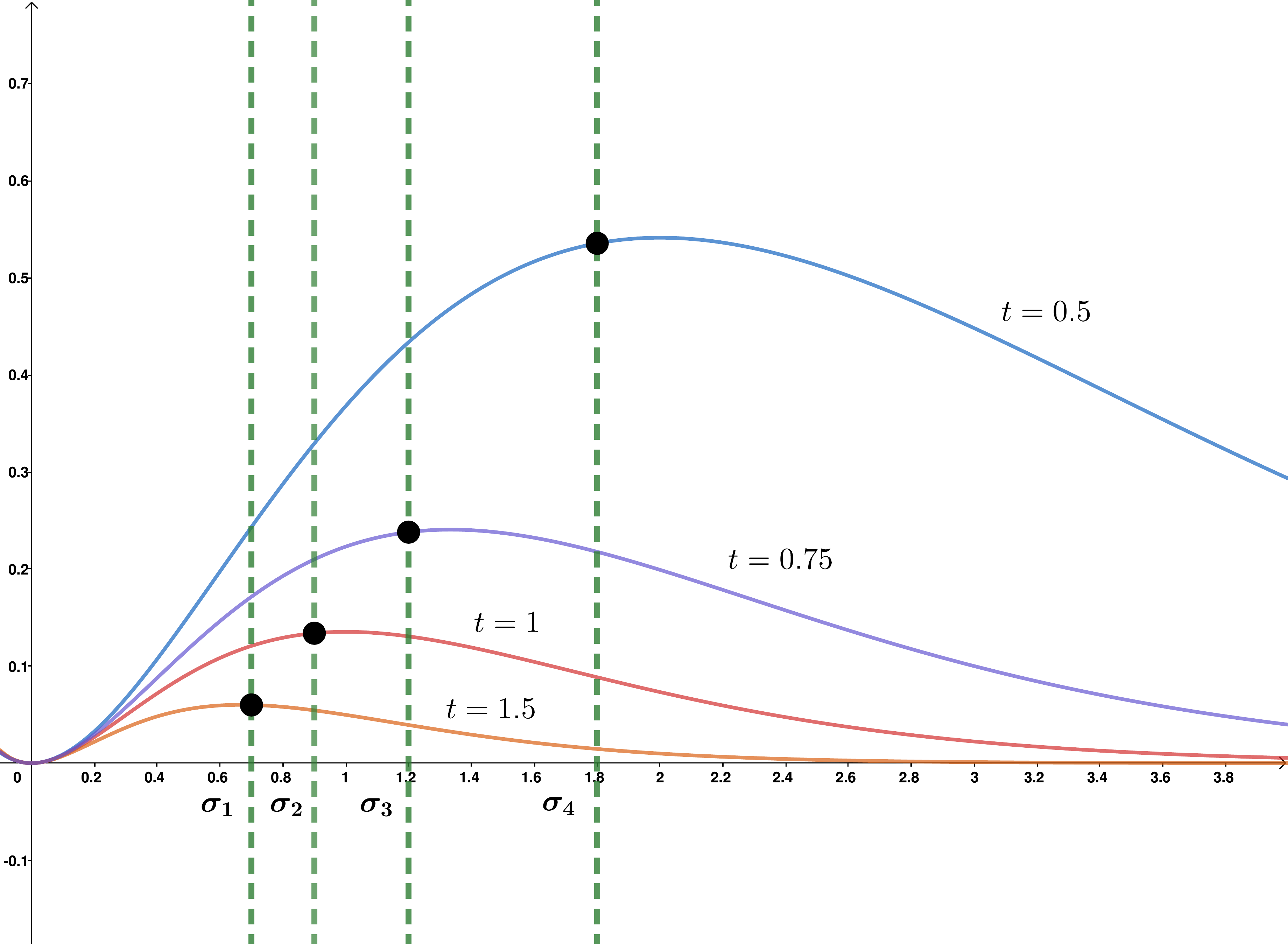}
    \caption{$g_t(x)$ for some values of $t$. Suggestive values of $\sigma_i = 0.7, 0.9, 1.2, 1.8$ are also shown. Observe that for every $t$, there is a different $\sigma_i$ that maximizes $g_t$ (indicated with large black points). }
    \label{fig:gPlots}
\end{figure}

\begin{align*}
    \pathlength ~&=~ \int_0^\infty \sqrt{\sum_{i=1}^{d^+} \expText{-2t\sigma_i}\sigma_i^2\alpha_i^2}\ dt 
    = \underbrace{\int_0^{1/\sigma_1} \sqrt{\sum_{i=1}^{d^+} \expText{-2t\sigma_i}\sigma_i^2\alpha_i^2}\ dt}_{=: E_1}\\
    &\quad \quad \quad \quad + \sum_{j=1}^{d^+-1} \underbrace{\int_{1/\sigma_j}^{1/\sigma_{j+1}} \sqrt{\sum_{i=1}^{d^+} \expText{-2t\sigma_i}\sigma_i^2\alpha_i^2}\ dt}_{=:T_j} ~+~ \underbrace{\int_{1/\sigma_{d^+}}^\infty \sqrt{\sum_{i=1}^{d^+} \expText{-2t\sigma_i}\sigma_i^2\alpha_i^2}\ dt}_{=: E_2}.
\end{align*}

For the first integral, when $t \leq 1/\sigma_1$, we have that $\sigma_i \leq \sigma_1 < 1/t$, and at this stage $g_t$ is an increasing function of $x$, so every term is upper bounded by $\exp(-2t\sigma_1)\sigma_1^2$. This leads to a bound for the first term above:
\begin{equation}
E_1 \leq \int_0^{1/\sigma_1} \exp(-t\sigma_1)\sigma_1 \|\alpha\|_2 dt =  (1-1/e)\|\diffx\|_2 .\label{eqn:E1}
\end{equation}
Similarly, for the last integral, when $t \geq 1/\sigma_{d^+}$, we have all $\sigma_i \geq \sigma_{d^+} > 1/t$, and here $g_t$ is a decreasing function of $x$, so every term is upper bounded by $\exp(-2t\sigma_{d^+})\sigma_{d^+}^2$, so the last term is upper bounded as:
\begin{equation}
E_2 \leq \int_{1/\sigma_{d^+}}^\infty \exp(-t\sigma_{d^+})\sigma_{d^+} \|\alpha\|_2 dt =  (1/e)\|\diffx\|_2 .\label{eqn:E2}
\end{equation}
Last, for the middle integral, consider a particular term $T_j$. If $\sigma_j = \sigma_{j+1}$, $T_j = 0 = \kappa_j^{-1/(\kappa_j-1)} (1-1/\kappa_j)$. Else, define $t_j := \log(\kappa_j)/(\sigma_{j}-\sigma_{j+1})$ and observe that that $t_j \in (1/\sigma_{j}, 1/\sigma_{j+1})$. We can split $T_j$ into two parts:
\[
T_j = \int_{1/\sigma_{j}}^{t_j} \sqrt{\sum_{i=1}^d \expText{-2t\sigma_i}\sigma_i^2\alpha_i^2}\ dt + \int_{t_j}^{1/\sigma_{j+1}} \sqrt{\sum_{i=1}^d \expText{-2t\sigma_i}\sigma_i^2\alpha_i^2}\ dt.
\]
Whenever $ 1/\sigma_{j} < t < 1/\sigma_{j+1}$, $\sigma_{j+1} < 1/t < \sigma_j$. Thus, for every $t$, the value of $\max\{g_t(\sigma_j), g_t(\sigma_{j+1})\}$ dominates every $g_t(\sigma_i)$. Which one of these two is larger depends on which side of $t_j$ we consider. In the first term, $g_t(\sigma_{j})$ dominates, and in the second $g_t(\sigma_{j+1})$ dominates. This yields an upper bound of:
\begin{align}
T_j ~&\leq \int_{1/\sigma_{j}}^{t_j} \expText{-t\sigma_{j}}\sigma_{j} \|\alpha\|_2 \ dt + \int_{t_j}^{1/\sigma_{j+1}} \expText{-t\sigma_{j+1}}\sigma_{j+1} \|\alpha\|_2\ dt \nonumber \\
&= (\exp(-t_{j} \sigma_{j+1}) - \exp(-t_{j+1}\sigma_{j})) \|\alpha\|_2 \nonumber \\
&= \kappa_j^{-1/(\kappa_j-1)} (1-1/\kappa_j) \enorm{\diffx}.\label{eqn:Tj}
\end{align}
Summing up the bounds in Equations~\eqref{eqn:E1},~\eqref{eqn:E2},~\eqref{eqn:Tj}, we get
\[
\pathlength ~\leq~ \roundbrack{1 + \sum_{j=d-1}^{1} \kappa_j^{-1/(\kappa_j-1)} (1-1/\kappa_j)}\|\diffx\|_2 .
\]

\begin{figure}[t!]
    \centering
    \includegraphics[width=0.4\linewidth]{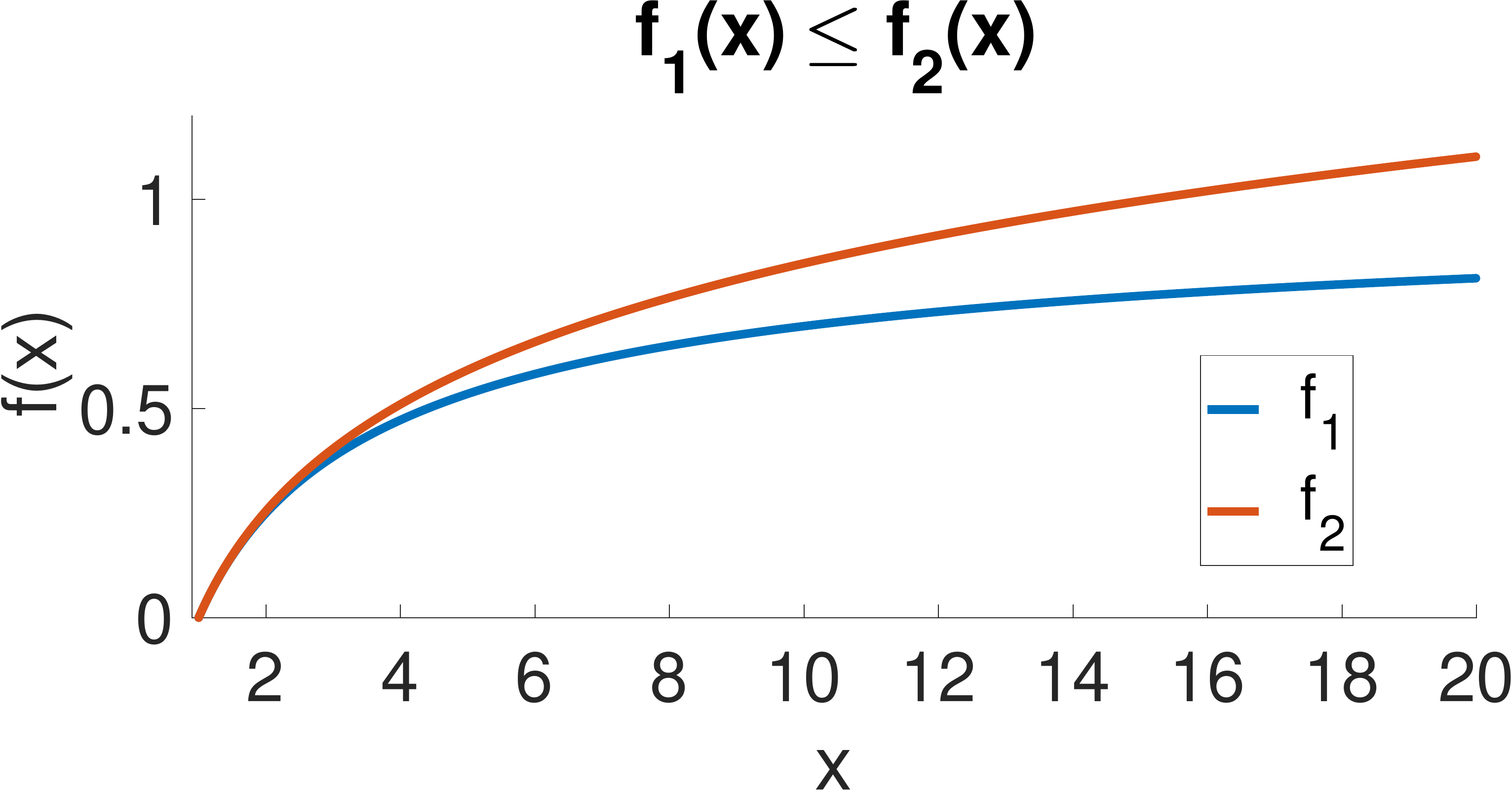}
    \caption{\rev{Graphical proof for $ x^{-1/(x-1)} (1-1/x) \leq \frac{\log x}{e}$. $f_1(x) = x^{-1/(x-1)} (1-1/x)$ and $f_2(x) = \frac{\log x}{e}$. Clearly $f_1(x) \leq 1$ so for $x \in [20, \infty)$ the inequality is trivial.}}
    \label{fig:graphicalProof}
\end{figure}
Next, we simplify the above expression in terms of $\kappa$. Note the following fact for all $x \geq 1$ (this can be seen graphically as shown in Figure~\ref{fig:graphicalProof}): 
    \[
    x^{-1/(x-1)} (1-1/x) \leq \frac{\log x}{e}. 
    \]
    Thus 
    \[
    \sum_{j=1}^{d^+-1} \kappa_j^{-1/(\kappa_j-1)} (1-1/\kappa_j) \leq \sum_{j=1}^{d^+-1} \frac{\log \kappa_j}{e} = \frac{\log \kappa}{e},
    \]
which gives an $\mathcal{O}(\log \kappa)$ bound. For the $\mathcal{O}(\sqrt{\log \kappa})$ bound, we first split $\pathlength$ as before:
\begin{align*}
    \pathlength ~&= E_1 +\roundbrack{\int_{1/\sigma_1}^{1/\sigma_{d^+}} \sqrt{\sum_{i=1}^{d^+} \expText{-2t\sigma_i}\sigma_i^2\alpha_i^2}\ dt} + E_2 
    \\ &\leq 1 +\int_{1/\sigma_1}^{1/\sigma_{d^+}} \sqrt{\sum_{i=1}^{d^+} \expText{-2t\sigma_i}\sigma_i^2\alpha_i^2}\ dt . 
\end{align*}
Now we bound the second summation. Fix a constant $b > 1$, to be specified later. Let $r = \ceil{\log_b \kappa} = \ceil{\log_b(\sigma_1/\sigma_{d^+})}$ and consider the $r$ intervals given by $I_k = [b^{k-1} \sigma_{d^+}, b^k\sigma_{d^+})$ for $k \in [r-1]$ and $I_r = [b^{r-1} \sigma_{d^+}, b^r\sigma_{d^+}]$. If $\sigma_i \in I_k$, that is if $b^k\sigma_{d^+} \geq \sigma_i \geq b^{k-1}\sigma_{d^+}$ then we have for any $t$
\[
\expText{-2t\sigma_i}\sigma_i^2 \leq \expText{-2t(b^{k-1}\sigma_{d^+})}b^{2k}\sigma_{d^+}^2.
\]
Define $\theta_k := \sqrt{\sum_{i: \sigma_i \in I_k} \alpha_i^2}$ and note that $\sqrt{\sum_{k=1}^r \theta_k^2} = \enorm{\diffx}$. Then

\begin{align*}
    \int_{1/\sigma_1}^{1/\sigma_{d^+}} \sqrt{\sum_{i=1}^{d^+} \expText{-2t\sigma_i}\sigma_i^2\alpha_i^2}\ dt &= \int_{1/\sigma_1}^{1/\sigma_{d^+}} \sqrt{\sum_{k = 1}^{r}\sum_{i: \sigma_i \in I_k}\expText{-2t\sigma_i}\sigma_i^2\alpha_i^2}\ dt
    \\ &\leq \int_{1/\sigma_1}^{1/\sigma_{d^+}} {  \sqrt{\sum_{k = 1}^{r}\sum_{i: \sigma_i \in I_k} \expText{-2t(b^{k-1}\sigma_{d^+})}b^{2k}\sigma_{d^+}^2\alpha_i^2}}\ dt\\ 
    &\leq \int_{1/\sigma_1}^{1/\sigma_{d^+}} \roundbrack{\sum_{k = 1}^{r}  \sqrt{\sum_{i: \sigma_i \in I_k} \expText{-2t(b^{k-1}\sigma_{d^+})}b^{2k}\sigma_{d^+}^2\alpha_i^2}}\ dt
    \\ &= \int_{1/\sigma_1}^{1/\sigma_{d^+}} \roundbrack{\sum_{k = 1}^{r} \theta_k \expText{-t(b^{k-1}\sigma_{d^+})}b^{k}\sigma_{d^+}}\ dt
    \\ &= \sum_{k = 1}^{r} \theta_k \int_{1/\sigma_1}^{1/\sigma_{d^+}} \expText{-t(b^{k-1}\sigma_{d^+})}b^{k}\sigma_{d^+}\ dt
    \\ &= b\sum_{k = 1}^{r} \theta_k \roundbrack{\expText{-(b^{k-1}\sigma_{d^+})/\sigma_1} - \expText{-(b^{k-1}\sigma_{d^+})/\sigma_{d^+})}}
%    \\ &\leq b\sum_{k = 1}^{r} \theta_k (\expText{-( b^{k-1}/b^{r})} 
    \\ &\leq b\sum_{k=1}^{r}\theta_k  
%    \\ &\leq b\sqrt{\sum_{k = 1}^{r} \expText{-(2b^{k-1}/b^{r})}}\sqrt{\sum_{k=1}^{r} \theta_k^2}
    \\ &\leq b\sqrt{r}\sqrt{\sum_{k=1}^{r} \theta_k^2}
    \\ &= b\sqrt{\log_b\kappa} \enorm{\diffx}
    \\ &= {b}\sqrt{\log_be}\ \sqrt{\log \kappa} \enorm{\diffx}.
%    \\ &\leq 4\max(1, \sqrt{\log_2 \kappa}) \enorm{\diffx}
%   \\ &\leq 5\max(1, \sqrt{\log \kappa}) \enorm{\diffx}.
\end{align*}

\ifx false
Let $r = \ceil{\log_2(\kappa)} = \ceil{\log_2(\sigma_1/\sigma_{d^+})}$ and consider the $r$ intervals given by $I_k = [2^{k-1} \sigma_{d^+}, 2^k\sigma_{d^+}]$ for $k \in [r ]$. If $\sigma_i \in I_k$, observe that $2^k\sigma_{d^+} \geq \sigma_i \geq 2^{k-1}\sigma_{d^+}$ and thus for any $t$
\[
\expText{-2t\sigma_i}\sigma_i^2 \leq \expText{-2t(2^{k-1}\sigma_{d^+})}2^{2k}\sigma_{d^+}^2.
\]
Define $\theta_k := \sqrt{\sum_{i: \sigma_i \in I_k} \alpha_i^2}$. Then

\begin{align*}
    \int_{1/\sigma_1}^{1/\sigma_{d^+}} \sqrt{\sum_{i=1}^{d^+} \expText{-2t\sigma_i}\sigma_i^2\alpha_i^2}\ dt &\leq \int_{1/\sigma_1}^{1/\sigma_{d^+}} \roundbrack{\sum_{k = 1}^{r}  \sqrt{\sum_{i: \sigma_i \in I_k} \expText{-2t(2^{k-1}\sigma_{d^+})}2^{2k}\sigma_{d^+}^2\alpha_i^2}}\ dt
    \\ &= \int_{1/\sigma_1}^{1/\sigma_{d^+}} \roundbrack{\sum_{k = 1}^{r} \theta_k \expText{-t(2^{k-1}\sigma_{d^+})}2^{k}\sigma_{d^+}}\ dt
    \\ &= \sum_{k = 1}^{r} \theta_k \int_{1/\sigma_1}^{1/\sigma_{d^+}} \expText{-t(2^{k-1}\sigma_{d^+})}2^{k}\sigma_{d^+}\ dt
    \\ &= 2\sum_{k = 1}^{r} \theta_k (\expText{-(2^{k-1}\sigma_{d^+})/\sigma_1} - \expText{-(2^{k-1}\sigma_{d^+})/\sigma_{d^+})}
    \\ &\leq 2\sum_{k = 1}^{r} \theta_k (\expText{-( 2^{k-1}/2^{r})} 
    \\ &\leq 2\sum_{k=1}^{r}\theta_k  
    \\ &\leq 2\sqrt{r}\sqrt{\sum_{k=1}^{r} \theta_k^2}
    \\ &\leq 2\sqrt{\log_2(2\kappa)} \enorm{\diffx}.
%    \\ &\leq 4\max(1, \sqrt{\log_2 \kappa}) \enorm{\diffx}
%   \\ &\leq 5\max(1, \sqrt{\log \kappa}) \enorm{\diffx}.
\end{align*}
\ifx false
We can remove the $\max$ in the above expression to make it succinct, by writing it as part of the more general Equation~\eqref{eqn:QUADGF}. Suppose $\sqrt{\log \kappa} < 1$ (that is, $1$ dominates the $\max$), then $\log \kappa < 1$ and thus 
\[
    \roundbrack{1 + \sum_{j=d-1}^{1} \kappa_j^{-1/(\kappa_j-1)} (1-1/\kappa_j)} \leq 1 + \frac{\log \kappa}{e} \leq 5 = 5\max(1, \sqrt{\log \kappa}), 
\]
so that the minimum operator in Equation~\eqref{eqn:QUADGF} would not pick out $5\max(1, \sqrt{\log \kappa})$. 
\fi 
\fi

Setting $b = 2$, which is close to the minima of $b \sqrt{\log_b e}$, and bounding $2 \sqrt{\log_2 e} \leq 2.5$ gives the result. This concludes the proof in the \GF case. 

For \GD, since $\A^Ty = (\A^T\A)\Pi_{\X^*}(\x_0)$, Equation~\eqref{eqn:quadraticsGD} leads to the following recurrence for $k \geq 1$: 
\[
\x_k - \Pi_{\X^*}(\x_0) = \x_{k-1} - \Pi_{\X^*}(\x_0) - \eta \roundbrack{\frac{\A^T\A}{n}}(\x_{k-1} -  \Pi_{\X^*}(\x_0)) = \roundbrack{I - \eta \Sigma}(\x_{k-1} - \Pi_{\X^*}(\x_0)).
\]
Thus for $k\in\wholes$, 
\[
(\x_k - \Pi_{\X^*}(\x_0)) = \roundbrack{I - \eta \Sigma}^{k}(\x_0 - \Pi_{\X^*}(\x_0))
\]
Then we can compute the path length as follows:
\begin{align*}
\text{For }k \geq 1, \ \enorm{\x_k - \x_{k-1}} &= \enorm{(\x_k - \Pi_{\X^*}(\x_0)) - (\x_{k-1} - \Pi_{\X^*}(\x_0))}
\\ &= \enorm{\roundbrack{I - \eta \Sigma}^{k-1}\roundbrack{\eta \Sigma}(\x_0 - \Pi_{\X^*}(\x_0))}.
\end{align*}
Taking a sum over all iterates from $k = 0$ to $\infty$, we obtain the bound
\begin{align*}
\pathlength_\eta&= \sum_{k=0}^\infty \enorm{\roundbrack{I - \eta \Sigma}^k\roundbrack{\eta \Sigma}(\x_0 - \Pi_{\X^*}(\x_0))}.
\end{align*}
Note that $\eta$ is such that the singular values of $\eta \Sigma$ are $1 \geq \sigma'_1 \geq \sigma'_2 \cdots \geq \sigma'_p > 0$. For every $j \in [d^+-1]$, we have $\kappa_j = \sigma'_j/\sigma'_{j+1}$. Also observe that $\alpha_i$ is the component of ($\diffx$) in the direction of the eigenvector of $\eta\Sigma$ that corresponds to the eigenvalue $\sigma'_i$. Thus, 
\begin{align*}
\pathlength_\eta &= \sum_{k=0}^\infty \sqrt{\sum_{i = 1}^{d^+} (1 - \sigma'_i)^{2k}\sigma_i^{'2}\alpha_i^2}
\\ &\leq \sum_{k=0}^\infty \sqrt{\sum_{i = 1}^{d^+} \expText{-2k\sigma'_i}\sigma_i^{'2}\alpha_i^2}
\\ &\leq \ \enorm{\diffx} + \sum_{k=1}^\infty \sqrt{\sum_{i = 1}^{d^+} \expText{-2k\sigma'_i}\sigma_i^{'2}\alpha_i^2}
\\ &\leq \ \enorm{\diffx} + \int_{0}^\infty \sqrt{\sum_{i = 1}^{d^+} \expText{-2t\sigma'_i}\sigma_i^{'2}\alpha_i^2}\ dt
\\ &= \ \enorm{\diffx} + \int_{0}^\infty \sqrt{\sum_{i = 1}^{d^+} \expText{-2t\sigma_i}\sigma_i^{2}\alpha_i^2}\ dt &\text{(reparameterizing $t \to \eta t$)}
\\&= \ \enorm{\diffx} + \pathlength,
\end{align*}
as was to be shown. 
\hfill\BlackBox

\subsection{Proof of Theorem~\ref{thm:upperSpecialSC}} 
In each component $i$, let $\x_\indi^*$ denote the unique minimum. We will consider \GF on $f$ with some initial point $\x_0$. For every index $i$ consider the following potential function for any time $t$ such that $(\x_t)_\indi \neq \x_\indi^*$:
\[
\phi(t) = \frac{g_\indi'((\x_t)_\indi)}{(\x_t)_\indi - \x_\indi^*}.
\]
First note that by convexity, $\phi(t) \geq 0$. We will show that $\phi(t)$ is decreasing in $t$.
\begin{align*}
    \phi'(t) &= \frac{d \phi(t)}{d\x_\indi}\cdot \frac{d\x_\indi}{dt}
    \\ &= \roundbrack{\frac{g_\indi''((\x_t)_\indi)}{(\x_t)_\indi - \x_\indi^*} - \frac{g_\indi'((\x_t)_\indi)}{((\x_t)_\indi - \x_\indi^*)^2}}\roundbrack{-g_\indi'(\x_\indi)}.
\end{align*}
Suppose $((\x_t)_\indi - \x_\indi^*) \geq 0$. Then by convexity, $g_\indi'(\x_\indi) \geq 0$. Now observe that
\begin{align*}
    g_\indi'((\x_t)_\indi) &= \int_{\x_\indi^*}^{(\x_t)_\indi} g_\indi''(x) dx
    \\ &\leq \int_{\x_\indi^*}^{(\x_t)_\indi} g_\indi''((\x_t)_\indi) dx &\text{(since $g''$ is assumed to be non-decreasing)}
    \\ &= g_\indi''((\x_t)_\indi) ((\x_t)_\indi - \x_\indi^*).
\end{align*}
Thus $\phi'(t) \leq 0$. Alternately, suppose $((\x_t)_\indi - \x_\indi^*) \leq 0$. Then by convexity, $g_\indi'(\x_\indi) \leq 0$. Now observe that
\begin{align*}
    g_\indi'((\x_t)_\indi) &= \int_{\x_\indi^*}^{(\x_t)_\indi} g_\indi''(x) dx
    \\ &\geq \int_{\x_\indi^*}^{(\x_t)_\indi} g_\indi''((\x_t)_\indi) dx &\text{(since $g''$ is assumed to be non-decreasing)}
    \\ &= g_\indi''((\x_t)_\indi) ((\x_t)_\indi - \x_\indi^*).
\end{align*}
In this case too we observe that $\phi'(t) \leq 0$.
Thus, for every $t$ such that $(\x_t)_\indi \neq \x_\indi^*$, $\phi'(t) \leq 0$. Also, since $g_i$ is $\mu$-strongly convex and has $L$-Lipschitz gradients $\phi(t) \in [\mu, L]$. 
Suppose $\phi(t) = c \geq \mu$. Then for every $s \leq t$, $\phi(t) \geq c$. Consider the Lyapunov function $\lyapunov_s = e^{2cs}((\x_s)_\indi - \x_\indi^*)^2$. For $s \leq t$
\begin{align*}
\vel{\lyapunov}_s &= e^{2cs}\roundbrack{2c((\x_s)_\indi - \x_\indi^*)^2 - 2((\x_s)_\indi - \x_\indi^*)(\vel{(\x_s)_\indi})}
\\ &= e^{2cs}\roundbrack{2c(\x_s)_\indi - \x_\indi^*)^2 - 2(\x_s)_\indi - \x_\indi^*)^2\phi(s)}
\\ &\leq 0. 
\end{align*}
Thus 
\begin{align*}
     e^{2ct}((\x_t)_\indi - \x_\indi^*)^2 \lyapunov_t \leq \lyapunov_0 = ((\x_0)_\indi - \x_\indi^*)^2. 
\end{align*}
Since $\phi(t) = c$, we can compute the following bound on $\abs{g_\indi'((\x_t)_\indi)}$:
\begin{align*}
    \abs{g_\indi'((\x_t)_\indi)} &= c\abs{(\x_t)_\indi - \x_\indi^*}
    \\ &\leq ce^{-ct}\abs{(\x_0)_\indi - \x_\indi^*}.
\end{align*}
Thus if $(\x_s)_\indi \neq \x_\indi^*$, then $\abs{g_\indi'((\x_t)_\indi)} \leq  ce^{-ct}\abs{(\x_0)_\indi - \x_\indi^*}$ for some $ c \in [\mu, L]$. However, if $(\x_s)_\indi = \x_\indi^*$, then $\abs{g_\indi'((\x_t)_\indi)} = 0 \leq  ce^{-ct}\abs{(\x_0)_\indi - \x_\indi^*}$. Thus for every $t$, $\abs{g_\indi'((\x_t)_\indi)} \leq  ce^{-ct}\abs{(\x_0)_\indi - \x_\indi^*}$ for some $ c \in [\mu, L]$.

Now split the integral as follows: 
\begin{align*}
    \pathlength &= \int_0^\infty\sqrt{\sum_{i=1}^d \roundbrack{g_\indi'((\x_t)_\indi)}^2}\ dt
    \\ &= \underbrace{\int_0^{1/L}\sqrt{\sum_{i=1}^d \roundbrack{g_\indi'((\x_t)_\indi)}^2}\ dt}_{E_1} + \underbrace{\int_{1/L}^{1/\mu}\sqrt{\sum_{i=1}^d \roundbrack{g_\indi'((\x_t)_\indi)}^2}\ dt}_{E_2} + \underbrace{\int_{1/\mu}^\infty\sqrt{\sum_{i=1}^d \roundbrack{g_\indi'((\x_t)_\indi)}^2}\ dt}_{E_3}.
\end{align*}
To bound $E_1$ observe that for $ t \in [0, 1/L]$ and $c \in [\mu, L]$, $ce^{-ct} \leq Le^{-Lt}$. Thus
\begin{align*}
    E_1 &\leq \int_0^{1/L}Le^{-Lt} \sqrt{\sum_{i=1}^d \roundbrack{(\x_0)_\indi - \x_\indi^*}^2}\ dt
    \\ &= \roundbrack{1 - \frac{1}{e}} \enorm{\diffx }.
\end{align*}
Similarly for $E_3$ observe that for $ t \in [1/\mu, \infty)$ and $c \in [\mu, L]$, $ce^{-ct} \leq \mu e^{-\mu t}$. Thus
\begin{align*}
    E_3 &\leq \int_{1/\mu}^\infty \mu e^{-\mu t} \sqrt{\sum_{i=1}^d \roundbrack{(\x_0)_\indi - \x_\indi^*}^2}\ dt
    \\ &= \roundbrack{\frac{1}{e}} \enorm{\diffx }.
\end{align*}
To bound $E_2$, we will further split the integral. Define $\alpha_i = (\x_0)_\indi - \x_\indi^*$. Observe that for some fixed $t > 0$ and $\max_{c \geq 0} c e^{-ct} = (1/te)$. Thus for $ t \in [2^{k-1}/L, 2^k/L]$ and $c \geq 0$, $ce^{-ct} \leq \frac{L}{2^{k-1}e}$. Then

\begin{align*}
    E_2 &\leq \sum_{k=1}^{r} \int_{2^{k-1}/L}^{2^k/L}\sqrt{\sum_{i=1}^d \roundbrack{g_\indi'((\x_t)_\indi)}^2}\ dt
    \\ &\leq \sum_{k=1}^{r} \int_{2^{k-1}/L}^{2^k/L} \roundbrack{\frac{L}{2^{k-1}e}}\enorm{\diffx}\
    \\ &= \sum_{k=1}^{r} \roundbrack{\frac{2^k}{L} - \frac{2^{k-1}}{L}}\roundbrack{\frac{L}{2^{k-1}e}}\enorm{\diffx}
    \\ &= r\enorm{\diffx}/e
    \\ &\leq (\log_2(2\kappa)/e)\enorm{\diffx}
    \\ &\leq (1 + \log \kappa) \enorm{\diffx}.
\end{align*}
Resubstituting the bounds for $E_1,E_2,E_3$, we get 
\[
\pathlength \leq (2 + \log \kappa) \enorm{\diffx}.
\]
This completes the proof. 
\hfill\BlackBox

\section{Proofs of Results in Section~\ref{sec:quasiconvex}}
Each proof in this section is organized in a separate subsection. 
\label{appsec:proofsQuasiconvex}
\subsection{Proof of Lemma~\ref{lemma:QCGD}}
\label{appsec:gd-self-contracted-proof}
First, we show that $f(\x_t)$ is non-increasing with respect to $t$. For any $s \geq 0$ observe that by \LG,
\begin{align*}
    f(\x_{s+1}) &= f(\x_s - \eta \nabla f(\x_s)) 
    \\ &\leq f(\x_s) + \inner{\nabla f(\x_s), -\eta \nabla f(\x_s)} + \frac{L}{2}\esqnorm{\eta\nabla f(\x_s)}
    \\ &= f(\x_s) - \eta(1 - \eta L/2)\esqnorm{\nabla f(\x_s)}
    \\ &\leq f(\x_s) - \frac{\eta\esqnorm{\nabla f(\x_s)}}{2},
\end{align*}
for $\eta \leq 1/L$. Thus for any $s$, $f(\x_{s+1}) \leq f(\x_s)$, as was to be shown. Next, fix any iterate $t$. We will show that $\esqnorm{\x_s - \x_t}$ is non-increasing in $s$ for $s \leq t$. This would show self-contractedness for $s_3 = t$, for any $t$, concluding the proof. Consider any $s < t$, then
\begin{align*}
    \esqnorm{\x_{s+1} - \x_t} &= \esqnorm{\x_{s} -\eta \nabla f(\x_s) - \x_t} 
    \\ &= \esqnorm{\x_s - \x_t} + 2\inner{\eta \nabla f(\x_s), \x_t - \x_s} + \eta^2\esqnorm{\nabla f(\x_s)}
    \\ &\textleq{(i)} \esqnorm{\x_s - \x_t} + 2\eta(f(\x_t) - f(\x_s)) + \eta^2\esqnorm{\nabla f(\x_s)}
    \\ &\textleq{(ii)} \esqnorm{\x_s - \x_t} + 2\eta(f(\x_{s+1}) - f(\x_s)) + \eta^2\esqnorm{\nabla f(\x_s)}
    \\ &\textleq{\LG} \esqnorm{\x_s - \x_t} + 2\eta\roundbrack{\inner{\nabla f(\x_s), -\eta\nabla f(\x_s)} + \frac{\eta^2L}{2}\esqnorm{\nabla f(\x_s)}} + \eta^2\esqnorm{\nabla f(\x_s)}
    \\ &= \esqnorm{\x_s - \x_t} +\eta^2(\eta L - 1) \esqnorm{\nabla f(\x_s)}
    \\ &\leq \esqnorm{\x_s - \x_t},
\end{align*}
for $\eta \leq 1/L$. Above, inequality (i) holds because of convexity and inequality (ii) holds as we have shown that $f(\x_t)$ is non-increasing in $t$, and $s + 1 \leq t$. Thus, $\esqnorm{\x_s - \x_t}$ is non-increasing in the iterates $s$, as was to be shown. 
\hfill\BlackBox

As indicated after the statement of Lemma~\ref{lemma:QCGD}, if $f$ is convex and has $L$-Lipschitz gradients, then \GD with $\eta \in (0, 2/L]$ is a descent method (we show it below); however for \GD to be self-contracted, one needs further restriction on the step-size. To see the latter, consider $f : \Real \to \Real$ given by $f(x) = x^2$. Here we have $L = 2$. Let us set $\eta = 7/4L = 7/8$. Let $x_0 = 8$, so that $x_1 = 8 - (7/8)\cdot(2\cdot 8) = -6$ and $x_2 = -6 - (7/8)\cdot(2\cdot-6) = 4.5$. However, \[\abs{x_2 - x_1} > \abs{x_2 - x_0},\] and so $(x_0, x_1, x_2)$ cannot be part of a self-contracted curve.

  The fact that \GD is a descent method if $\eta \in (0, 1/L]$ is a well-known fact, but perhaps a bit less known is that this can be shown for $\eta \in (0, 2/L]$ (using standard techniques). We show it here for completeness. For any $\x^* \in \X^*$,
  \begin{align*}
    \esqnorm{\x_{s+1} - \x^*} &= \esqnorm{\x_{s} -\eta \nabla f(\x_s) - \x^*} 
    \\ &= \esqnorm{\x_s - \x^*} - 2\inner{\eta \nabla f(\x_s), \x_s - \x^*} + \eta^2\esqnorm{\nabla f(\x_s)}
    \\ &\textequal{(i)} \esqnorm{\x_s - \x^*} - 2\inner{\eta \nabla f(\x_s) - \eta \nabla f(\x^*), \x_s - \x^*} + \eta^2\esqnorm{\nabla f(\x_s)}.
    \\ &\textleq{(ii)} \esqnorm{\x_s - \x^*} - 2\eta\esqnorm{\nabla f(\x_s)}/L + \eta^2\esqnorm{\nabla f(\x_s)}.
    \\ &\leq \esqnorm{\x_s - \x^*},
  \end{align*}
  since $\eta \leq 2/L$ (with strict inequality in the last step if $\eta < 2/L$ and $\nabla f(\x_s) \neq 0$). Equality (i) holds because $\nabla f(\x^*) = 0$ and inequality (ii) holds since LG+convexity implies $\inner{\nabla f(\x) - \nabla f(\y), \x - \y} \geq \esqnorm{\nabla f(\x) - \nabla f(\y)}/L$. This is shown formally by \citet[Lemma 4]{zhou2018fenchel}.

\subsection{Proof of Theorem~\ref{thm:QC}}

The \GF result is due to~\citet{manselli1991maximum}. We analyze \GD curves by using Lemma~\ref{lemma:QCGD} to extend the techniques introduced by~\citet[Theorem~3.1]{daniilidis2015rectifiability} and~\citet{manselli1991maximum} for analyzing \GF curves. 

We will assume $d \geq 2$ since in the case $d = 1$ the self-contracted curve is the shortest path. Some definitions are in order:
\begin{itemize}
    \item The projection of any set $\convexHull$ to a line $u$ will be denoted as $\Pi_\uvec(\convexHull)$.
    \item Length of a one dimensional object (for example the projection of a set $\convexHull$ to a line $u$) will be denoted as $\ell(\cdot)$ (for example $\ell(\Pi_\uvec(\convexHull)$). 
    \item Mean width of a convex set $\convexHull$:
    \[
    W(\convexHull) := (\sigma_d)^{-1}\int_{\uvec \in \shell^{d-1}} \ell(\Pi_\uvec(K)) \ d\uvec,
    \]
    where $\sigma_d$ is the volume of $\shell^{d-1}$, with respect to the Lebesgue measure. 
    \item For $k \in \naturals_0$, $\futurePath(k)$ is the set of iterates after iteration $k$ (inclusive): $\futurePath(k) := \{\x_k, \x_{k+ 1}, \ldots\}$. 
    \item The convex closure of the set $\futurePath(k)$ will be denoted as $\convexClosure(k)$. 
\end{itemize}
Note that since the \GD curve is self-contracted and converges to $\x_\infty$, we have for all $t \in \naturals_0$,
\[\enorm{\x_t - \x_\infty} \leq \enorm{\x_0 - \x_\infty}. 
\]
Thus, all iterates $\x_0, \x_1, \ldots$ stay within a ball of radius $\enorm{\x_0 - \x_\infty}$ centered at $\x_\infty$. The mean width of this path can be at most the diameter of the ball, that is, $W(\convexClosure(0)) \leq 2\enorm{\x_0 - \x_\infty}$.
We will be showing that if $\eta \leq 1/L$
\begin{equation}
    \pathlength_\eta \leq \placeholder\cdot W(\convexClosure(0)),\label{eqn:goal1}
\end{equation}
and if $\eta \leq 1/2L\sqrt{d}$
\begin{equation}
    \pathlength_\eta \leq \improvedPlaceholder\cdot W(\convexClosure(0)),\label{eqn:goal2}
\end{equation}
which will lead to the bound in the theorem since $W(\convexClosure(0)) \leq 2\enorm{\x_0 - \x_\infty}$ and $d \geq 2$. Both these bounds will be shown by setting up a recurrence of the form. \begin{equation}
W(\convexClosure(k+1)) + \increment\enorm{\x_{k+1} - \x_k} \leq W(\convexClosure(k)). \label{eqn:recurrence}
\end{equation}
for two different values of $\increment$. 

By telescoping to $T$ iterations, this would lead to 
\[
\sum_{k = 0}^T \enorm{\x_{k+1} - \x_k} \leq \frac{1}{\increment}\roundbrack{W(\convexClosure(0)) -  W(\convexClosure(T+1))} \leq \frac{ W(\convexClosure(0))}{\increment}.
\]
Since the right hand side is the same for every $T$, indeed we would obtain
\[
\pathlength_\eta = \sum_{k = 0}^\infty \enorm{\x_{k+1} - \x_k}  \leq \frac{ W(\convexClosure(0))}{\increment}, 
\]
which would complete the proof. It remains to prove Equation~\eqref{eqn:recurrence} with the appropriate values of $\epsilon$ that would lead to Equations~\eqref{eqn:goal1} and~\eqref{eqn:goal2}.

\subsubsection{Proof of Equation~\eqref{eqn:goal1}}
\label{sec:proofGoal1}
We will show that if $\eta \leq 1/L$, Equation~\eqref{eqn:recurrence} is true with $\epsilon = (1/28)^{2d^2}$. Define the following entities (see Figure~\ref{fig:selfContractedIllustration}): 
\begin{align*}
    \x' &:= \frac{\x_{k+1}}{3} + \frac{2\x_k}{3}, \\
    \vvec &:= \frac{\x_{k} - \x_{k+1}}{\enorm{\x_{k} - \x_{k+1}}}, \\
    \xi'(\y) &:= \frac{\y - \x'}{\enorm{\y - \x'}}, \text{ for $\y \neq \x'$},\\
    \xi(\y) &:= \frac{\y - \x_{k+1}}{\enorm{\y - \x_{k+1}}}, \text{ for $\y \neq \x_{k+1} $}.
\end{align*}

\begin{figure}[t!]
    \centering
    \includegraphics[width=0.9\linewidth,trim=1.5cm 0.2cm 0.01cm 1.5cm]{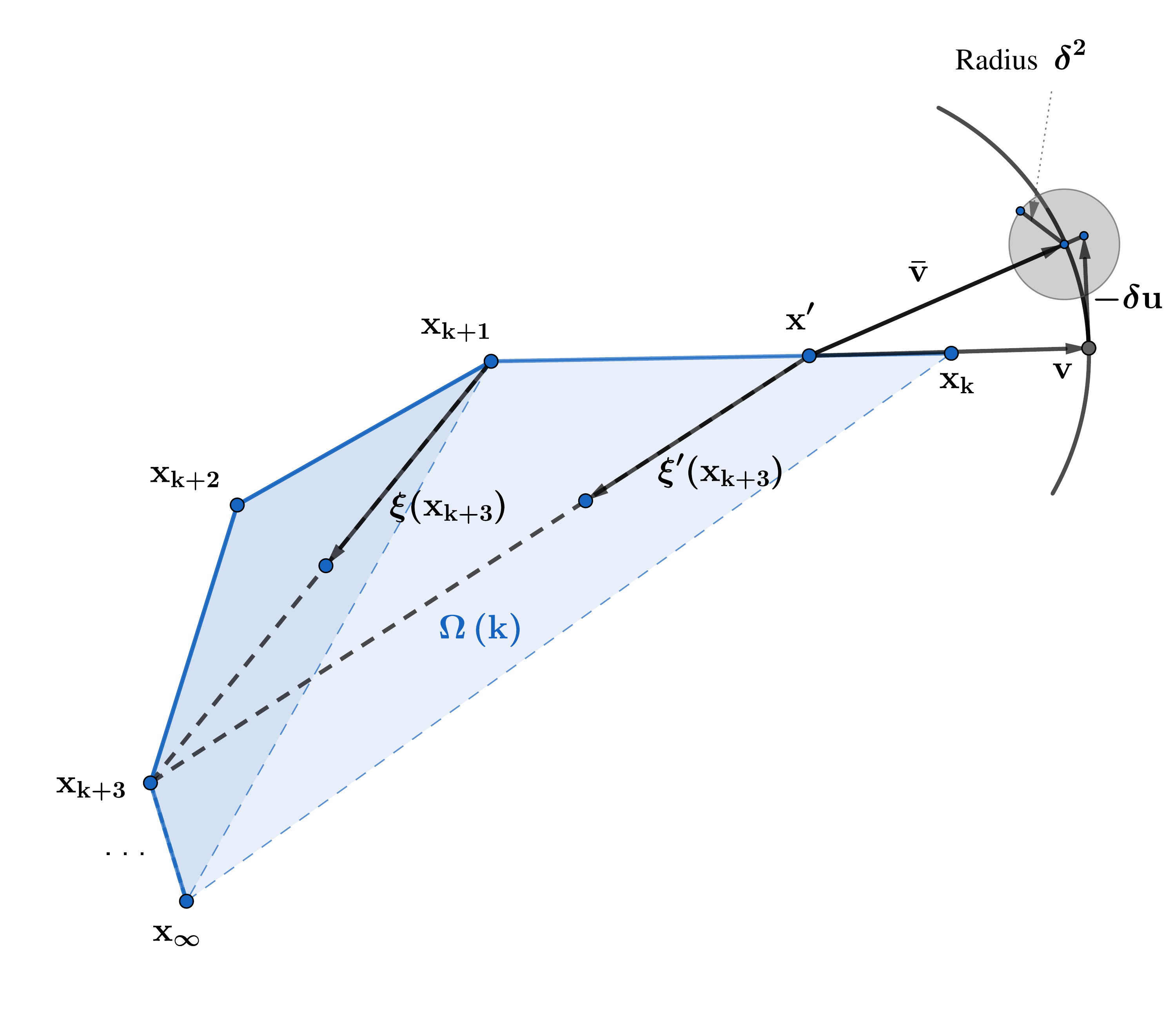}
    \caption{Illustration of some entities appearing in the proof of Theorem~\ref{thm:QC}.}
    \label{fig:selfContractedIllustration}
\end{figure}
Also denote $\vvec^\perp$ as the orthogonal hyperplane to $\vvec$. Define $\delta = (1/27)^d$. We will prove that for some unit vector $\uvec \in \vvec^\perp$, the unit vector 
\begin{equation}\Bar{\vvec} = \frac{\vvec - \delta \uvec}{\enorm{\vvec - \delta \uvec}}, \label{eqn:vbar}
\end{equation}
has at most a fixed negative inner product with any unit vector $\xi'(\x_t)$, namely: 
\begin{equation}
\label{eqn:negativeInner}
\inner{\Bar{\vvec}, \xi'(\x_t)} \leq -\delta^2. 
\end{equation}
As we see later, this will allow us to bound the mean width integral for $\convexClosure(k+1)$, in terms of the mean width integral for $\convexClosure(k)$ in order to prove Equation~\eqref{eqn:recurrence}. Note that since $\delta$ is a small constant, $\Bar{\vvec}$ is very close to $\vvec$. To motivate the truth of \eqref{eqn:negativeInner}, note the following fact about $\vvec$ itself. For all $\y \in \futurePath(k+2)$,
\begin{equation}
\label{eqn:step1}
\inner{\vvec, \y - \x'} \leq -\frac{\enorm{\x_k - \x_{k+1}}}{6} < 0. \qquad \roundbrack{\equiv \inner{\vvec, \xi'(\y)} \leq -\frac{\enorm{\x_k - \x_{k+1}}}{6\enorm{\y - \x'}} < 0.}
\end{equation}
This is true by the self-contractedness property. To see this, think of $\x'$ as the origin
, and $\x_k$ as the `positive' direction. Since $\enorm{\y - \x_{k+1}} \leq \enorm{\y - \x_k}$, the projection of $\y$ onto the segment $[\x_k, \x_{k+1}]$ lies towards the negative side and farther than the mid-point. However, $\x'$ lies towards $\x_k$, and hence the positive side. Thus, the projection of $\y - \x'$ points in the opposite direction as $\vvec$ and has a magnitude at least the distance between $\x'$ and the mid-point: $\enorm{\x_k - \x_{k+1}}/6$. Thus,
\[
\inner{\vvec, \y - \x'} \leq -\enorm{\x_k - \x_{k+1}}/6.
\]
The algebra above suggests that if the projection of $\y$ onto the segment $[\x_k, \x_{k+1}]$ is only slightly negative, then $\enorm{\y - \x'}$ is large and most of the component is in the $\vvec^\perp$ direction. In this case, we need to only find a small vector (namely $\delta \uvec$) in the perpendicular direction that has negative inner product with $\y - \x'$. This motivates the definition of $\Bar{\vvec}$ in \eqref{eqn:vbar}. 
Note however that this vector $\delta \uvec$ needs to uniformly have a negative inner product with respect to every $\xi'(\y)$. To show that this is possible, we will use the self-contractedness property to argue that all the $\xi'(\y)$ lie nearly in a hemisphere. In what follows, we formalize these ideas. 

First let us divide the unit vectors $\xi'(\x_t)$ into two sets: points that have a small component in the direction opposite to $\vvec$ and points that lie mostly in $\vvec^\perp$. Define, 
\[\futurePath' := \{\y \in \futurePath(k+1): \inner{\vvec, \xi'(\y)} \leq -2\delta\}.
\]
Note that for $\y \in \futurePath'$, 
\begin{align*}
\inner{\Bar{\vvec}, \xi'(\y)} &= \inner{\frac{\vvec - \delta \uvec}{\enorm{\vvec - \delta \uvec}}, \xi'(\y)} 
\\&\leq \frac{-2\delta}{\sqrt{1 + \delta^2}} + \delta &\text{(by Cauchy-Schwarz since $\enorm{\uvec} = \enorm{\vvec} = 1$)}
\\ &\leq -\delta^2,
\end{align*}
since $0 \leq \delta \leq 1/27$. Thus Equation~\eqref{eqn:negativeInner} is satisfied for $\y\in\futurePath'$. For $\y \in \futurePath \setminus \futurePath'$, we note three properties that will be used later. First, from \eqref{eqn:step1}, and the definition of $\futurePath'$
\begin{align}
    -2\delta < \inner{\vvec, \xi'(\y)} &\leq -\frac{\x_k - \x_{k+1}}{6\enorm{\y - \x'}}.  \nonumber
\end{align}
On cross multiplying
\begin{align}
 \enorm{
    \y - \x'} &> \frac{\enorm{\x_k - \x_{k+1}}}{12\delta}, \label{eqn:step2.1}
\end{align}
which says that $\enorm{\y - \x'}$ is large as motivated earlier. Thus, 
\begin{equation}
    \enorm{\y - \x} \geq \enorm{\y - \x'} - \enorm{\x' - \x} \geq \roundbrack{\frac{1}{12\delta} - \frac{2}{3}}\enorm{\x - \x'} = \frac{1 - 8\delta\enorm{\x - \x'}}{12\delta} \label{eqn:step3.1}
\end{equation}
Also, $-2\delta < \inner{\vvec, \xi'(\y)} < 0$, so that, 
\begin{equation}
    \abs{\inner{\vvec, \xi'(\y)}} < 2\delta.\label{eqn:step2.2}
\end{equation}
Let the component of $\xi'(\y)$ in $\vvec^\perp$ be $\xi'_\perp(\y)$. Thus, 
\begin{equation}
\esqnorm{\xi'_\perp(\y)} = 1 - (\inner{\vvec, \xi'(\y)})^2 \geq 1 - 4\delta^2. \label{eqn:step2.3}
\end{equation}

Using these facts, the goal now will be to show that for all $\y \in \futurePath \setminus \futurePath'$, $\xi'(\y)$'s are almost in a hemisphere so that there we can find a common vector $\uvec$ as desired which has a high negative inner product for \eqref{eqn:negativeInner}. The idea is that because of \eqref{eqn:step2.3}, $\xi'_\perp(\y)$ is almost perpendicular to $\vvec$ and so $\xi'(\y)$ looks almost like $\xi(\y)$. Observe that for $\xi(\y)$ the hemisphere property is easy to see: for all $\y, \z \in \futurePath(k+1) \setminus \futurePath'$ such that $\z$ comes after $\y$ in the path, using self-contractedness (Lemma~\ref{lemma:QCGD}), we know that $\enorm{\y - \z} \leq \enorm{\x_{k+1} - \z}$, and thus in the triangle formed by $\x_{k+1}, \y, \z$, the segment between $\y$ and $\z$ is not the longest side. This means that the angle at $\x_{k+1}$ is acute so that \begin{equation}
    \inner{\xi(\y), \xi(\z)} \geq 0. \label{eqn:quadrantLemma}
\end{equation} 
Hence all vectors $\{\xi(\y): \y \in \futurePath(k+1) \setminus \futurePath' \}$ belong in the same hemisphere. 
To show a similar result for $\xi'(\y)$, we first bound $\enorm{\xi(\y) - \xi'(\y)}$:
\begin{align*}
    \enorm{\xi(\y) - \xi'(\y)} &= \enorm{\frac{\y - \x_{k+1}}{\enorm{\y - \x_{k+1}}} - \frac{\y - \x'}{\enorm{\y - \x'}}}
    \\ &\leq \enorm{\frac{\y - \x_{k+1}}{\enorm{\y - \x_{k+1}}} - \frac{\y - \x'}{\enorm{\y - \x_{k+1}}}} + \enorm{\frac{\y - \x'}{\enorm{\y - \x_{k+1}}} - \frac{\y - \x'}{\enorm{\y - \x'}}}
    \\ &= \frac{\enorm{\x_{k+1} - \x'}}{\enorm{\y - \x_{k+1}}} + \frac{\enorm{\y - \x'}- \enorm{\y - \x_{k+1}}}{\enorm{\y - \x_{k+1}}}
    \\ &\leq \frac{\enorm{\x_{k+1} - \x'}}{\enorm{\y - \x_{k+1}}} + \frac{ \enorm{\x' - \x_{k+1}}}{\enorm{\y - \x_{k+1}}}
    \\ &= \frac{2\enorm{\x_{k+1} - \x'}}{\enorm{\y - \x_{k+1}}} = \frac{4\enorm{\x_{k+1} - \x'}}{3\enorm{\y - \x}}
    \\ &\textleq{\eqref{eqn:step3.1}} \frac{16\delta}{ 1- 8\delta}
    \\ &\leq 32\delta,
\end{align*}
since $\delta \leq 1/27$. Now we consider any $\y, \z \in \futurePath(k+1) \setminus \futurePath'$. Define $\delta_\y := \xi(\y) - \xi'(\y)$ and $\delta_\z := \xi(\z) - \xi'(\z)$, then,
\begin{align*}
    0 \leq \inner{\xi(\y), \xi(\z)} &=  \inner{\xi'(\y) + \delta_\y, \xi(\z)} 
    \\ &\leq \inner{\xi'(\y), \xi(\z)} + \enorm{\delta_\y}  &\text{(Cauchy-Schwarz)}
    \\ &\leq  \inner{\xi'(\y), \xi'(\z) + \delta_\z} + 32\delta 
    \\ &\leq \inner{\xi'(\y), \xi'(\z)} + 64\delta.
\end{align*}
Thus $\inner{\xi'(\y), \xi'(\z)} \geq - 64\delta$. Further, from \eqref{eqn:step2.2}
\begin{align*}
    \inner{\xi'_\perp(\y), \xi'_\perp(\z)} &= \inner{\xi'(\y), \xi'(\z)} - (\inner{\vvec, \xi'(\y)})(\inner{\vvec, \xi'(\z)})
    \\ &\geq -64\delta - 4\delta^2 
    \\ &\geq -65\delta, 
\end{align*}
since $\delta \leq 1/27$.
From \eqref{eqn:step2.3}, $\enorm{\xi'_\perp(\y)}\cdot \enorm{\xi'_\perp(\y)} \geq 1 - 4\delta^2$, so that the set of vectors $S = \{\widehat{\xi'_\perp(\y)}: \y \in \futurePath(k+1) \setminus \futurePath'\}$ ($\widehat{\mathbf{a}}$ denotes the unit vector in the direction of $\mathbb{\mathbf{a}}$) satisfies: for all $\y, \z \in S$, 
\begin{equation}
\inner{\y, \z} \geq \frac{-65\delta}{1 - 4\delta^2} \geq -66\delta = -66\roundbrack{\frac{1}{27}}^d \geq -\roundbrack{\frac{1}{3}}^d, \label{eqn:hemisphereHypothesis}
\end{equation}
for $d \geq 2$. As motivated earlier, this is a set of vectors that is almost in a hemisphere. At this point, we invoke the following lemma proved by~\citet{daniilidis2015rectifiability}. 

\begin{lemma}[Lemma 3.2,~\citep{daniilidis2015rectifiability}]
Let $\Sigma \subset \shell^{d-1}$ be a set satisfying
\[
\inner{\x, \y} \geq -\roundbrack{\frac{1}{3}}^{d+1} \text{ for all }\x, \y \in \Sigma.
\]
Then there exists a $\uvec \in \shell^{d-1}$ such that 
\[
\inner{\uvec, \y} \geq \roundbrack{\frac{1}{3}}^{2d+1} \text{ for all } \y \in \Sigma.
\]
\end{lemma}
The proof of the above lemma uses a packing argument. We use the lemma for the set $S$ identified above (Equation~\eqref{eqn:hemisphereHypothesis}). Note that all vectors in $S$ lie in $\shell^{d-1} \cap \vvec^\perp$, which can be identified as a shell in $d-1$ dimensions, homomorphic to $\shell^{d-2}$. Thus there exists a vector $\uvec \in \shell^{d-1} \cap \vvec$ such that for all $\y \in \futurePath(k+1) \setminus \futurePath'$,
\begin{equation}
    \inner{\uvec, \widehat{\xi'_\perp(\y)}} \geq \roundbrack{\frac{1}{3}}^{2(d-1)+1}  = \roundbrack{\frac{1}{3}}^{2d-1}. \label{eqn:hemisphereApplication}
\end{equation}

We pick this $\uvec$ to define $\Bar{\vvec}$ in Equation~\eqref{eqn:vbar}. Thus, for $\y \in \futurePath(k+1) \setminus \futurePath'$, 
\begin{align*}
    \inner{\vvec - \delta \uvec, \xi'(\y)} &\textleq{\eqref{eqn:step1}} -\delta \inner{ \uvec, \xi'(\y)}
    \\ &= -\delta\inner{ \uvec, \xi'_\perp(\y)}
    \\ &= -\delta\enorm{\xi'_\perp(\y)}\inner{ \uvec, \widehat{\xi'_\perp(\y)}}
    \\ &\textleq{\eqref{eqn:step2.3}} -\delta\sqrt{1 - 4\delta^2}\inner{ \uvec, \widehat{\xi'_\perp(\y)}}
    \\ &\textleq{\eqref{eqn:hemisphereApplication}} -\delta \sqrt{1 - 4\delta^2}\roundbrack{\frac{1}{3}}^{2d-1}
    \\ &\leq -\delta \roundbrack{\frac{1}{3}}^{2d},
\end{align*}
since $\delta \leq 1/27$. Finally, $\enorm{\vvec - \delta\uvec} = \sqrt{1 + \delta^2} \leq 3$ so that 
\[
\inner{\Bar{\vvec}, \xi'(\y)} = \frac{\inner{\vvec - \delta \uvec, \xi'(\y)}}{\enorm{\vvec - \delta \uvec}} \leq \frac{-\delta \roundbrack{\frac{1}{3}}^{2d}}{3} \leq -\delta \roundbrack{\frac{1}{3}}^{3d} = - \delta^2, 
\]
which gives us \eqref{eqn:negativeInner} as needed. 

Finally, we use this identified $\Bar{\vvec}$ to prove the recurrence ~\eqref{eqn:recurrence}. Consider the part of the shell $\shell^{d-1}$ $\delta^2$-close to $\Bar{\vvec}$:
\[
\shell' := \{\vvec' \in \shell^{d-1}: \enorm{\vvec' - \Bar{\vvec}} \leq \delta^2\},
\]
then
\begin{align}
    \sigma_d\width(\convexClosure(k+1)) &= \int_{\uvec \in \shell^{d-1}} \nonumber \ell(\Pi_\uvec(\convexClosure(k+1)))\ d\uvec
    \\ &= \int_{\uvec \in \shell'} \ell(\Pi_\uvec(\convexClosure(k+1))) d\uvec + \int_{\uvec \notin \shell'} \ell(\Pi_\uvec(\convexClosure(k+1)))\ d\uvec \nonumber
    \\ &\leq \int_{\uvec \in \shell'} \ell(\Pi_\uvec(\convexClosure(k+1))) d\uvec + \int_{\uvec \notin \shell'} \ell(\Pi_\uvec(\convexClosure(k))) \ d\uvec, \label{eqn:integralSplit} 
\end{align}  
since $\convexClosure(k+1) \subset \convexClosure(k)$. For the first integral above, note that for $\uvec \in \shell'$: 
\begin{enumerate}
    \item Inner product with the vector $\x_{k} - \x'$ is high:
    \begin{align*}
        \inner{\x_{k} - \x', \uvec} &= \enorm{\x_{k} - \x'} \inner{\vvec, \uvec} 
        \\ &= \roundbrack{\frac{\enorm{\x_{k} - \x_{k+1}}}{3}} \inner{\vvec, \uvec} 
        \\ &\geq \roundbrack{\frac{\enorm{\x_{k} - \x_{k+1}}}{3}}(1 - \delta - \delta^2) 
        \\ &\geq \frac{\enorm{\x_{k} - \x_{k+1}}}{4}.
    \end{align*}
    since $1 - \delta - \delta^2 \geq 3/4$ for $\delta \leq 1/27$.
    \item Inner product with the vector $\y - \x'$ for every $\y \in \futurePath(k+1)$ is non-positive: 
    \begin{align*}
        \inner{\y - \x', \uvec} &= \enorm{\y - \x'}\inner{\xi'(\y), \uvec} 
        \\ &= \enorm{\y - \x'}\inner{\xi'(\y), \vvec} + \enorm{\y - \x'}\inner{\xi'(\y), \vvec - \uvec}
        \\ &\textleq{\eqref{eqn:negativeInner}} \enorm{\y - \x'}(-\delta^2) + \enorm{\y - \x'}\inner{\xi'(\y), \vvec - \uvec}
        \\ &\leq -\enorm{\y - \x'}\delta^2 + \enorm{\y - \x'}\enorm{\vvec - \uvec}
        \\ &\leq -\enorm{\y - \x'}\delta^2 + \enorm{\y - \x'}(\delta^2)
        \\ &\leq 0.
    \end{align*}
    Indeed, this means that for any point in the convex hull of $\futurePath(k+1)$ the same is true---that is for $\y \in \convexClosure(k+1)$, $\inner{\y - \x', \uvec} \leq 0$. 
\end{enumerate}
Using these two facts, we have the following lower bound on the length of $\Pi_\uvec(\convexClosure(k))$ for any $\uvec \in \shell'$, 
\begin{align*}
    \ell(\Pi_\uvec(\convexClosure(k))) &\geq \inner{\x_k - \x', \uvec} + \ell(\Pi_\uvec(\convexClosure(k+1)))
    \\ &\geq \frac{\enorm{\x_{k} - \x_{k+1}}}{4} + \ell(\Pi_\uvec(\convexClosure(k+1))). 
\end{align*}
Thus, 
\[
    \int_{\uvec \in \shell'} \ell(\Pi_\uvec(\convexClosure(k + 1))) \ d\uvec \leq \int_{\uvec \in \shell'} \ell(\Pi_\uvec(\convexClosure(k))) \ d\uvec - \int_{\uvec \in \shell'} \roundbrack{\frac{\enorm{\x_k - \x_{k+1}}}{4}} \ d\uvec 
\]
\begin{figure}[t!]
    \centering
    \includegraphics[width=0.5\linewidth,trim=3cm 1.5cm 5cm 7.5cm]{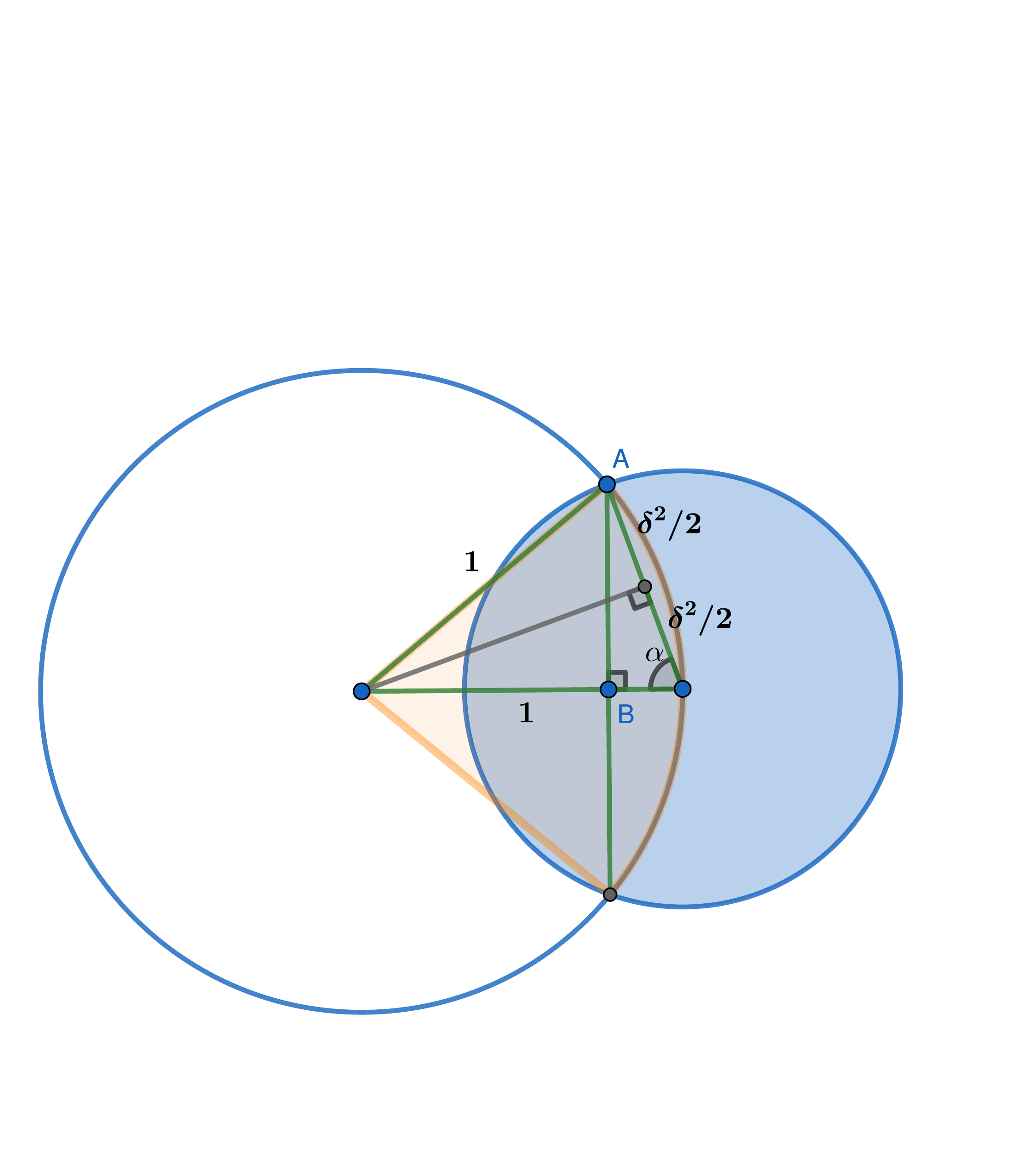}
    \caption{Two dimensional illustration of the intersection of a shell of radius $1$ (unshaded circle above) with a shell of radius $\delta^2$ (shaded circle above).}
    \label{fig:arcSurfaceArea}
\end{figure}
Continuing from \eqref{eqn:integralSplit}, and substituting the above inequality,
\begin{align*}
    \sigma_d\width(\convexClosure(k +1)) &\leq -\int_{\uvec \in \shell'} \roundbrack{\frac{\enorm{\x_k - \x_{k+1}}}{4}} \ d\uvec + \int_{\uvec \in \shell^{d-1}} \ell(\Pi_\uvec(\convexClosure(k)))\ d\uvec 
    \\ &= -\frac{\enorm{\x_k - \x_{k+1}}\cdot\text{Volume}(\shell')}{4} + \int_{\uvec \in \shell^{d-1}} \ell(\Pi_\uvec(\convexClosure(k)))\ d\uvec.
\end{align*}
On simplifying, this leads to the bound
\begin{align*}
 \width(\convexClosure(k+1)) &\leq - \underbrace{\roundbrack{\frac{\text{Volume}(\shell')}{4 \sigma_d}}}_{=:\ \increment, \text{ as needed for \eqref{eqn:recurrence}}}\enorm{\x_k - \x_{k+1}} + W(\convexClosure(k)),
\end{align*}
where Volume$(\cdot)$ is defined with respect to the Lebesgue measure in $d-1$ dimensions. To compute $\epsilon$, we find this volume. Note that $\shell'$ is a sector whose boundary is the intersection between $\shell^{d-1}$ and a shell of radius $\delta^2$ with center on the surface of $\shell^{d-1}$. This boundary defines a shell in $(d-1)$ dimensions. See Figure~\ref{fig:arcSurfaceArea}. In two dimensions the intersection is just two points, but for general $d$-dimensions, it is a shell in $(d-1)$-dimensions. The radius of this shell $\gamma$ is the length of the segment AB which can be calculated with simple trigonometric calculations. 
\[\gamma = \text{length(AB)} = \delta^2\sin\alpha = \delta^2\sqrt{1 - \cos^2\alpha} = \delta^2\roundbrack{1 - \frac{\delta^4}{4}} \geq \roundbrack{\frac{1}{28}}^{2d}.
\]
The volume of the sphere defined by this $(d-1)$-dimensional shell lower bounds the volume of $\shell'$. To illustrate in two dimensions (see Figure~\ref{fig:arcSurfaceArea}), think of this volume as the length of the orange arc that lies in the shaded circle. This length is lower bounded by the length of the diameter $2\gamma$. In general, for $d$-dimensions, this would be the volume of the $(d-1)$-dimensional sphere of radius $\gamma$. Using the formula of the volume of a $(d-1)$-dimensional sphere ($\Gamma$ below denotes the gamma function), 
\[
\text{Volume}(\shell') \geq \frac{((\pi\gamma^2)^{d-1/2} }{\Gamma\roundbrack{\frac{d-1}{2} + 1} }. 
\]
Also $\sigma_d$ is the volume of a $d$-dimensional shell, given by
\[
\sigma_d = \frac{d\pi^{d/2}}{\Gamma\roundbrack{\frac{d}{2} + 1}}.
\]
Thus,
\[
\increment = \frac{\text{Volume}(\shell')}{4\sigma_d} \geq \frac{\gamma^{d-1} \cdot\Gamma\roundbrack{\frac{d}{2} + 1}}{4\sqrt{\pi}\cdot\Gamma\roundbrack{\frac{d-1}{2} + 1}} \geq \gamma^d = \roundbrack{\frac{1}{28}}^{2d^2},
\]
as was needed to be shown to prove the recurrence \eqref{eqn:recurrence}.

\subsubsection{Proof of Equation~\eqref{eqn:goal2}}

We will show that for $\eta \leq 1/2L\sqrt{d}$, Equation~\eqref{eqn:recurrence} is true with $\epsilon = 1/{\improvedPlaceholder}$. We use some of the notation introduced in Section~\ref{sec:proofGoal1}. 

The step-size constraint ensures that gradients at two consecutive iterates have a high inner product (or small angle). For some $k \geq 0$, we have by \LG
\begin{align*}
    &\enorm{\nabla f(\x_k) - \nabla f(\x_{k+1})}  \leq L \enorm{\x_k - \x_{k+1}} = \eta L \enorm{\nabla f(\x_k)}
\end{align*}
Squaring both sides and rearranging leads to 
\begin{align*}
 &\esqnorm{\nabla f(\x_{k+1})} + (1 - \eta^2L^2) \esqnorm{\nabla f(\x_{k})}   \leq 2\inner{\nabla f(\x_{k+1}), \nabla f(\x_{k})} = 2\enorm{\nabla f(\x_k)}\enorm{\nabla f(\x_{k+1})} \cos{\theta},
\end{align*}
where $\theta$ is the angle between $\nabla f(\x_k)$ and $\nabla f(\x_{k+1})$. Further observe that
\begin{align*}
        \esqnorm{\nabla f(\x_{k+1})} + (1 - \eta^2L^2) \esqnorm{\nabla f(\x_{k})} \geq 2\sqrt{(1 - \eta^2L^2)}\enorm{\nabla f(\x_{k+1})} \enorm{\nabla f(\x_{k})}.
\end{align*}
Putting it together, we obtain
\begin{align*}
    &\sqrt{1 - \eta^2L^2} \leq \cos{\theta}, \nonumber
\end{align*}
which leads to the following dimension dependent lower bound 
\begin{align}
\sqrt{1 - (1/4d)} \leq \cos{\theta}\label{eqn:gradInner}.
\end{align}
As shown in Equation~\eqref{eqn:quadrantLemma}, $\inner{\xi(\y), \xi(\z)} \geq 0$ for $\y, \z \in \futurePath(k+2)$. Using this fact, we want to show that there exists a single unit vector $w$ such that $\inner{\w, \xi(\y)}$ is large for all $\y \in \futurePath(k+2)$.  To obtain this, we use a result due to~\citep[Theorem 1]{santalo1946convex}. To apply Santaló's theorem, consider the set $K = \{\xi(\y): \y \in \futurePath(k+2)\} \subset \shell^{d-1}$ (thus $n = d-1$). By Equation~\eqref{eqn:quadrantLemma}, the spherical diameter $D$ of $K$ has cosine at least $0$. We wish to lower bound the cosine of the spherical radius $R$. The result of Santaló has three case-wise conclusions, but each of them assert that if $\cos{D} \geq 0$, then 
\begin{align*}
    \frac{d\cos^2{R} - 1}{(\text{positive quantity})} \geq \cos{D} \geq 0.
\end{align*} 
{\it A fortiori} this implies that $d\cos^2{R} - 1 \geq 0$ or $\cos{R} \geq \sqrt{1/d}$. Using the definition of spherical radius, we conclude that there exists a $\w \in \shell^{d-1}$ such that for all $\y \in \futurePath(k+2)$
\begin{equation}
    \inner{\w, \xi(\y)} \geq \sqrt{1/d}.
    \label{positiveInner}    
\end{equation}
For $\y = \x_{k+2} = \x_{k+1} - \eta \nabla f(\x_{k+1})$, we have $\inner{\w, \nabla f(\x_{k+1}} \geq \sqrt{1/d}$. We use this fact to show that $\inner{\w, \xi(\x_k)}$ is negative, as follows. Let $\angle(\uvec, \vvec)$ denote $\arccos(\inner{\widehat{\uvec}, \widehat{\vvec}})$, where $\widehat\uvec$ and $\widehat\vvec$ are unit vectors in the direction of $\uvec$ and $\vvec$ respectively (so that $\arccos$ of their inner product gives us the angle). Then
\begin{align}
    \inner{\w, \xi(\x_k)} &= \cos(\angle(\w, \x_{k} - \x_{k+1}))
    \\ &= - \cos(\angle(\w, \nabla f(\x_k))) \nonumber
    \\ &\leq - \cos(\angle(\w, \nabla f(\x_{k+1})) + \angle(\nabla f(\x_k), \nabla f(\x_{k+1}))) \nonumber
    \\ &\leq - \cos(\angle(\w, \nabla f(\x_{k+1})))\cos(\angle(\nabla f(\x_k), \nabla f(\x_{k+1})))  \nonumber
    \\ & \qquad \qquad + \sin(\angle(\w, \nabla f(\x_{k+1})))\sin(\angle(\nabla f(\x_k), \nabla f(\x_{k+1})))  \nonumber
    \\ &\leq -\sqrt{1/d}\sqrt{1 - (1/4d)} + \sqrt{1-(1/d)}\sqrt{1/4d}  \nonumber
    \\ &\leq -\sqrt{1/4d}, \label{negativeInner}
\end{align}
for any $d \geq 2$. Define $\shell^\perp$ to be a shell in $d-1$ dimensions of unit vectors orthogonal to $\w$, ie $\shell^\perp = \{\uvec : \enorm{\uvec} = 1, \inner{\uvec, \w} = 0\}$. Now consider a set of unit vectors \emph{close} to $\w$ given by 
\[
\shell' = \{\uvec = \lambda \w + \sqrt{1 - \lambda^2}\ \w^\perp: \abs{\lambda} \in [\sqrt{1 - (1/4d)}, 1], \w^\perp \in \shell^\perp]\}. 
\]
We will relate the mean width integral of $\convexClosure(k)$ and $\convexClosure(k+1)$ by splitting across $\shell'$ and its complement. 

\begin{align}
    \sigma_d\width(\convexClosure(k+1)) &= \int_{\uvec \in \shell^{d-1}} \nonumber \ell(\Pi_\uvec(\convexClosure(k+1)))\ d\uvec
    \\ &= \int_{\uvec \in \shell'} \ell(\Pi_\uvec(\convexClosure(k+1)))\ d\uvec + \int_{\uvec \notin \shell'} \ell(\Pi_\uvec(\convexClosure(k+1)))\ d\uvec \nonumber
    \\ &\leq \int_{\uvec \in \shell'} \ell(\Pi_\uvec(\convexClosure(k+1)))\ d\uvec + \int_{\uvec \notin \shell'} \ell(\Pi_\uvec(\convexClosure(k))) \ d\uvec, \label{eqn:splitIntegral} 
\end{align}  
since $\convexClosure(k+1) \subset \convexClosure(k)$. Thus we reduce the second integral to the corresponding integral in the mean width calculation of $\convexClosure(k)$. The first part of the integral leads to a negative term which we upper bound to obtain Equation~\eqref{eqn:recurrence}. Pick any $\uvec = \lambda \w + \sqrt{1 - \lambda^2}\ \w^\perp \in \shell'$ and $\y \in \Gamma^{k+2}$. Suppose $\lambda \in [\sqrt{1 - (1/4d)}, 1]$, $\inner{\w, \xi(\y)} \geq \sqrt{1 - \lambda^2}$. Thus $\inner{\uvec, \xi(\y)} \geq 0$. Similarly, $\inner{\w, \xi(\x_k)} \leq -\sqrt{1 - \lambda^2}$, and so $\inner{\uvec, \xi(\x_k)} \leq 0$. Thus in such directions $u$, 
\[
\ell(\Pi_\uvec(\convexClosure(k))) - \ell(\Pi_\uvec(\convexClosure(k+1))) \geq \abs{\inner{\x_k - \x_{k+1}, \uvec}}.
\]
This is also true for the corresponding $-\lambda \in [\sqrt{1 - (1/4d)}, 1]$ (since $\ell(\Pi_\uvec(\cdot)) = \ell(\Pi_{-\uvec}(\cdot))$). Thus
\begin{align*}
    \int_{\uvec \in \shell'} \ell(\Pi_\uvec(\convexClosure(k)))\ d\uvec &\geq \int_{\uvec \in \shell'} \ell(\Pi_\uvec(\convexClosure(k+1)))\ d\uvec + \underbrace{\int_{\uvec \in \shell'}  \abs{\inner{\x_k - \x_{k+1}, \uvec}} d\uvec}_Z.
\end{align*}
We now lower bound the second term. To do so, we perform the integration over all values of $\lambda$ and $\vvec$ that determine $\uvec$. Note that if $\uvec = \lambda \w + \sqrt{1 - \lambda^2} \vvec$, $d\uvec = (\sqrt{1 - \lambda^2})^{d-2} d\vvec d\lambda$. Then
\begin{align*}
    Z &= 2\int_{\lambda \in [\sqrt{1 - (1/4d)}, 1]}\int_{\vvec \in \shell^\perp}  (\sqrt{1 - \lambda^2})^{d-2} \inner{\x_k - \x_{k+1}, \uvec}\ d\vvec d\lambda
    \\ &= 2\int_{\lambda \in [\sqrt{1 - (1/4d)}, 1]}\int_{\vvec \in \shell^\perp} (\sqrt{1 - \lambda^2})^{d-2} (\lambda \inner{\x_{k} - \x_{k+1}, \w} + \sqrt{1 - \lambda^2}\inner{\x_{k} - \x_{k+1}, \vvec})\ d\vvec d\lambda
    \\ &= 2\int_{\lambda \in [\sqrt{1 - (1/4d)}, 1]}\int_{\vvec \in \shell^\perp}   (\sqrt{1 - \lambda^2})^{d-2} (\lambda \inner{\x_{k} - \x_{k+1}, \w}) \ d\vvec d\lambda \qquad \qquad \roundbrack{\because \int_{\vvec \in \shell^\perp} \vvec d\vvec = 0}
    \\ &\geq  \roundbrack{\frac{2\sigma_{d-1}\enorm{\x_k - \x_{k+1}}}{\sqrt{4d}}} \int_{\lambda \in [\sqrt{1 - (1/4d)}, 1]} \lambda (\sqrt{1 - \lambda^2})^{d-2} 
     d\lambda
    \\ &= \frac{2\sigma_{d-1}\enorm{\x_k - \x_{k+1}}}{(\sqrt{4d})^d}.
\end{align*}
Using this value of $Z$ with Equation~\eqref{eqn:splitIntegral}, we obtain 
\begin{align*}
    \width(\convexClosure(k+1)) &\leq \width(\convexClosure(k)) - \increment'\enorm{\x_k - \x_{k+1}},
\end{align*}
where 
\[
\increment' = \frac{Z}{d} = \frac{2\sigma_{d-1}}{\sigma_d(\sqrt{4d})^d} \geq \frac{1}{d(\sqrt{4d})^d} \geq \frac{1}{\improvedPlaceholder}.
\]
This implies that Equation~\eqref{eqn:recurrence} holds for $\epsilon = 1/{\improvedPlaceholder}$, completing the proof. 
\hfill\BlackBox

\subsection{Proof of Theorem~\ref{thm:decomposable}}
The separability ensures that we are solving $d$ different optimization problems. For every index $i$, let $\X^*_i$ be the optimal set with respect to $g_i$. For such a one dimensional quasiconvex function, we showed in Lemma~\ref{lemma:QCGD} that $x$ follows a self-contracted curve. Thus, it cannot go in the opposite direction of the minima. By continuity in one dimension it clearly cannot overshoot. The length of this direct path is  $\dist{(\x_0)_{(i)}}{\X^*_i}$. Now observe that, \begin{align*}
    \int_0^\infty \enorm{\vel{\x}_t}\ dt &= \int_0^\infty \sqrt{\sum_{j=1}^d \roundbrack{\vel{\x}_t}_j^2} \ dt
    \\ &\leq \int_0^\infty \sum_{j=1}^d |\roundbrack{\vel{\x}_t}_j| \ dt
    \\ &= {\sum_{j=1}^d \abs{(\x_0)_j - \X^*_j}} 
    \\ &\leq \sqrt{d}\ \dist{\x_0}{\X^*}.
\end{align*}
For \GD, we have a similar proof. First, notice that the direction of the update is towards $(\x^*_j - (\x_k)_j )$. By quasiconvexity and since $g_j((\x_k)_j) \geq g_j((\x_0)_j)$, $(-\nabla g_j((\x_k)_j) (\x^*_j - (\x_k)_j )) \geq 0$. The only thing we need to show that the \GD curve does not overshoot. For an index $j$ let $\x^*_j$ be the closest minimum in $\X^*_j$ to $(\x_0)_j$. Then by \LG, $\abs{\nabla g_j((\x_k)_j} \leq L((\x_k)_j  - \x^*_j)$. Thus $\abs{\eta{\nabla g_j((\x_k)_j}} \leq \abs{(\x_k)_j  - \x^*_j}$, and the update cannot cross $\x^*_j$. 
Consequently for every $j$: $\sum_{k=0}^\infty ((\x_k)_j - (\x_{k+1})_j) = ((\x_0)_j - \x^*_j)$. Then
\begin{align*}
    \sum_{k=0}^\infty \enorm{\x_k - \x_{k+1}} &\leq \sum_{k=0}^\infty \norm{1}{\x_k - \x_{k+1}}
    \\ &= \sum_{k=0}^\infty\sum_{j=1}^d \abs{(\x_k)_j - (\x_{k+1})_j}
    \\ &= \sum_{j=1}^d\sum_{k=0}^\infty \abs{(\x_k)_j - (\x_{k+1})_j}
    \\ &= \sum_{j=1}^d\abs{(\x_0)_j - \x^*_j}
    \\ &\leq \sqrt{d} \enorm{\x_0 - \x^*},
\end{align*}
where in the last inequality we observed that for any vector $\uvec \in \Real^d$, $\norm{1}{\uvec} \leq \sqrt{d}\ \enorm{\uvec}$.
\hfill\BlackBox

\section{Proof of Theorem~\ref{thm:PLlowerBound}}
\label{appsec:proofsLowerBound}

We first provide a broad structure of the proof, and then prove the individual claims in separate subsections. \rev{The \PKL bound is proved first (Section~\ref{subsec:pkl-lb-fn}---\ref{subsec:pkl-lb-gd}), and then it is shown that the same function used in the \PKL lower bound also leads to a lower in the linear convergence case (Section~\ref{subsec:lb-linconv}). }

For any dimension $d \geq 6$, we will construct a function $f:\Real^d \to \Real$ that will be separable over its parameter $\x = (\x_{(1)}, \x_{(2)}, \ldots \x_{(d)}) \in \Real^d$, as per Equation~\eqref{eqn:decomposable}:
\[
f(\x) = \sum_{i=1}^d g(\x_{(i)}).
\]
$f$ will be constructed so that its condition number $\nu$ will be bounded as $\nu \leq  3d^2$. For this $f$ we will exhibit an $\x_0$ such that \GF or \GD have a path length lower bounded as $\pathlength \text { (or $\pathlength_\eta$) } \geq \frac{c\sqrt{d}}{\log d}\ \dist{\x_0}{\X^*}$ for $c = 1/6$ in the \GF case (see Equation~\eqref{eqn:finalLowerBoundGF} and $c = 1/16$ in the \GD case (see Equation~\eqref{eqn:finalLowerBoundGD}). 

To motivate this construction, we first prove that such a construction will lead to the statement of the theorem. Suppose the prescribed condition number bound $\kappa$ is such that $\kappa > 3d^2$. Then the condition number $\nu$ of the function we constructed is bounded by $\nu \leq \kappa$ and hence the function is in $\fancyF_\kappa$. Then

\[
\pathlength \text { (or $\pathlength_\eta$) } \geq \frac{c\sqrt{d}}{\log d}\ \dist{\x_0}{\X^*} \geq \min\curlybrack{\frac{c\sqrt{d}}{\log d}, \frac{c\kappa^{1/4}}{\log \kappa}}\ \dist{\x_0}{\X^*},
\]
completing the proof for $\kappa > 3d^2$.  

On the other hand suppose $\kappa \leq 3d^2$. Since $\kappa \geq 216$, $\sqrt{\kappa} - \sqrt{\kappa/2} > \sqrt{3}$, and so there exists a $\kappa/2 \leq \nu \leq \kappa$ such that $\sqrt{\nu/3}$ is an integer, say $d'$. Note that $d \geq d' \geq 6$ because $d' = \sqrt{\nu/3} \geq \sqrt{\kappa/6} \geq 6$ (this is the only part of the proof that uses $\kappa \geq 216$). Since $f$ is separable, we can simply ignore $d - d'$ components of $\Real^d$ and instead write $f(\x) = \sum_{i=1}^{d'} g(\x_{(i)})$. Via the same construction, the path length and condition number of $f$ will depend on $d'$ as follows:
\[
\text{the condition number is at most }\nu \leq \kappa, \text{ so that $f \in \fancyF_\kappa$},
\]and the path length is at least 
\[
\pathlength \text { (or $\pathlength_\eta$) } \geq \frac{c\sqrt{d'}}{\log d'}\ \dist{\x_0}{\X^*}.
\]
Since $216 \leq \kappa \leq 2\nu \leq 6d'^2$, $e \leq 6 \leq \sqrt{\kappa/6} = d'$. Again, since the function $\sqrt{x}/(\log x)$ is increasing in $x$ for $x \geq e^2$, we conclude

\begin{align*}
\frac{\sqrt{d'}}{ \log d'}  &\geq \frac{(\kappa/6)^{1/4}}{ \log (\sqrt{\kappa/6})}
\\ &\geq \frac{2(\kappa/6)^{1/4}}{ \log \kappa}
\\ &\geq \frac{\kappa^{1/4}}{ \log \kappa} .   
\end{align*}
This leads to the path length bound 
\[
\pathlength \text { (or $\pathlength_\eta$) } \geq \frac{c\sqrt{d'}}{\log d'}\ \dist{\x_0}{\X^*} \geq \min\curlybrack{\frac{c\sqrt{d}}{\log d}, \frac{c\kappa^{1/4}}{\log \kappa} }\ \dist{\x_0}{\X^*},
\]
completing the proof for $\kappa \leq 3d^2$. 

We now exhibit the pathological function $g$ that defines $f$ and prove some properties about it. 

\subsection{Construction of $\boldsymbol{g}$}
\label{subsec:pkl-lb-fn}
For any dimension $d \geq 6$ we will exhibit the $g$ such that for $f(\x) = \sum_{i=1}^d g(\x_{(i)})$, (a) the condition number $\nu$ of $f$ is bounded as $\nu \leq 3d^2$ and (b) the path length for some initial point $\x_0$ (different for \GF and \GD) is lower bounded by 
\[\pathlength \geq \frac{\sqrt{d}}{6\log d}\ \dist{\x_0}{\X^*} \text{ that is } c = 1/6,\]
for \GF and 
\[\pathlength_\eta \geq \frac{\sqrt{d}}{16\log d}\ \dist{\x_0}{\X^*} \text{ that is } c = 1/16,\]
for \GD with some $\eta \in [\nicefrac{1}{2L}, \nicefrac{1}{L}]$. As argued before this will prove the theorem statement. 

Define $\delta = 1/d$ and note that since $d \geq 6$,  $\delta \leq 0.2$. Define the  component function $g:\Real \to \Real$ as follows:
\begin{equation}
    g(x) =
\left\{
	\begin{array}{ll}
		0  & \mbox{if } x \leq 0 \\
		x^2 & \mbox{if } x \in [0, 0.5] \\
		0.5 - (1 - x)^2 & \mbox{if } x \in [0.5, 1 - \delta] \\
		(0.5 - \delta^2) + 2\delta(x - (1 - \delta)) & \mbox{if } x \in [1 - \delta, \gamma]
		\\ \alpha + \beta x^2 & \mbox{if } x \geq \gamma,
	\end{array}
\right.\label{eqn:gDefinition}
\end{equation}
\begin{figure}[t]
    \centering
    \includegraphics[width=\linewidth,trim=0cm 5cm 0cm 6cm, clip]{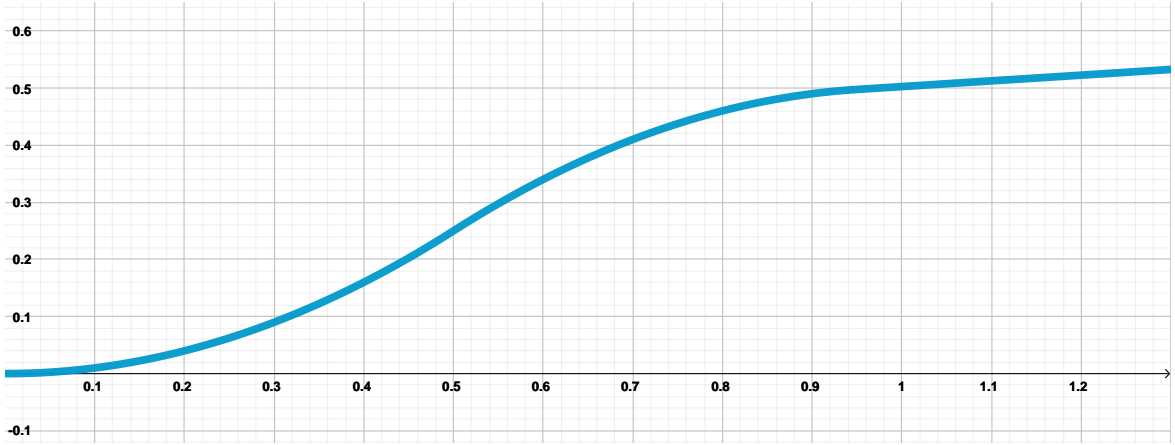}
    \caption{$g(\cdot)$ as defined in Equation~\eqref{eqn:gDefinition}.}
    \label{fig:PKLlowerbound}
\end{figure}
where 
\begin{align*}
\gamma &= 1 - \delta + 6 \log(1/2\delta),
\\\beta &= \frac{\delta}{\gamma},
\\\alpha &= (0.5 - \delta^2) + 2\delta(\gamma - (1 - \delta)) - \beta\gamma^2.
\end{align*}
(These precise values of $\alpha, \beta, \gamma$ are required for the function to satisfy differentiability at $x = \gamma$ and PKL at all points. For a reader interested in the broad idea and not the fine details we summarize the rationale behind setting these values: $\gamma$ is set so that \GF or \GD with the initial point (to be defined shortly) does not ever access the region $ x \geq \gamma$. Yet, to ensure that the function is \PKL everywhere, we have quadratic growth away from $\gamma$. $\alpha, \beta$ are set such that the function remains differentiable at $x = \gamma$.)

$g$ is plotted in Figure~\ref{fig:PKLlowerbound}. $g$ is everywhere continuously differentiable with the following gradient: 
\[\nabla g(x) =
\left\{
	\begin{array}{ll}
		0  & \mbox{if } x \leq 0 \\
		2x & \mbox{if } x \in [0, 0.5] \\
		2(1 - x) & \mbox{if } x \in [0.5, 1 - \delta] \\
		2\delta & \mbox{if } x  \in [1 - \delta, \gamma]  \\
		2\beta x & \mbox{if } x \geq \gamma.
	\end{array}
\right.
\]
Thus $f$ is also continuously differentiable. 
We consider gradient flow and gradient descent on $f$. We can write the update equations for each component separately:
\begin{align}
    \GF:& \qquad (\vel{\x})_{(i)} = -\nabla g(\x_{(i)}),\\
    \GD:& \qquad (\x^+)_{(i)} = -\eta \nabla g(\x_{(i)}).
\end{align}

We prove a path length lower bound starting from specific initialization points $\x_0$ which are introduced next. 

\subsection{Identification of $\x_0$}
\label{subsec:pkl-lb-x_0}
 We will need slightly different initialization points for \GF and \GD to ease computation (the main principle is the same). For \GF we will set the following initialization point $\x_0$:
\[(\x_0)_{(i)} =
\left\{
	\begin{array}{ll}
		0.5  & \mbox{if } i = 1 \\
		(1 - \delta) + \delta(i - 2)\log\roundbrack{1/2\delta} & \mbox{if } i > 1.
	\end{array}
\right.
\]
For \GD, for ease of computation set $\eta \in [1/4, 1/2]$ $(\equiv [1/2L, 1/L])$ such that 
\[
k_1 := \log_{1 + 2\eta}(1/2\delta) = \frac{\log(1/2\delta)}{\log(1 + 2\eta)}
\]
is a natural number. This is possible for $d \geq 6$ since \[\log(1/2\delta)~\geq~1.7 \text{ and } \roundbrack{\frac{1}{\log(2)}, \frac{1}{\log(1.5)}} \supset (1.5, 2.4) \text{ so that }\log(1/2\delta)(2.4 - 1.5) \geq 1.\] 
Observe that $k_1 \leq 3\log(1/2\delta)$. Given this $\eta$ and $k_1$, define the following initialization point $\x_0$ for \GD: 
\[(\x_0)_{(i)} =
\left\{
	\begin{array}{ll}
		0.5  & \mbox{if } i = 1 \\
		(1 - \delta) + 2\eta k_1\delta(i - 2) & \mbox{if } i > 1.
	\end{array}
\right.
\]
The \GD and \GF bounds follow the same technique but the precise computations are slightly different. Thus we write the \GD and \GF bounds separately in the following subsections. The main idea behind the staggering us discussed in the main paper and illustrated in Figure~\ref{fig:PKLiterates}. Because of the staggering, in every consecutive iterate, a single component goes from value $0.5$ to $0.0$ leading
a large path length. 
\subsection{GF Analysis}
\label{subsec:pkl-lb-gf}
We make the following observations about the function $f$ and the initial point $\x_0$. 
\begin{enumerate}[label=(1.\arabic*)]
    \item The distance between the initial and the optimal set is bounded as, 
    \begin{align}
        \dist{\x_0}{\X^*} = \enorm{\x_0 - \mathbf{0}} &\leq \sqrt{\sum_{i = 1}^d \roundbrack{1 + (i\delta\log(1/2\delta))}^2}\nonumber
        \\&\leq \sqrt{\sum_{i = 1}^d \roundbrack{1 + (i\log(d)/d)}^2}\nonumber
        \\ &\leq \sqrt{2\sum_{i = 1}^d \roundbrack{1 + (i\log(d)/d)^2}}\nonumber
        \\ &= \sqrt{2d + 2\sum_{i = 1}^d  (i\log(d)/d)^2}\nonumber
        \\ &= \sqrt{2d +  2\roundbrack{\frac{d(d+1)(2d+1)\log^2d}{6d^2}}} \nonumber
        \\ &\leq \sqrt{2d + \roundbrack{d\log^2d}} &\text{(since $d \geq 6$)} \nonumber
        \\ &\leq \sqrt{d\log^2d +  d\log^2d} &\text{(since $d \geq 6$)} \nonumber
        \\ &= \sqrt{2d}\ \log d. \label{eqn:distUBa}
    \end{align}
    \item $f^* = 0$.
    
    \item The gradients of $f$ are $L$-Lipschitz with $L = 2$. To see this, first notice that the gradients of $g$ are $2$-Lipschitz since they are $2$-Lipschitz in each of the pieces in the definition~\eqref{eqn:gDefinition}, and the derivatives are continuous. Then 
    \begin{align*}
        \inner{\nabla f(\x) - \nabla f(\y), \x - \y} &= \sum_{i=1}^d (\nabla g(\x_{(i)}) -\nabla  g(\y_{(i)}))(\x_{(i)} - \y_{(i)}) 
        \\ &\leq \sum_{i=1}^d 2(\x_{(i)} - \y_{(i)})^2 &\text{($\nabla g$ is $2$-Lipschitz)}
        \\ &= 2\esqnorm{\x - \y}.
    \end{align*}
    \item $f$ is $\mu$-\PKL for $\mu = 2/3d^2$. To see this, first observe that 
    \[
    \frac{\esqnorm{\nabla f(\x)}}{2(f(\x) - f^*)} = \frac{\esqnorm{\nabla f(\x)}}{2f(\x)} \geq \min_i\roundbrack{\min_{\x_{(i)} \in (0, \max_i(\x_0)_{(i)}]} \frac{(\nabla g(\x_{(i)}))^2}{2g(\x_{(i)})}}
     = \min_{x \in (0, \max_i(\x_0)_{(i)}]}\frac{(\nabla g(x))^2}{2g(x)}.\]
     We bound the final quantity for each piece in the definition~\eqref{eqn:gDefinition} where $g(x) \neq 0$:
     \begin{itemize}
         \item $x \in (0, 0.5]$. $g(x) = x^2$, $(\nabla g(x))^2 = 4x^2$. Thus $(\nabla g(x))^2/2g(x) = 2 \geq \mu$. 
         \item $x \in [0.5, 1- \delta]$. $g(x) \leq 1$, $(\nabla g(x))^2 \geq (2\delta)^2 = 4\delta^2$. Thus $(\nabla g(x))^2/2g(x) \geq 2\delta^2 \geq \mu$.
         \item $x \in [1 - \delta, \gamma]$. Here, $(\nabla g(x))^2 = 4\delta^2$, and 
         
         \begin{align*}
             \max_{x \in [1-\delta, \gamma ]} g(x) = g(\gamma) = &= (0.5 - \delta^2) + 12\delta\log(1/2\delta)
             \\ &\leq 0.5 + 12\log(d/2)/d 
             \\ &\leq 3 &\text{(for all $x$, $12\log(x/2)/x  \leq 2.5$)}.
         \end{align*}
         Thus, $(\nabla g(x))^2/2g(x) \geq 2\delta^2/3 = \mu$
         
         \item $x \in [\gamma, \infty)$. $(\nabla g(x))^2 = 4\beta^2x^2$, and $g(x) = (\alpha + \beta x)^2$. The ratio is minimized at $x = \gamma$, where $g(x) = (0.5 - \delta^2) + 12\delta\log(1/2\delta)$:
         \begin{align*}
             \frac{4\beta^2\gamma^2 }{g(\gamma)} &= \frac{2\delta^2 }{0.5 + 12\log(d/2)/d}
             \\ &\geq \frac{2\delta^2 }{3} &\text{(for all $x$, $12\log(x/2)/x  \leq 2.5$)}
             \\ &= \mu.
         \end{align*}
     \end{itemize}
     \item The condition number of $f$, $\nu = L/\mu \leq 3d^2$.
\end{enumerate}
Consider the time interval $[0, t_1]$ where $t_1$ is given by,
\[
t_1 = \frac{\log \roundbrack{1/2\delta}}{2}.
\]
We make the following computations to determine the value of $\x_{t_1}$:
\begin{enumerate}[label=(2.\arabic*)]
    \item $(\x_{t_1})_{(1)}$: The flow for $(\x_t)_{(1)}$ in the interval $[0, 0.5]$ is given as $(\x_t)_{(1)} = 0.5e^{-2t}$. Thus
    \[
        (\x_{t_1})_{(1)} = 0.5e^{-\log(1/2\delta)} = \delta.
    \]
    \item $(\x_{t_1})_{(2)}$: The flow for $(\x_t)_{(2)}$ in the interval $[0.5, 1 - \delta]$ is given as $(\x_t)_{(2)} = 1 - \delta e^{2t}$. As computed below, $(\x_{t})_{(2)}$ decreases from  $(1 - \delta)$ to $0.5$ for $t \in [0, t_1]$, and achieves the value $0.5$ at $t_1$:
    \[
        (\x_{t_1})_{(2)} = 1 - \frac{\delta}{2\delta} = 0.5.
    \]
    \item $(\x_{t_1})_{(i)}$, for $i \geq 3$: The flow for $(\x_t)_{(i)}$ in the interval $[1 - \delta, \infty)$ is given as $(\x_t)_{(i)} = (1 - \delta) + \delta(i - 2)\log\roundbrack{1 + 2\delta^2} - 2\delta t$. Thus,
    \begin{align*}
        (\x_{t_1})_{(i)} &= (1 - \delta) + \delta(i - 2)\log\roundbrack{1/2\delta} - 2\delta t_1
        \\ &= (1 - \delta) + \delta(i - 3)\log\roundbrack{1/2\delta}.
    \end{align*}
\end{enumerate}
Given this, first we lower bound the path length for the interval $[0, t_1]$, the path length is at least:
\[
\enorm{\x_0 - \x_{t_1}} \geq (\x_0)_{(2)} - (\x_{t_1})_{(2)} = 0.5 - \delta \geq 0.3.  
\]
Next we perform the same computations for the interval $[t_1, 2t_1]$. In observations (2.1), (2.2), (2.3) we obtained 
\[(\x_{t_1})_{(i)} =
\left\{
	\begin{array}{ll}
		\delta  & \mbox{if } i = 1 \\
		0.5  & \mbox{if } i = 2 \\
		(1 - \delta) + \delta(i - 3)\log\roundbrack{1/2\delta} & \mbox{if } i > 2.
	\end{array}
\right.
\]
Compare this to $\x_0$. Observe that $(\x_{t_1})_{(1)} \leq \delta$, and that for $i > 1$, $(\x_{t_1})_{(i)} = (\x_{0})_{(i-1)}$. Thus, for the time interval $[t_1, 2t_1]$, $((\x_{t})_{(2)}, (\x_{t})_{(3)}, \ldots (\x_{t})_{(d)})$ follow the exact same dynamics as did $((\x_{t})_{(1)}, (\x_{t})_{(2)}, \ldots (\x_{t})_{(d-1)})$ for the time interval $[0, t_1]$.
Continuing in this fashion, more generally for every $s \in \{0, 1, \ldots d-2\}$, the same computations for the interval $[st_1, (s+1)t_1]$ lead to a path length which is at least:
\[
\enorm{\x_{st_1} - \x_{(s+1)t_1}} \geq (\x_{st_1})_{(s+2)} - (\x_{(s+1)t_1})_{(s+2)} = 0.5 - \delta \geq 0.3.  
\]
Adding up all path length lower bounds for $s \in \{0, 1, \ldots, d-2\}$ we obtain a lower bound on the overall path length:
\begin{align}
    \pathlength &\geq 0.3(d-1)  \nonumber
    \\ &\geq \roundbrack{\frac{\sqrt{d}}{6\log d}}\roundbrack{\sqrt{2d}\log d} &\text{(since $d \geq 6$)} \nonumber
    \\ &\geq \roundbrack{\frac{\sqrt{d}}{6\log d}}\dist{\x_0}{\X^*}, \label{eqn:finalLowerBoundGF}
\end{align}
where the last inequality uses Equation~\eqref{eqn:distUBa}. This concludes the proof for the \PKL lower bound in the \GF case. 

\subsection{GD Analysis}
\label{subsec:pkl-lb-gd}
Consider the iterates $k \in \{1, 2, \ldots k_1\}$. As noted previously, $k_1$ is given by
\[
k_1 = \log_{1 + 2\eta}\roundbrack{1/2\delta}.
\]
First, we make the following observations analogous to the ones made in the \GF case (observations 1.1--1.5). The details can be found in the \GF analysis.
\begin{enumerate}[label=(3.\arabic*)]
    \item The distance between the initial and the optimal set is bounded as, 
    \begin{align}
        \dist{\x_0}{\X^*} = \enorm{\x_0 - \mathbf{0}} &\leq \sqrt{\sum_{i = 1}^d \roundbrack{1 + (2i\delta k_1)}^2} \nonumber
        \\ &= \sqrt{2d + 36\sum_{i = 1}^d (i\log(d)/d)^2} &\text{(since $k_1 \leq 3\log(1/2\delta))$} \nonumber
        \\ &= \sqrt{2d +  36\roundbrack{\frac{d(d+1)(2d+1)\log^2d}{6d^2}}}\nonumber
        \\ &\leq \sqrt{2d + \roundbrack{13d\log^2d}} &\text{(since $d \geq 6$)}\nonumber
        \\ &\leq 4\sqrt{d}\ \log d. \label{eqn:distUB}
    \end{align}
    \item $f^* = 0$.
    \item The gradients of $f$ are $L$-Lipschitz with $L = 2$. This is the same observation as (1.3). 
    \item $f$ is $\mu$-\PKL for $\mu = 2/3d^2$. This is the same observation as (1.4). 
     \item The condition number of $f$, $\nu = L/\mu \leq 3d^2$.
\end{enumerate}
We make the following computations to determine the value of $\x_{k_1}$:
\begin{enumerate}[label=(4.\arabic*)]
    \item $(\x_{k_1})_{(1)}$: 
    \[
        (\x_{k_1})_{(1)} \leq (\x_1)_{(1)} = 0.5 - \eta \leq 0.25.
    \]
    \item $(\x_{k_1})_{(2)}$: The iterates for $(\x_k)_{(2)}$ for $k \in \{1, 2, \ldots k_1\}$ are $(\x_k)_{(2)} = 1 - (1 + 2\eta)^k\delta$. As computed below, $(\x_{k})_{(2)}$ decreases from  $(1 - \delta)$ to $0.5$ for $k \in \{1, 2, \ldots k_1\}$ and achieves the value $0.5$ at $k_1$:
    \[
        (\x_{k_1})_{(2)} = 1 - \frac{\delta}{2\delta} = 0.5.
    \]
    \item $(\x_{k_1})_{(i)}$, for $i \geq 3$: The updates for $(\x_k)_{(i)}$ in the interval $[1 - \delta, \infty)$ are given as $(\x_k)_{(i)} = (1 - \delta) + 2\eta\delta(i - 2)k_1 - 2\eta\delta k$. Thus,
    \begin{align*}
        (\x_{k_1})_{(i)} &= (1 - \delta) + 2\eta\delta(i - 2)k_1 - 2\eta\delta k_1
        \\ &= (1 - \delta) + 2\eta\delta(i - 3)k_1.
    \end{align*}
\end{enumerate}
Given this, first we lower bound the path length for the iterates $\{1, 2, \ldots k_1\}$. The path length is at least:
\[
\enorm{\x_0 - \x_{k_1}} \geq (\x_0)_{(2)} - (\x_{k_1})_{(2)} = 0.5 - \delta \geq 0.3.  
\]
Next we perform the same computations for the interval $[t_1, 2t_1]$. Through observations (4.1), (4.2), (4.3) we obtained $(\x_{k_1})_{(1)} \in [0, 0.5 - \eta]$ and 
\[(\x_{k_1})_{(i)} =
\left\{
	\begin{array}{ll}
		%0.5 - \eta  & \mbox{if } i = 1 \\
		0.5  & \mbox{if } i = 2 \\
		(1 - \delta) + 2\eta\delta(i - 3)k_1 & \mbox{if } i > 2.
	\end{array}
\right.
\]
Compare this to $\x_0$. Observe that for $i > 1$, $(\x_{k_1})_{(i)} = (\x_{0})_{(i-1)}$. Thus, for the iterates $\{k_1 + 1,\ldots  2k_1\}$, $((\x_{k})_{(2)}, (\x_{k})_{(3)}, \ldots (\x_{k})_{(d)})$ follow the exact same dynamics as did $((\x_{k})_{(1)}, (\x_{k})_{(2)}, \ldots (\x_{k})_{(d-1)})$ for the time interval $\{1, 2, \ldots k_1\}$. 
Continuing in this fashion, more generally for every $s \in \{0, 1, \ldots d-2\}$, the same computations for the interval $\{sk_1 + 1, sk_1 + 2, \ldots (s+1)k_1\}$ lead to a path length bound at least as large as:
\[
\enorm{\x_{sk_1} - \x_{(s+1)k_1}} \geq (\x_{sk_1})_{(s+2)} - (\x_{(s+1)k_1})_{(s+2)} = 0.5 - \delta \geq 0.3.  
\]
Adding up all path length lower bounds for $s \in \{0, 1, \ldots, d-2\}$ we obtain a lower bound on the overall path length:
\begin{align}
    \pathlength_\eta &\geq 0.3(d-1) \nonumber
    \\ &\geq \roundbrack{\frac{\sqrt{d}}{16\log d}}\roundbrack{4\sqrt{d}\ \log d} \nonumber
    \\ &\geq \roundbrack{\frac{\sqrt{d}}{16\log d}}\dist{\x_0}{\X^*}, \label{eqn:finalLowerBoundGD}
\end{align} 
where the last inequality uses the bound~\eqref{eqn:distUB}. This concludes the proof for the \PKL lower bound in the \GD case. 

\rev{\subsection{Lower Bound Under Linear Convergence}
\label{subsec:lb-linconv}
We compute the linear convergence constant $c$ for the function $f$ and relate it to the path length lower bounds shown in Section~\ref{subsec:pkl-lb-gf} (\GF) and Section~\ref{subsec:pkl-lb-gd} (\GD). }

\textbf{\GF analysis.} Suppose $\{x_t\}_{t\in\Real_0^+}$ follows the dynamics $\vel{x} = -\nabla g(x)$ with some initial point. Note that the optimal set is $X^*= [-\infty, 0]$. Then for any $x > 0$, 
\[
\frac{\vel{x}}{\dist{x}{X^*}} = \frac{\vel{x}}{x}  = -\frac{\nabla g(x)}{x} \geq -\min(2, 2\delta, 2\delta/\gamma, 2\beta).
\]
This corresponds to the definition $\nabla g(x)$ in different regions. The minimum above is given by $2\delta/\gamma$ since $\delta \leq 1$ and $\gamma > 1$ for $d = 6$. We compute $2\delta/\gamma = 2d^{-1}(1 + 6 \log(d/2))^{-1} \geq (4d\log d)^{-1}$. Thus linear convergence holds for every component of $f$ with $A = 1$ and $c = (4d\log d)^{-1}$; consequently linear convergence also holds for $f$ with constants $(1, c)$. For the $\x_0$ chosen in Section~\ref{subsec:pkl-lb-x_0}, we showed a lower bound on the path length in terms of $d$ (Section~\ref{subsec:pkl-lb-gf}) that can be translated into a lower bound in terms of $c$ as follows: 
\[
\zeta \geq \roundbrack{\frac{\sqrt{d}}{6\log d}} \dist{\x_0}{\X^*} \geq \roundbrack{\frac{\sqrt{1/c}}{12\log^{1.5}(1/c)}}\dist{\x_0}{\X^*}.
\]
This concludes the analysis in terms of the linear convergence constant $c$ for \GF. Since $d \geq 6$, this construction admits $c \in (0, (4\cdot 6\log 6)^{-1}) \supseteq (0, 0.023)$.

\textbf{\GD analysis.} As identified in Section~\ref{subsec:pkl-lb-x_0}, $\eta \geq 1/4$. Consider any $x > 0$ (corresponding to one component of the iterates) and note that the optimal set is $X^*= [-\infty, 0]$. Then for an update $x^+ \leftarrow x - \eta \nabla g(x)$ we have %note that $\vel{x} = -\nabla g(x)$ with some initial point. Note that the optimal set $X^*= [-\infty, 0]$. Then for any $x > 0$, 
\[
\frac{x^+ - x}{\dist{x}{X^*}} = \frac{- \eta \nabla g(x)}{x}  \geq -\eta(4d\log d)^{-1} \geq (16d\log d)^{-1} .
\]
(The lower bound for $\frac{\nabla g(x)}{{x}}$ is shown in the \GF computation above and the final inequality uses $\eta \geq 1/4$.)
Thus linear convergence holds for every component of $f$ with $c = (16d\log d)^{-1}$; consequently linear convergence also holds for $f$ with this value of $c$. For the $\x_0$ chosen in Section~\ref{subsec:pkl-lb-x_0}, we showed a lower bound on the path length in terms of $d$ (Section~\ref{subsec:pkl-lb-gd}) that can be translated into a lower bound in terms of $c$ as follows: 
\[
\zeta \geq \roundbrack{\frac{\sqrt{d}}{16\log d}} \dist{\x_0}{\X^*} \geq \roundbrack{\frac{\sqrt{1/c}}{64\log^{1.5}(1/c)}}\dist{\x_0}{\X^*}.
\]
This concludes the analysis in terms of the linear convergence constant $c$ for \GD. Since $d \geq 6$, this construction admits $c \in (0, (4\cdot 6\log 6)^{-1}) \supseteq (0, 0.0058)$.
 
\hfill\BlackBox

\section{Proof of Theorem~\ref{thm:PLlowerBoundQuadratic}}
\label{appsec:proofLowedQuadratic}
For any dimension $d$, we will construct a separable quadratic function in $d$ dimensions: 
\[
f(\x) = \frac{1}{2}\sum_{i=1}^d a_i \x_{(i)}^2,
\]
where $a_i = \omega^{d-i}$ for $\omega = 11$. Since $f$ is a separable quadratic function, it is also a separable quasiconvex function. The condition number of $f$ is $\nu = \omega^{d-1}$, so that $\log \nu = (d-1)(\log\omega)$ and in particular $\sqrt{\log \nu} \leq \sqrt{d(\log\omega)}$. Also $\X^* = \{\mathbf{0}\}$. In what follows we will identify an $\x_0$ such that the path length for \GF satisfies $\pathlength \geq 0.7 \sqrt{d}\ \dist{\x_0}{\X^*}$ and the path length for \GD for $\eta = 1/2a_1 = 1/2L$ satisfies $\pathlength_\eta \geq 0.5 \sqrt{d}\ \dist{\x_0}{\X^*}$. Before we illustrate this construction, we prove that such a construction would lead to the results in Equations~\eqref{eqn:quadraticLowerGF} and~\eqref{eqn:quadraticLowerGD}. 

For the given value of $\kappa$ and $d$, we consider two cases. Suppose $\log \kappa \geq (d-1)(\log \omega)$. Then, the above function with condition number $\nu$ satisfies $\nu \geq \kappa$ so that $f \in \fancyQ_\kappa$. The construction ensures that
\[
\pathlength \text { (or $\pathlength_\eta$) } \geq c\sqrt{d}\ \dist{\x_0}{\X^*},
\]
for appropriate values of $c$ which implies Equations~\eqref{eqn:quadraticLowerGF} and~\eqref{eqn:quadraticLowerGD}. 

On the other hand, suppose $\log \kappa \leq (d-1)(\log \omega)$. Then we identify the largest $2 \leq d' \leq d$ such that $\log \kappa \geq (d'-1)(\log \omega)$ (since $\kappa \geq 5$, this is possible). Note that this means $\log \kappa \leq d'(\log \omega)$. Next we instantiate the construction for dimension $d'$ instead of $d$, that is
\[
f(\x) = \frac{1}{2}\sum_{i=1}^{d'} a_i \x_{(i)}^2,
\]
with $a_i = \omega^{d' - i}$. For this $f$, the condition number $\omega^{d'-1}$ is smaller than $\kappa$ so that $f \in \fancyQ_\kappa$. The path length is at least 
\begin{align*}
    \pathlength \text { (or $\pathlength_\eta$) } &\geq c\sqrt{d'}\ \dist{\x_0}{\X^*} 
    \\ &\geq c\sqrt{\log \kappa/\log \omega}\ \dist{\x_0}{\X^*} 
    \\ &= (c/\sqrt{\log \omega})\sqrt{\log \kappa}\ \dist{\x_0}{\X^*}.
\end{align*}
As we will prove, in the $\GF$ case, $c  = 0.7$ so that $c/\sqrt{\log 11} \geq 0.45$, and in the $\GD$ case, $c  = 0.5$ so that $c/\sqrt{\log 11} \geq 0.3$. This leads to the bounds in Equations~\eqref{eqn:quadraticLowerGF} and~\eqref{eqn:quadraticLowerGD}.

Now we analyze the path length for $f$ in $d$ dimensions and show the bound $\pathlength \text { (or $\pathlength_\eta$) } \geq c\sqrt{d}\ \dist{\x_0}{\X^*}$ for a specific $\x_0$.

\subsection{\GF Analysis}
Define $\x_0 = \mathbf{1}_d$ so that $\dist{\x_0}{\X^*} = \sqrt{d}$. Set $\delta = 0.07$. Observe that the \GF dynamics lead to $(\x_t)_{(i)} = e^{-a_i t}$. Consider the time steps $\{t_i\}_{i=0}^d$ where $t_0 = 0$ and $t_i = \log (1/\delta)/a_i$. Observe the following: 
\begin{enumerate}
    \item For every $i \in [d]$, $(\x_{t_i})_{(i)} = \delta$. 
    \item For every $i \in [d-1]$, 
    \begin{align*}
        (\x_{t_i})_{(i + 1)} &= e^{-a_{i+1} t_i}
       \\ &= e^{-\log(1/\delta)/\omega}
        \\ &= \delta^{1/\omega}.
    \end{align*}
\end{enumerate}
Thus splitting the path length integral in these time steps, 
\begin{align*}
    \pathlength &= \int_0^\infty \enorm{\vel{\x}_t}\ dt
    \\ &\geq \sum_{i=1}^d \int_{t_{i-1}}^{t_{i}}
    \enorm{\vel{\x}_t}\ dt
    \\ &\geq \sum_{i=1}^d \enorm{\x_{t_i} - \x_{t_{i-1}}}
    \\ &\geq \sum_{i=1}^d (\x_{t_{i-1}} - \x_{t_i} )_{(i)}
    \\ &\geq \sum_{i=1}^d (\delta^{1/\omega} - \delta )
    \\ &= d(\delta^{1/\omega} - \delta )
    \\&= \sqrt{d}\ (\delta^{1/\omega} - \delta )\ \dist{\x_0}{\X^*}
    \\&\geq 0.7\sqrt{d}\ \dist{\x_0}{\X^*} &\text{(plugging in values of $\delta$ and $\omega$)}.
\end{align*}

\subsection{\GD Analysis}
Define $\x_0 = \mathbf{1}_d$ so that $\dist{\x_0}{\X^*} = \sqrt{d}$. Set $\delta = e^{-3}$. Let $\eta = 1/2a_1 = 1/2L$. Observe that the \GD iterates are $(\x_k)_{(i)} = (1 - \eta a_i)^k$. Consider the iterates $\{k_i\}_{i=0}^d$ where $k_0 = 0$ and $k_i = a_1\log (1/\delta)/a_i$ which are integers. Observe the following: 
\begin{enumerate}
    \item For every $i \in [d]$, \begin{align*}
        (\x_{t_i})_{(i)} &= (1 - \eta a_i)^{k_i}
        \\ &\leq e^{-k_i\eta a_i}
        \\ &\leq \sqrt{\delta}. 
    \end{align*}
    \item For every $i \in [d-1]$, 
    \begin{align*}
        (\x_{t_i})_{(i + 1)} &= (1 - \eta a_{i+1})^{k_i}
        \\ &\geq e^{-2a_{i+1} k_i} &\text{($1 - x \geq e^{-2x}$ for $x \leq 0.5$)}
       \\ &= e^{-\log(1/\delta)/\omega}
        \\ &= \delta^{1/\omega}.
    \end{align*}
\end{enumerate}
Thus splitting the path length sum in these iterates, 
\begin{align*}
    \pathlength_\eta &= \sum_{k=0}^\infty \enorm{\x_k - \x_{k+1}}
    \\ &\geq \sum_{i=0}^{d-1}\sum_{k=k_i}^{k_{i+1}-1} \enorm{\x_k - \x_{k+1}}
    \\ &\geq \sum_{i=0}^{d-1} \enorm{\x_{k_i} - \x_{k_{i+1}}}
    \\ &\geq \sum_{i=0}^{d-1} (\x_{k_{i}} - \x_{k_{i+1}} )_{(i+1)}
    \\ &\geq \sum_{i=0}^{d-1} (\delta^{1/\omega} - \sqrt{\delta} )
    \\ &= d(\delta^{1/\omega} - \sqrt{\delta} )
    \\&= \sqrt{d}\ (\delta^{1/\omega} - \sqrt\delta )\ \dist{\x_0}{\X^*}
    \\&\geq 0.5\sqrt{d}\ \dist{\x_0}{\X^*}, 
\end{align*}
plugging in values of $\delta$ and $\omega$.
\hfill\BlackBox

\rev{
\section{Lower Bound Simulations}
In this section, we empirically simulate the lower bound constructions shown in Section~\ref{sec:lower-bounds} and draw some insights into the tightness of these bounds. }
\subsection{\PKL}
\label{appsec:PL-simulation}
In this simulation, we assess the \PKL lower bound construction of Theorem~\ref{thm:PLlowerBound}. %We compute the path length with \GD for the function $f$ described in the proof 
For the constructed lower bound function $f$ in the proof, it is not entirely evident if the computations are tight and that 
%The goal is to identify if the computations are tight and 
the final path length bound is $\Omega(\kappa^{1/4}/\log\kappa)$ (perhaps it has a larger dependence like $\Omega(\sqrt{\kappa})$). To assess this, we simulate \GD with the $f$ described in Section \ref{subsec:pkl-lb-fn} and numerically compute the path length. By varying the dimension $d \in [e^2, e^9]$, we obtain a range of values for the condition number of $f$: $\kappa \in [28.6, 3.29\cdot 10^7]$. For each of these functions with different $\kappa$ values, the step-size for \GD and the initialization point are set as described in Section~\ref{subsec:pkl-lb-x_0}. %$\eta$ is picked to be the largest value in $[1/4, 1/2]$ such that $k_1$ is an integer, and the corresponding initialization point $\x_0$ is fixed. 
The path length is computed by adding the lengths of all the updates, until $\enorm{\x_k} \leq 10^{-6}$. At this point the remaining distance to the origin $\enorm{\x_k}$ is added to the path length computation. Finally, $\mu$ is computed as the `effective' PKL constant from the iterates actually seen while running the simulation: 
\[
\mu = \max_{t \in \{0, 1, \ldots k\}} \frac{\esqnorm{\nabla f(\x_t)}}{2f(\x_t)}.
\]
Consequently, since $L  = 2$,  we obtain $\kappa = 2/\mu$. The observed path length ratio $\pathlength_\eta/\dist{\x_0}{\X^*}$ is plotted against the effective $\kappa$ in Figure~\ref{fig:PLGD-simulation}. 

The X-axis in the figure corresponds to $\kappa \in [28.6, 3.29\cdot 10^7]$. From the figure, we observe that the path length ratio is $\approx 3\kappa^{1/4}/\log\kappa$ across various values of the effective $\kappa$. Note that for the given range of $\kappa$, $\log\log\kappa \in [1.21, 2.85]$. Thus if we were to modify the X-axis with a $\log\log\kappa$ factor, the dependence would no longer be linear. We infer visually that $\Omega(\kappa^{1/4}/\log\kappa)$ is the right dependence on $\kappa$ including $\log\log\kappa$ factors. However, the constants in the lower bound can be improved since Theorem~\ref{thm:PLlowerBound} only shows $\kappa \geq \kappa^{1/4}/16\log\kappa$. 

\begin{figure}
    \centering
    \includegraphics[width=0.6\linewidth]{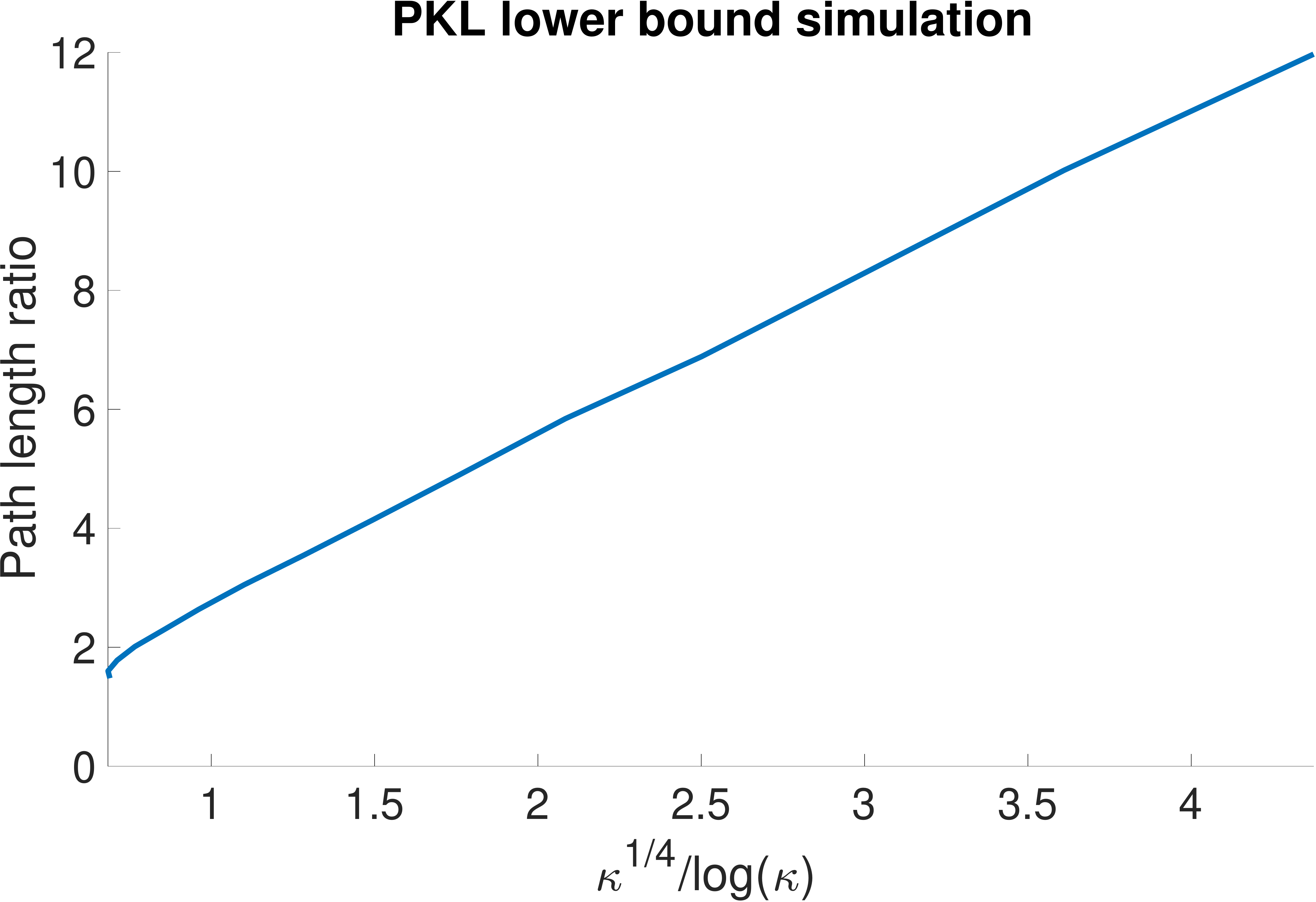} 
    \caption{Dependence of the path length ratio $\pathlength_\eta/\dist{\x_0}{\X^*}$ on the condition number for the PKL lower bound construction in the proof of Theorem~\ref{thm:PLlowerBound}. We infer that the dependence is $\Omega(\kappa^{1/4}/\log\kappa)$, including $\log\log\kappa$ factors.}
    \label{fig:PLGD-simulation}
\end{figure}

\subsection{Quadratics}
\label{appsec:quadratic-simulation}
For quadratics, we showed a %described in the proof of Theorem~\ref{thm:PLlowerBoundQuadratic} to assess the constants in the 
$\Theta(\log(\sqrt{\kappa}))$ upper bound (Theorem~\ref{thm:QUADGF}) and lower bound (Theorem~\ref{thm:PLlowerBoundQuadratic}). We simulate the lower bound example described in the proof of Theorem~\ref{thm:PLlowerBoundQuadratic} to assess the constants in these bounds. Consider gradient descent and gradient flow with 
\[
f(\x) = \frac{1}{2} \sum_{i=1}^da_i \x_{(i)}^2,
\]
where $d = 150$ and $a_i$ is set similar to the construction in the proof of Theorem~\ref{thm:PLlowerBoundQuadratic}: for every $i \in [d]$, $a_i = \omega^{d-i}$. To vary the condition number $\kappa = \omega^{d-1}$, we vary $\omega \in [1.00008, 2]$ (instead of varying $d$ as we did in the proof of Theorem~\ref{thm:PLlowerBoundQuadratic}). This leads to $\kappa \in [9.553 \cdot 10^{4}, 6.2807 \cdot 10^{11}]$. For the initialization point, we set $(\x_0)_i = 1$ for every $i \in [d]$.

%\begin{itemize}
%\item Lower bound construction (from the proof of Theorem~\ref{thm:PLlowerBoundQuadratic}): $a_i = \omega^{d-i}$. To vary the condition number $\kappa = \omega^{d-1}$, we vary $\omega \in [1.00008, 2]$ (instead of varying $d$ as we did in the proof of Theorem~\ref{thm:PLlowerBoundQuadratic}). This leads to $\kappa \in [9.553 \cdot 10^{4}, 6.2807 \cdot 10^{11}]$. For every $i \in [d]$, $\beta_i = 1$. 
%\item Random construction: $a_1 = 1, a_d = 1/\kappa, a_i \sim \text{Unif}(1/\kappa,1)$ where $\kappa = \omega^{d-1}$ for $\omega \in [1.00008, 2]$. Every $\beta_i$ is sampled independently from $\text{Unif}(0, 1)$ and then the $\beta$ vector is normalized so that $\enorm{\beta} = \sqrt{d}$. 
%\end{itemize}
\begin{figure}
    \centering
	\includegraphics[width=0.48\linewidth]{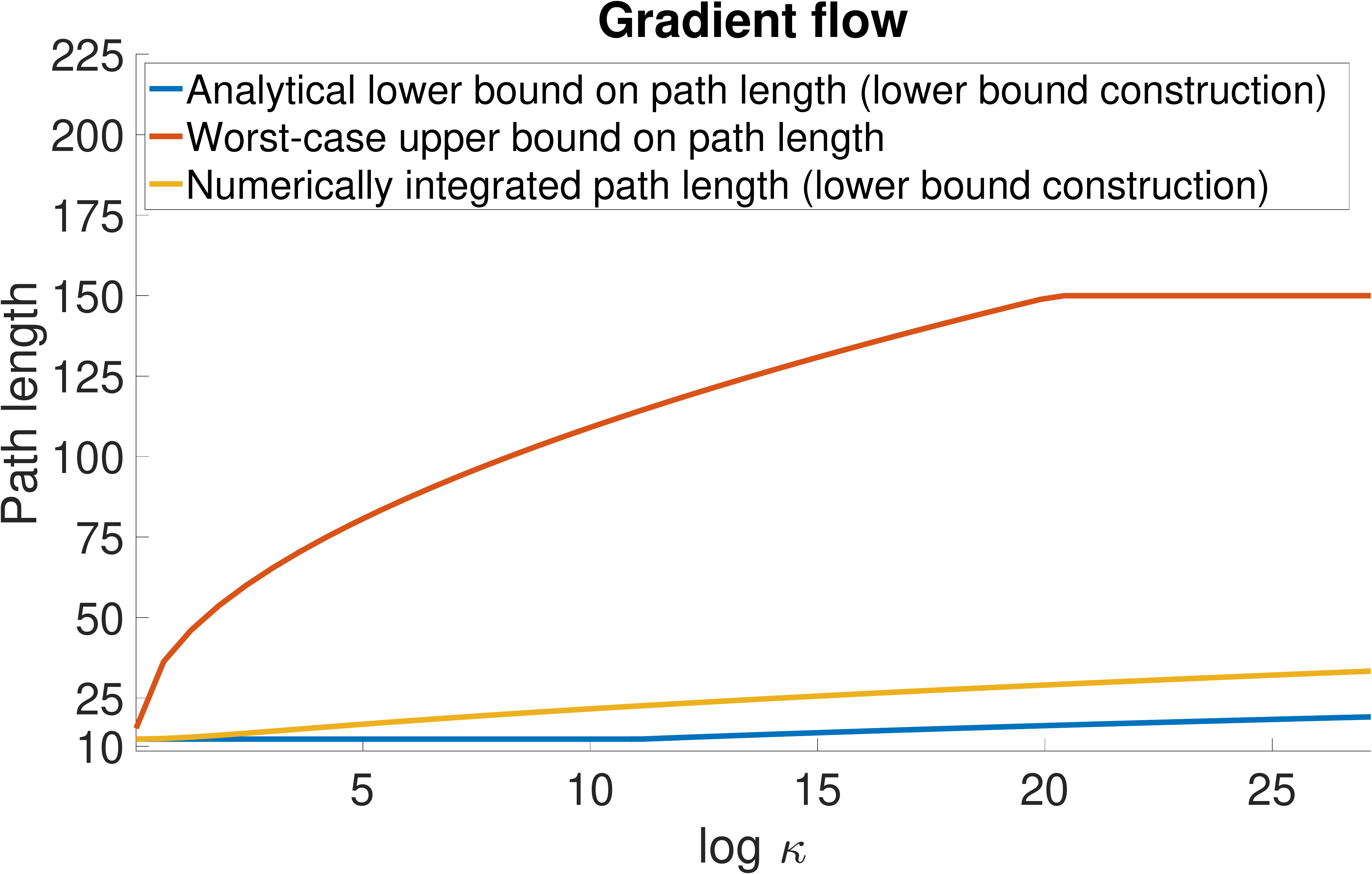}
	\includegraphics[width=0.48\linewidth]{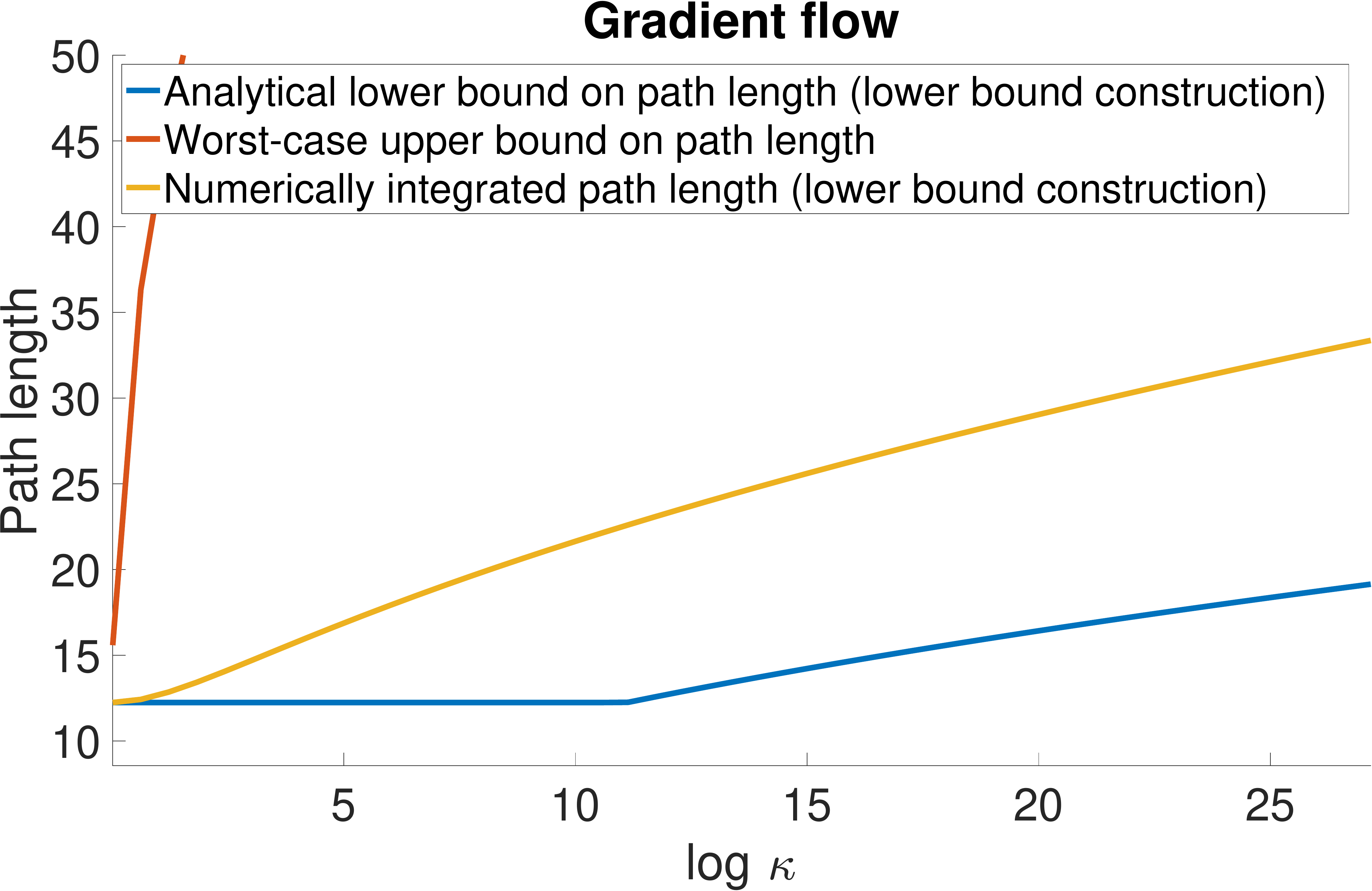}\\
    \vspace{0.2cm}
    \includegraphics[width=0.48\linewidth]{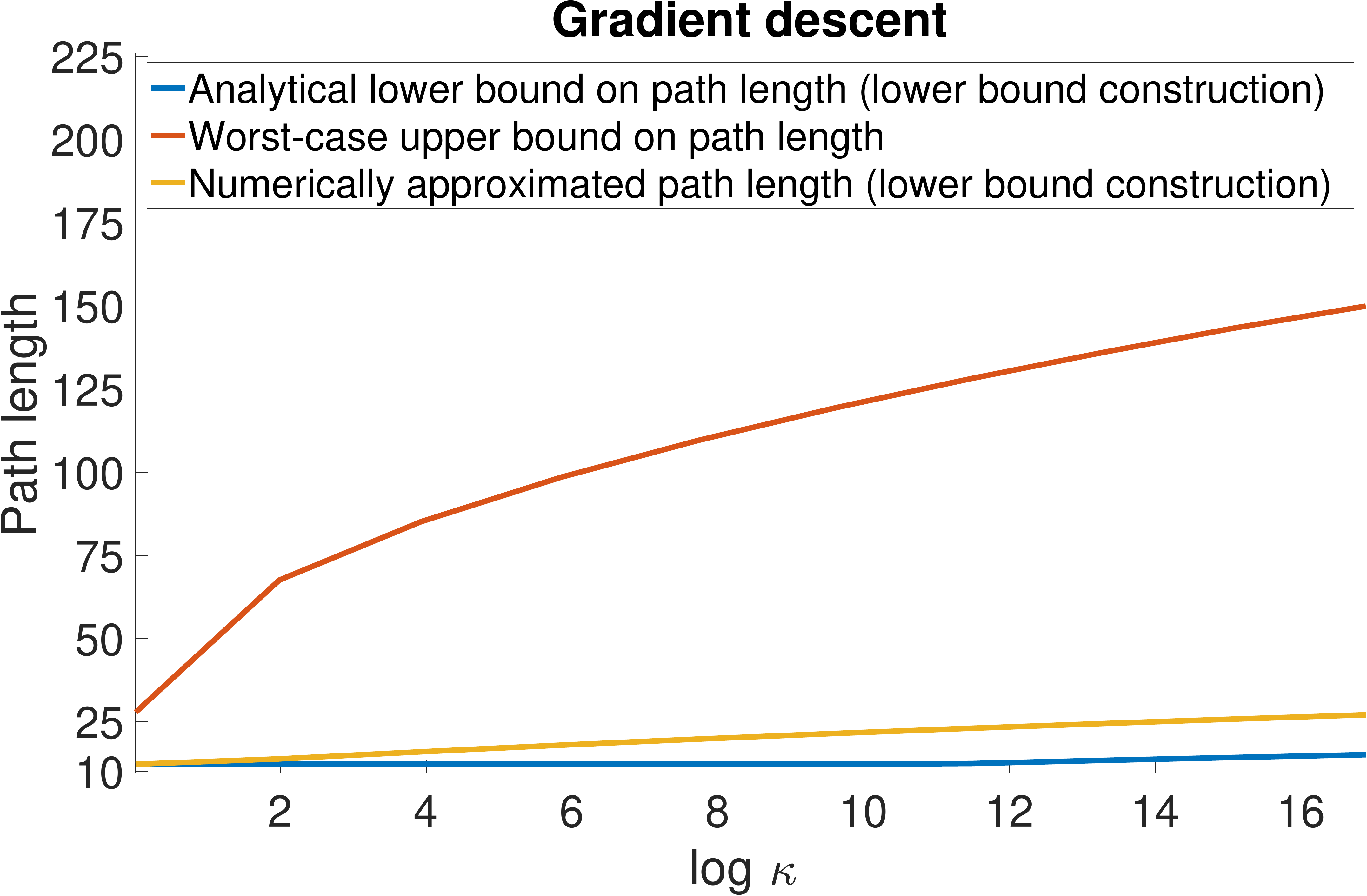} 
   \includegraphics[width=0.48\linewidth]{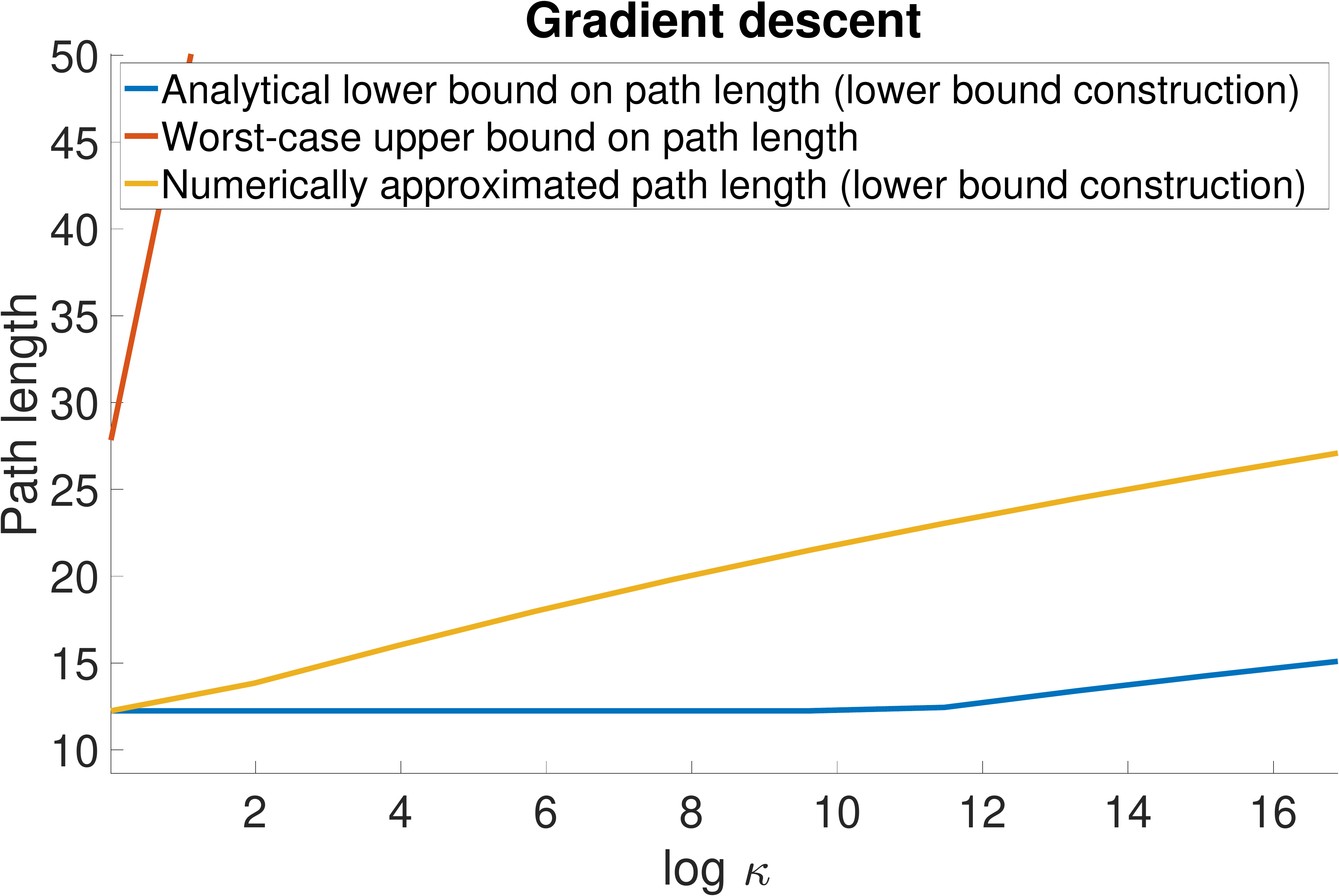}
%    \caption{Path length of \GF (top) and \GD (bottom) for quadratic objectives. The lower limit on the X-axis is the same for all plots. The plots on the right are `zoomed in' versions of the plots on the left; they focus on a smaller range for the Y-axis. The lower bound refers to the specific construction in Theorem~\ref{thm:PLlowerBoundQuadratic} and the upper bound refers to Theorem~\ref{thm:QUADGF}. The random construction is described in Section~\ref{appsec:quadratic-simulation}. The shortest path has length $\sqrt{d} = \sqrt{150} \approx 12.25$, and the $\kappa$-independent upper bound on the path length is $d$.}
    \caption{Path length of \GF (top) and \GD (bottom) for quadratic objectives. The lower limit on the X-axis is the same for all plots. The plots on the right are `zoomed in' versions of the plots on the left; they focus on a smaller range for the Y-axis. The lower bound refers to the specific construction in Theorem~\ref{thm:PLlowerBoundQuadratic}, described in Section~\ref{appsec:quadratic-simulation}. The shortest path has length $\sqrt{d} = \sqrt{150} \approx 12.25$, and the $\kappa$-independent upper bound on the path length is $d$.}
    \label{fig:quadratic-simulation}
\end{figure}

%The smallest value of $\kappa$ was decided to ensure that the \GD lower bound $0.3\sqrt{\log\kappa}$ is larger than the trivial bound of $1$. This leads to $\kappa > 6.69 \cdot 10^4$. The largest value of $\kappa$ was decided to take into account 

 The \GF path length is computed by performing the path length integral numerically (since we know $\x_t$ at each point in the case of quadratics). We use MATLAB's numerical integration function \texttt{integral}, with the \texttt{AbsTol} parameter set to $10^{-50}$. The \GD path length is computed by simulating \GD with step size $\eta = 1/2a_1$ and adding the lengths of all the updates, until every component except the last one is smaller than $10^{-2}$, ie $\max_{i \neq d} (\x_{k})_{(i)} < 10^{-2}$. At this point the remaining distance to the origin $\enorm{\x_k}$ is added to the path length computation. We use this heuristic since \GD convergence is quite slow when $\kappa$ is very large (the asymptotic convergence rate with respect to $\kappa$ is $\mathcal{O}(\kappa)$). Further, we limit the largest $\kappa$ in the \GD plot to around $2\cdot 10^7$. The results are plotted in Figure~\ref{fig:quadratic-simulation}. 
 
Further, to verify that our lower bound construction is non-trivial, we compare the path length of the construction in Theorem~\ref{thm:PLlowerBoundQuadratic} with a randomized construction that has the $a_i$ values sampled as follows: $a_1 = 1, a_d = 1/\kappa, a_i \sim \text{Unif}(1/\kappa,1)$, where $\kappa = \omega^{d-1}$ for $\omega \in [1.00008, 2]$ as in the previous construction. For the initialization point, every $(\x_0)_i$ is sampled independently from $\text{Unif}(0, 1)$ and then normalized so that $\enorm{\x_0} = \sqrt{d}$. The results are plotted in Figure~\ref{fig:quadratic-simulation-random}. 
\begin{figure}
    \centering
    \includegraphics[width=0.48\linewidth]{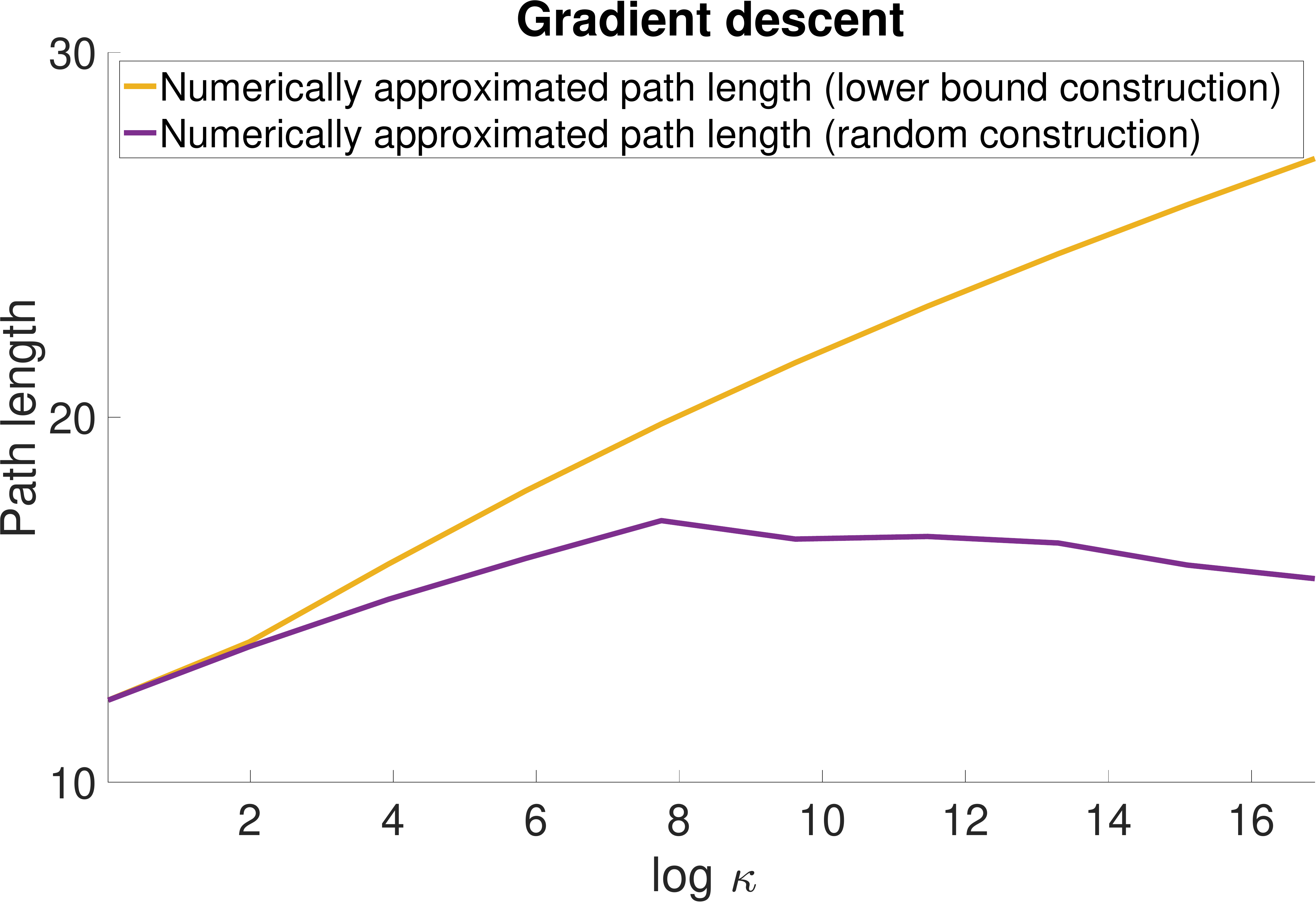} 
     \includegraphics[width=0.48\linewidth]{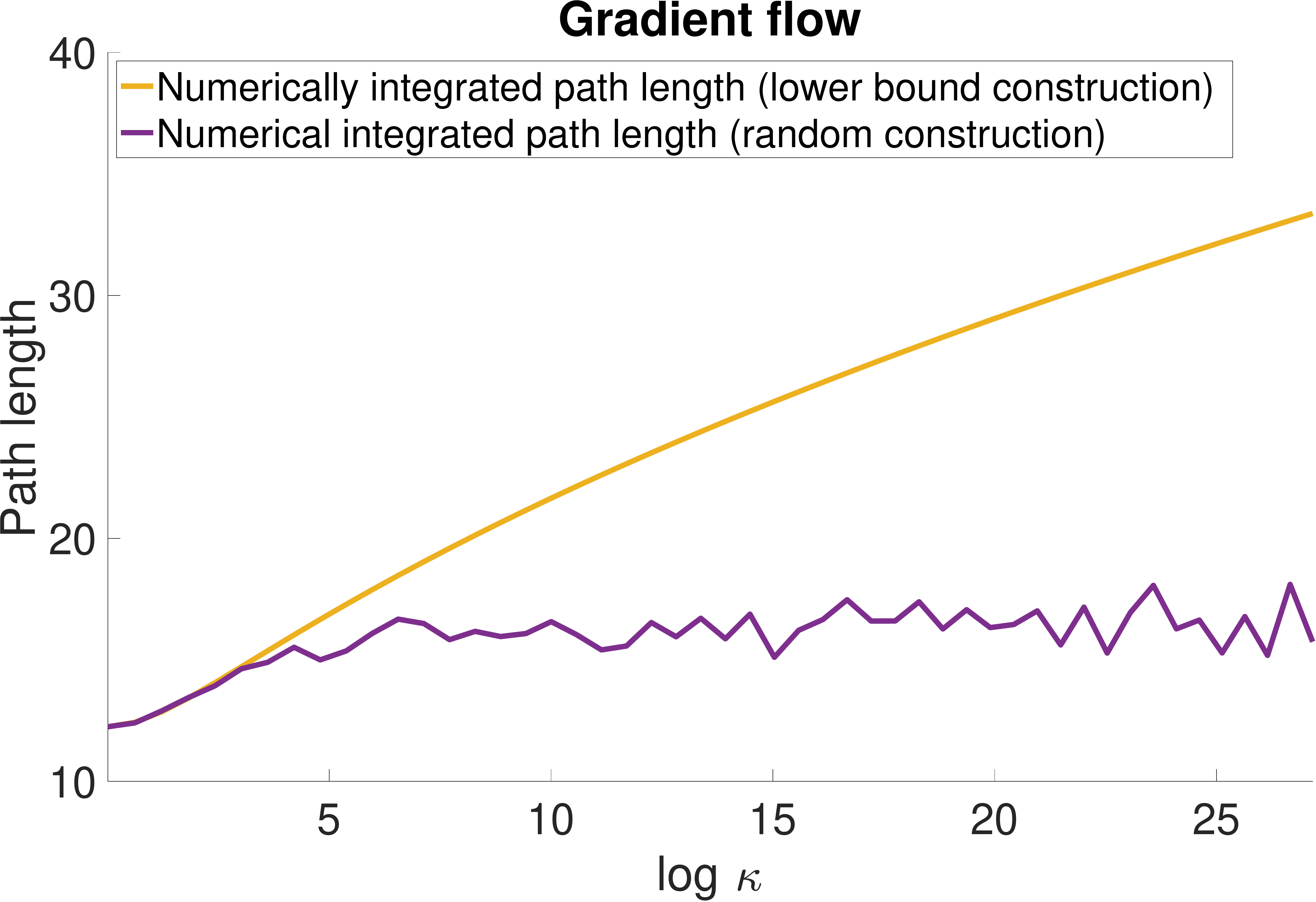} 
    \caption{The lower bound construction of Theorem~\ref{thm:PLlowerBoundQuadratic} has larger path length than the randomized construction defined in Appendix~\ref{appsec:quadratic-simulation}. }
    \label{fig:quadratic-simulation-random}
\end{figure} 
 
We do not show the results with averaging across multiple runs for the randomized construction, but we observed behavior similar to Figure~\ref{fig:quadratic-simulation-random} across runs. Thus the randomized quadratic construction does not have a $\sqrt{\log \kappa}$ dependence for its path length. This indicates that the lower bound construction in Theorem~\ref{thm:PLlowerBoundQuadratic} is non-trivial. At the same time, Figure~\ref{fig:quadratic-simulation} shows that the upper bound overestimates the path length of this lower bound construction. Given these  results, we conjecture that the constants in the upper bound can be improved.

\end{document}